\documentclass[12]{amsart}
\usepackage[utf8]{inputenc}

\usepackage[stable]{footmisc}
\usepackage{amsmath}  
\usepackage{graphicx} 
\usepackage{amsthm}
\usepackage{amssymb}
\usepackage{bm}
\usepackage{bbm}
\usepackage{booktabs}
\usepackage{dsfont}
\usepackage{graphicx}
\usepackage{subfigure}
\usepackage{blindtext}
\usepackage{url}
\usepackage{algorithm}
\usepackage{algorithmic}
\usepackage{setspace}
\usepackage{enumitem}
\usepackage[margin=1in]{geometry}
\usepackage{cite} 
\usepackage[final]{hyperref} 
\hypersetup{
	colorlinks=true,       
	linkcolor=blue,        
	citecolor=blue,        
	filecolor=magenta,     
	urlcolor=blue         
}

\newtheorem{theorem}{Theorem}[section]

\newtheorem{lemma}[theorem]{Lemma}
\newtheorem{remark}{Remark}
\newtheorem{definition}{Definition}[section]
\newtheorem{assumption}{Assumption}

\newcommand{\red}{\color{red}}
\newcommand{\nc}{\normalcolor}

\DeclareMathOperator*{\argmin}{arg\,min}
\newcommand{\R}{\mathbb{R}}
\newcommand{\N}{\mathbb{N}}
\newcommand{\M}{\mathcal{M}}

\newcommand{\E}{\mathbb{E}}
\makeatletter
\def\url@leostyle{%
	\@ifundefined{selectfont}{\def\UrlFont{\sf}}{\def\UrlFont{\small\ttfamily}}}
\makeatother
\title{FedCBO: Reaching Group Consensus in Clustered Federated Learning through Consensus-based Optimization}
\author{Jos\'e A. Carrillo}
\address{Mathematical Institute, University of Oxford, Oxford OX2 6GG, UK.}
\email{carrillo@maths.ox.ac.uk}
\author{Nicol\'as {Garc\'ia Trillos}}
\address{Department of Statistics, University of Wisconsin-Madison, 1300 University Avenue, Madison, Wisconsin 53706, USA.}
\email{garciatrillo@wisc.edu}
\author{Sixu Li}
\address{Department of Statistics, University of Wisconsin-Madison, 1300 University Avenue, Madison, Wisconsin 53706, USA.}
\email{sli739@wisc.edu}

\author{Yuhua Zhu}
\address{Department of Mathematics, Hal{\i}c{\i}o{\u g}lu Data Science Institute, University of California, San Diego, La Jolla, California 92093, USA }
\email{yuz244@ucsd.edu}

\begin{document}
\pagestyle{plain}

\thanks{{\bf Acknowledgements:}
Authors' names are listed in alphabetical order by family name. The authors are thankful to Huang Hui and Jinniao Qiu for enlightening discussions on mean-field limits of CBO. This work was started while the authors were visiting the Simons Institute to participate in the program ``Geometric Methods in Optimization and Sampling" during the Fall of 2021. The authors would like to thank the institute for hospitality and support. NGT was supported by NSF-DMS grants 2005797 and 2236447, and, together with SL would like to thank the IFDS at UW-Madison and NSF through TRIPODS grant 2023239 for their support. JAC was supported by the Advanced Grant Nonlocal-CPD (Nonlocal PDEs for Complex Particle Dynamics: Phase Transitions, Patterns and Synchronization) of the European Research Council Executive Agency (ERC) under the European Union’s Horizon 2020 research and innovation programme (grant agreement No. 883363). JAC was also partially supported by the EPSRC grant numbers EP/T022132/1 and EP/V051121/1. }

\maketitle





\begin{abstract}
Federated learning is an important framework in modern machine learning that seeks to integrate the training of learning models from multiple users, each user having their own local data set, in a way that is sensitive to data privacy and to communication loss constraints. In clustered federated learning, one assumes an additional unknown group structure among users, and the goal is to train models that are useful for each group, rather than simply training a single global model for all users. In this paper, we propose a novel solution to the problem of clustered federated learning that is inspired by ideas in consensus-based optimization (CBO). Our new CBO-type method is based on a system of interacting particles that is oblivious to group memberships. Our model is motivated by rigorous mathematical reasoning, including a mean field analysis describing the large number of particles limit of our particle system, as well as convergence guarantees for the simultaneous global optimization of general non-convex objective functions (corresponding to the loss functions of each cluster of users) in the mean-field regime. Experimental results demonstrate the efficacy of our FedCBO algorithm compared to other state-of-the-art methods and help validate our methodological and theoretical work.

\end{abstract}

\section{Introduction}
The wide use of \textit{internet of things} (IoT) devices in various applications such as home automation, personal health monitoring, and vehicle-to-vehicle communications has led to the generation of vast amounts of data across a collective of users. However, concerns around data privacy and security, as well as limitations on communication costs and bandwidth, have made it challenging for an individual user to take advantage of this large amount of stored information.
This has motivated the design and development of federated learning (FL) strategies, which aim at pooling information from learning models trained on local devices to build models \textit{without} relying on the collection of local data \cite{mcmahan2017communication, kairouz2021advances}.

Standard FL approaches aim to learn one global model for all local clients/users \cite{mcmahan2017communication,li2020federated,mohri2019agnostic,karimireddy2020scaffold}. 
However, data heterogeneity, also known as non-i.i.d. data, naturally arises in FL applications since data are usually generated from users' personal devices. Thus, it is expected that \textit{no single} global model can perform well across all clients \cite{sattler2020clustered}. On the other hand, it is reasonable to expect that users with similar backgrounds are likely to make similar decisions and thus generate data following similar distributions. This paper studies one formulation of federated learning with non-i.i.d. data, namely Clustered Federated Learning (CFL)\cite{sattler2020clustered,ghosh2020efficient,ruan2022fedsoft,long2023multi,ma2022convergence}. 
In CFL, users are partitioned into different clusters, and the objective is to train models for each cluster of users. These clusters may represent, for example, groups of users with preferences in different categories of movies and TV series. Our focus in this work is on the mathematical modeling and analysis of CFL methods and on exploring CFL's effectiveness in improving the performance of FL when dealing with non-i.i.d. data. Specifically, we investigate how CFL can create personalized models for clusters of users with similar preferences. Our research is motivated by previous studies of CFL that have shown promising results in enhancing the performance of FL in the non-i.i.d. data setting. 


To start making our set-up more precise, let us consider the clustered federated learning setting with one global server and $N$ different agents. 
We assume that each agent belongs to one of $K$ non-overlapping groups denoted by $S_1^*, \dots, S_K^*$. 
We further assume that each agent belonging to group $S_k^*$ owns data points generated from distribution $\mathcal{D}_k$ that can be used for training its own learning models. Ideally, an agent would seek to communicate with other agents in its group to accelerate the training process of its own model. 
{However, under data privacy constraints (see the discussions of privacy in Remark \ref{remark: privacy}), i.e. an agent will not share its local data with global server and other agents, the underlying partition $S_1^*, \dots, S_K^*$ is never revealed to the learning algorithm.}
Let $l(\theta; z): \Theta \rightarrow \mathbb{R}$ be the loss function associated with data point $z$, where $\Theta \subset \mathbb{R}^d$ is the parameter space for the learning models. 
Our goal is to minimize the population loss function
\begin{equation}\label{eqn: population loss}
    L_i(\theta) := \mathbb{E}_{z \sim \mathcal{D}_i} [l(\theta; z)]
\end{equation}
for all $i \in [K]$ simultaneously. In other words, the goal is to  find minimizers $\theta_i^*$ for all loss functions
\begin{equation}\label{eqn: CFL problem}
    \theta_i^* \in  \argmin_{\theta \in \Theta} L_i(\theta), \quad i \in [K].
\end{equation}
A toy example illustrating the clustered federated learning framework is shown in Fig. \ref{fig: CFL toy example}.
\begin{figure}
    \centering
    \includegraphics[width=0.9\textwidth]{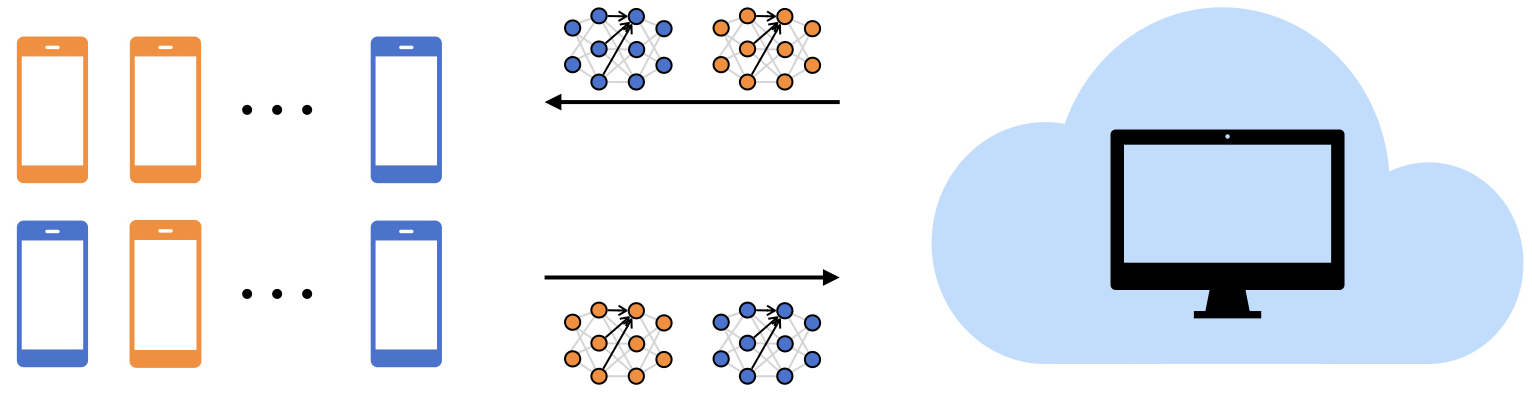}
    \caption{Toy example of a clustered federated learning problem. Each mobile phone user has an underlying cluster identity, here represented by the colors orange and blue. We aim to identify the group memberships of users while simultaneously training models for every cluster by communicating model parameters with the cloud global server.}
    \label{fig: CFL toy example}
\end{figure}

As suggested by the discussion above, the main difficulty in CFL comes from the fact that \textit{cluster identities of users are unknown}. A CFL algorithm must then be able to induce clustering among users and simultaneously train models in a distributed setting without relying on local data collection. In order to propose an algorithm that accomplishes this, in this paper we abstract the CFL problem and formulate it mathematically borrowing ideas from consensus-based optimization (CBO) \cite{carrillo2018analytical,carrillo2021consensus,totzeck2021trends}. CBO is a family of global optimization methodologies based on systems of interacting particles that seek consensus around global minimizers of target objectives. Precisely, consider
\begin{equation*}
    \min_{\theta \in \R^d}\; L(\theta),
\end{equation*}
where the target function $L$, which may be non-convex, is a continuous function with a unique global minimizer $\theta^*$. For each $i\in [N]$, let  $\theta^i \in \R^d$ represent the position of particle $i$ and consider the following system of equations
\begin{equation}
d\theta^i_t = -\lambda(\theta_t^i - m_t) dt + \sigma |\theta_t^i - m_t| dB_t^i, \qquad \text{for $i =1,2, \dots, N$},
\label{eq:CBO}
\end{equation}
where the $\{ B^i\}_i$ are independent Brownian motions and $m_t$ is a weighted average defined by
\begin{equation*}
    m_t := \frac{\sum_{i=1}^N \theta_t^i \exp\big(-\alpha L(\theta_t^i)\big)}{\sum_{i=1}^N \exp \big(-\alpha L(\theta_t^i)\big)}.
\end{equation*}
One can alternatively consider other types of noise for  \eqref{eq:CBO} (see \cite{carrillo2021consensus,carrillo2022consensus}). For instance, one may substitute the diffusion term in \eqref{eq:CBO} with a geometric component-wise Brownian motion as in \cite{carrillo2021consensus}. This allows to improve the performance of the CBO algorithm by making some assumptions independent of the dimension $d$, and then useful for machine learning applications in high dimension. One can also adapt in a more general and anisotropic way the noise by using a covariance matrix defined similarly to $m_t$ as in \cite{carrillo2022consensus}, a method called Consensus Based Sampling (CBS) in optimization mode. Here we will stick to the basic CBO method for simplicity and refer the interested reader to \cite{carrillo2021consensus,carrillo2022consensus} for more details on other existing variants of CBO.  

We notice that the term $\exp (-\alpha L(\theta) )$ in the formula for $m_t$ is the Gibbs distribution corresponding to the objective function $L(\theta)$ and temperature $\frac{1}{\alpha}$. The motivation for assigning weights in this way comes from the Laplace principle \cite{miller2006applied,dembo2009large}, which states that for any probability measure $\rho \in \mathcal{P}(\R^d)$ compactly supported with $\theta^* \in \text{supp}(\rho)$ we have
\begin{equation*}
\lim_{\alpha \rightarrow +\infty} \bigg(-\frac{1}{\alpha} \log \bigg(\int_{\R^d} \exp\big( -\alpha L(\theta) \big) d\rho(\theta) \bigg) \bigg) = L(\theta^*).
\end{equation*}
Hence, if $\theta^*$ is the unique minimizer of $L$, then the measure $\exp(-\alpha L(\theta)) \rho(\theta)$, normalized by a constant factor, will assign most of its mass to a small neighborhood of $\theta^*$, and if $\alpha$ is large enough, this measure will approximate the Dirac delta distribution at $\theta^*$. Consequently, the weighted-average $m_t$, which is the empirical first moment of a normalized version of the measure $\exp(-\alpha L(\theta)) \rho(\theta)$,  is a reasonable target for particles to follow as in equation \eqref{eq:CBO}. 

Although CBO can be easily adapted to the distributed setting when cluster identities are known (as one could simply run CBO on each cluster), it is not directly applicable to CFL if the goal is to propose dynamics that are oblivious to agents' identities. Our problem is also different from standard multi-objective optimization, for which CBO has already been adapted; see \cite{borghi2022adaptive,borghi2022consensus}. Indeed, in standard multi-objective optimization the goal is to find a point that is Pareto optimal for $K$ different target functions, whereas our goal is to find, simultaneously, $K$ global minimizers for $K$ different objective functions. On the other hand, the CBO approach is inherently gradient-free, so it is preferable when the objective function is not smooth enough or its derivative is expensive to evaluate. However, if communication costs are expensive as in real FL applications, one may consider introducing additional local gradient terms in the dynamics of each user so that training may continue even when there is no communication among users.

%


\subsection{Contributions and related works}\label{sec: related works}

\subsubsection{Contributions} 
Motivated by the discussion above, we propose a new CBO-inspired interacting particle system (see \eqref{eqn: particle system} below and the discussion right after) that is suited for the clustered federated learning setting. In our system, the evolution of each individual particle is completely determined by its own loss function, its own identity, and the locations of the other particles \textit{but no} knowledge of their identities. More precisely, our main contributions can be summarized as follows:\\
\indent\textbf {(1)} We introduce a novel CBO-type framework that enables the minimization of $K$ objective functions in a CFL setting without knowing the cluster identity of any of the particles.
This is achieved by introducing a mechanism that secretly forces consensus among particles belonging to the same cluster. Moreover, we incorporate a gradient term in the particle dynamics, which dramatically reduces the number of communication rounds required for the CBO algorithm to achieve good performance.\\
\indent\textbf{(2)} We provide rigorous theoretical justification for the proposed framework. In particular, we prove the well-posedness of the proposed particle system and its convergence toward a mean-field limit equation and study the consensus behavior of the mean-field dynamics and their ability to concentrate around global minimizers of each of the underlying loss functions. \\
\indent\textbf{(3)} We discretize our continuous dynamics in a reasonable way and fit it into the conventional federated training protocol to obtain a new federated learning algorithm that we call FedCBO.  We conduct extensive numerical experiments to verify the efficiency of our algorithm and demonstrate that it outperforms other state-of-the-art methods in the non-i.i.d. data setting. 

\subsubsection{Related work in clustered federated learning}

In the setting of CFL, it is assumed that there is an underlying clustering structure among users, and the goal is to identify the clusters' identities and federate among each group. Both IFCA \cite{ghosh2020efficient} and HypCluster \cite{mansour2020three}
alternate between identifying cluster identities of users and updating models for the user clusters via local gradient descent. These methods identify cluster identities by finding the model with the lowest loss on each local dataset.
FedSEM \cite{long2023multi} groups the users at each federated step by measuring the distance between users using model parameters and accuracy and then running a simple $K$-means algorithm. All these methods require prior knowledge or estimation of the number of underlying clusters, which may be difficult to have/do in practice. In contrast, as we discuss below, our method does not require any prior knowledge of the clustering structure, and consensus among clients in the same cluster will be automatically induced by our particle dynamics. 

In \cite{sattler2020clustered}, clusters are found in a hierarchical way. In particular, clients are recursively divided into two sets based on the cosine similarity of the clients' model gradients or weight-updates. 
WeCFL \cite{ma2022convergence} formalizes clustered federated learning problems into a unified bi-level optimization framework. Unlike the two framework mentioned above \cite{sattler2020clustered,ma2022convergence}, which conduct the convergence analysis under convexity assumptions, 
our paper considers target functions that are non-convex and provide an asymptotic convergence result in the mean-field regime.

\subsubsection{Related work in consensus-based optimization}
The idea of using interacting particle dynamics with consensus-inducing terms to solve global optimization problems was first introduced in \cite{pinnau2017consensus}. Since then, this approach has gained a lot of interest from both theoretical and applied perspectives.

On the theoretical side, \cite{carrillo2018analytical} provided the first local convergence analysis of a mean-field CBO equation under relatively stringent assumptions on the initialization of the system. This is achieved by first proving consensus formation at the mean-field level in the infinite time horizon and then tuning the consensus point using the Laplace principle. Later, \cite{fornasier2021consensus} relaxed some of these stringent assumptions and proved that mean-field dynamics can reach an arbitrary level of concentration around a global minimizer within a finite time interval; however, this time horizon may be difficult to estimate a priori. 
By showing that the finite particle CBO system converges to the mean-field limit, \cite{huang2021mean} furthers the theoretical underpinnings of the CBO framework. In our paper, we use similar strategies as in \cite{fornasier2021consensus,riedl2022leveraging} and \cite{huang2021mean} to study the behavior of our mean field system (Section \ref{sec: long term behavior}) and prove that our proposed particle dynamics converge toward a suitable mean-field limit (Section \ref{sec: M-F limit}). In our setting, we need to face new challenges due to the fact that particles may have different dynamics that depend on the loss functions that they try to optimize. For instance, we need to estimate the time horizon needed to achieve a given accuracy at the mean field level differently from \cite{fornasier2021consensus,riedl2022leveraging}. Likewise, our convergence of the finite particle system toward a suitable mean field limit involves additional technical difficulties arising from the fact that in our setting there are multiple types of particles interacting with each other. 

On the algorithmic side, with motivations from a variety of applications, researchers have extended and adapted the original CBO model to include new settings such as global optimization on compact manifolds like the sphere \cite{fornasier2021sphere}, general constraints \cite{carrillo2023consensus}, high-dimensional machine learning problems \cite{carrillo2021consensus}, global optimization of objective functions with multiple minimizers  \cite{bungert2022polarized}, sampling from  distributions \cite{carrillo2022consensus,bungert2022polarized}, and saddle point optimization problems \cite{huang2022consensus}. 
In \cite{riedl2022leveraging}, the authors introduce a gradient term in the CBO system which is shown to be beneficial numerically when applied the compressed sensing problems. 
In our paper, we also incorporate  gradient information in our particle dynamics, but our motivation, different from the one in \cite{riedl2022leveraging}, is to reduce communication costs among users, one of the important practical constraints in federated learning. For a more comprehensive review of the development of CBO-type methods we refer the interested reader to the recent survey \cite{totzeck2021trends}.

\subsection{Notation}
We use $|\cdot|$ to denote the absolute value or $\ell_2$-norm of vectors in Euclidean space and denote by $B_r(\theta)$ the open ball of radius $r$ centered at $\theta \in \R^d$. We denote by $\mathcal{C}(X, Y)$ the space of continuous functions $f: X \rightarrow Y$ between $X \subset \R^n$ and a given topological space $Y$. The space $\mathcal{C}(X, Y)$ is endowed with the sup-norm as is standard. When $Y=\R$, we simply use the notation $\mathcal{C} (X)$.
We also use $\mathcal{C}_c^k(X)$ and $\mathcal{C}_b^k(X)$ to denote, respectively, the space of real-valued functions that are $k$-times continuously differentiable with compact support and the space of bounded functions that are $k$-times continuously differentiable.
Let $\mathcal{P}(\R^d)$ be the space of all Borel probability measures over $\R^d$ equipped with the Levy-Prokhorov metric, which metrizes the topology of weak convergence. For a given $p\geq 1$, we let $\mathcal{P}_p(\R^d) \subseteq \mathcal{P}(\R^d)$ be the collection of probability measures $\rho \in \mathcal{P}(\R^d)$ with finite $p$-th moments, i.e., $\int_{\R^d} |\theta|^p d\rho(\theta)<\infty$. The space $\mathcal{P}_p(\R^d)$  is endowed with the $p$-Wasserstein distance $W_p$ ($1 \leq p < \infty$) defined according to
\begin{equation*}
    W_p (\rho, \hat{\rho}) := \inf_{\pi \in \Gamma(\rho, \hat{\rho})} \int_{\R^d \times \R^d} |\theta - \hat{\theta}|^p \pi(d\theta, d\hat{\theta}), \qquad \rho, \hat{\rho} \in \mathcal{P}_p(\R^d),
\end{equation*}
where $\Gamma(\rho, \hat{\rho})$ denotes the set of all joint probability measures over $\R^d \times \R^d$ with first and second marginals $\rho$ and $\hat{\rho}$, respectively. 

When studying the laws of stochastic processes $\rho \in \mathcal{C}([0, T], \mathcal{P}(\R^d))$, we denote the law at time $t$ as $\rho_t \in \mathcal{P}(\R^d)$.
Given a continuous function $f \in \mathcal{C}(\R^d)$ and a fixed probability measure $\rho \in \mathcal{P}(\R^d)$, we denote by $\|f\|_{\mathbb{L}^1(\rho)} := \int_{\R^d} |f(\theta)| d\rho(\theta)$ the $L^1$-norm of $f$ with respect to the measure $\rho$.

\subsection{Organization of the paper} 
The rest of the paper is organized as follows. In Section \ref{sec:ModelFormulation}, we introduce the interacting particle system motivating our FedCBO algorithm. In Section \ref{sec:MainResults}, we state our main theoretical results, which include the well-posedness of a mean-field system associated to our proposed interacting particle dynamics (Theorem \ref{thm: well-posedness of mean-field system}), the convergence of the interacting particle system toward the mean-field as the number of particles in the system grows (Theorem \ref{thm:mainMeanField})), and, finally, the behavior of the mean-field in time and its ability to concentrate around global optimizers for each of the objective functions (Theorem \ref{thm: large time behavior}). 
Motivated by our particle system, in Section \ref{sec: alg} we introduce our FedCBO algorithm. In Section \ref{sec: experiment}, we present a series of numerical experiments to validate our proposed algorithm. Section \ref{sec:WellPosedness} is devoted to the proof of Theorem \ref{thm: well-posedness of mean-field system}, Section \ref{sec: long term behavior} to the proof of Theorem \ref{thm: large time behavior}, and Section \ref{sec: M-F limit} to the proof of Theorem \ref{thm:mainMeanField}. We wrap up the paper in Section \ref{sec:Conclusions}, where we summarize our contributions and discuss future research directions.


\section{CBO for Clustered Federated Learning}\label{sec: proposed model and well-posedness}
\subsection{Model formulation}
\label{sec:ModelFormulation}
In the rest of the paper we assume, without the loss of generality, that there are only two clusters in the CFL problem  \eqref{eqn: CFL problem}, i.e., $K=2$. Indeed, it will become clear from our discussion below that extending the proposed model and its corresponding theoretical analysis to the case $K>2$ is straightforward.  We also assume that all agents in class $1$ use a single loss function $L_1$ and all agents in class $2$ use a single loss function $L_2$.\footnote{In practice, each agent only has access to finite data samples and thus their empirical loss function may actually differ from that of other agents in the same cluster. We leave the modelling and study of this more realistic and more difficult setting to future work.}
In order to optimize $L_1$ and $L_2$ simultaneously, we consider a collection of $N \in \mathbb{N}$ interacting particles with positions \{$\theta_t^{1, i}\}_{i=1}^{N_1} \in \R^d$ (class $1$ particles) and
$\{\theta_t^{2, j}\}_{j=1}^{N_2} \in \mathbb{R}^d$ (class $2$ particles) described by the system of stochastic differential equations:
\begin{subequations}\label{eqn: particle system}
\begin{equation}\label{eqn: sde for particle 1}
d \theta_t^{1, i} = -\lambda_1 (\theta_t^{1, i} - m_t^1)dt - \lambda_2 \nabla L_1(\theta_t^{1,i})dt + \sigma_1 |\theta_t^{1, i} - m_t^1| dB_t^{1, i} + \sigma_2 |\nabla L_1 (\theta_t^{1,i})| d\widetilde{B}_t^{1,i}
\end{equation}
\begin{equation}\label{eqn: sde for particle 2}
d \theta_t^{2, j} = -\lambda_1 (\theta_t^{2, j} - m_t^2)dt - \lambda_2 \nabla L_2(\theta_t^{2,j})dt + \sigma_1 |\theta_t^{2, j} - m_t^2 |dB_t^{2, j} + \sigma_2 |\nabla L_2 (\theta_t^{2,i})| d\widetilde{B}_t^{2,j}
\end{equation}
\begin{equation}\label{eqn: class 1 consensus point}
m_t^1 := \frac{\sum_{i=1}^{N_1}\theta_t^{1,i} w_{L_1}^{\alpha}(\theta_t^{1,i}) + \sum_{j=1}^{N_2} \theta_t^{2, j} w_{L_1}^{\alpha}(\theta_t^{2, j})}{\sum_{i=1}^{N_1} w_{L_1}^{\alpha}(\theta_t^{1,i}) + \sum_{j=1}^{N_2} w_{L_1}^{\alpha}(\theta_t^{2, j})}
\end{equation}
\begin{equation}\label{eqn: class 2 consensus point}
m_t^2 := \frac{\sum_{i=1}^{N_1}\theta_t^{1,i} w_{L_2}^{\alpha}(\theta_t^{1,i}) + \sum_{j=1}^{N_2} \theta_t^{2, j} w_{L_2}^{\alpha}(\theta_t^{2, j})}{\sum_{i=1}^{N_1} w_{L_2}^{\alpha}(\theta_t^{1,i}) + \sum_{j=1}^{N_2} w_{L_2}^{\alpha}(\theta_t^{2, j})},
\end{equation}
\end{subequations}
with $\lambda_1, \lambda_2, \sigma_1, \sigma_2 > 0$,  $w_{L_k}^{\alpha}(\theta) := \exp(-\alpha L_k(\theta))$ for $k=1,2$, and $\alpha>0$. In the sequel, we may use the terms particle and user interchangeably to refer to an agent. 


Let us discuss the above system term by term. Equation \eqref{eqn: sde for particle 1} describes the time evolution of the model parameters of agent $i$ in class $1$, while equation \eqref{eqn: sde for particle 2} does the same for agent $j$ in class $2$.  The term
$m_t^1$ defined in \eqref{eqn: class 1 consensus point} is a weighted-average of \textit{all} particle positions $\{\theta_t^{1,i}\}_{i=1}^{N_1}, \{\theta_t^{2,j}\}_{j=1}^{N_2}$ with respect to the loss function $L_1$. In particular, an agent in class $1$ \textit{can} compute $m_t^1$ without knowing the class identities of any of the other agents, an essential feature for our purposes. If we imagine for a moment that class $1$ particles concentrate around regions where $L_1$ is small, one should expect that $\{\theta_t^{1,i}\}_{i=1}^{N_1}$ have smaller $L_1$-loss than the class $2$ particles $\{\theta_t^{2,j}\}_{j=1}^{N_2}$, which presumably should concentrate around regions where the loss function $L_2$ is small. Then, intuitively, in the expression for $m_t^1$ class $1$ particles $\{\theta_t^{1,i}\}_{i=1}^{N_1}$ will receive higher weights than class $2$ particles and hence $m_t^1$ should be close to the weighted-average of the particles $\{\theta_t^{1,i}\}_{i=1}^{N_1}$ only. Thus,  $m_t^1$ can be thought of as an evolving consensus point that corresponds to class $1$ particles only. 
A similar intuition holds for $m_t^2$, which is an evolving consensus point for class $2$ particles only. 
The first part of the drift terms in both \eqref{eqn: sde for particle 1} and \eqref{eqn: sde for particle 2} can then be thought of as a consensus-inducing term for each of the classes. The second part of the drift terms in \eqref{eqn: sde for particle 1} and \eqref{eqn: sde for particle 2}, on the other hand, introduces local gradient information, which can be interpreted as a local model update through gradient descent. In federated learning, using local gradient information ensures that all models continue to update even when there is no communication between them; as discussed in \cite{mcmahan2017communication}, communication in federated settings is in general costly. Finally, the diffusion terms in the dynamics guarantee that each particle continues to explore the optimization landscape of its loss function until it reaches a critical point and aligns with its class consensus point. The $d$-dimensional Brownian motions $\{{B}^{k,i} \}_{k,i}$, $\{ \widetilde{B}^{k,i} \}_{k,i}$ are assumed to be independent of each other.

Let us denote by $\bm{\theta}_t^{1, N} := \{\theta_t^{1,1}, \dots, \theta_t^{1, N_1}$\}, $\bm{\theta}_t^{2, N} := \{\theta_t^{2,1}, \dots, \theta_t^{2, N_2}\}$ with $N = N_1 + N_2 \in \mathbb{N}$ the solution of the particle system \eqref{eqn: particle system}. Consider its empirical measures given by
\begin{equation}
\rho_t^{1, N} := \frac{1}{N_1}\sum_{i=1}^{N_1} \delta_{\theta_t^{1, i}}, \qquad
\rho_t^{2, N} := \frac{1}{N_2}\sum_{j=1}^{N_2} \delta_{\theta_t^{2, j}}, \qquad
\rho_t^{N} := \frac{N_1}{N}\rho_t^{1, N} + \frac{N_2}{N}\rho_t^{2, N},
\label{eq:EmpiricalMeasuresParticleSystem}
\end{equation}
where we use $\delta_{\theta}$ to represent a Dirac delta measure at $\theta \in \mathbb{R}^d$. Observe that $m_t^1, m_t^2$ can be rewritten in terms of $\rho_t^N$ as follows:
\begin{equation*}
m_t^1 = \frac{1}{\|w_{L_1}^{\alpha}\|_{\mathbb{L}^1(\rho_t^N)}} \int \theta w_{L_1}^{\alpha} d\rho_t^N =: m_{L_1}^{\alpha}[\rho_t^N], \qquad
m_t^2 = \frac{1}{\|w_{L_2}^{\alpha}\|_{\mathbb{L}^1(\rho_t^N)}} \int \theta w_{L_2}^{\alpha} d\rho_t^N =: m_{L_2}^{\alpha}[\rho_t^N].
\end{equation*}
In turn, system \eqref{eqn: particle system} can be rewritten as
\begin{subequations}
\begin{align*}
d \theta_t^{1, i} &= -\lambda_1 (\theta_t^{1, i} - m_{L_1}^{\alpha}[\rho_t^N])dt - \lambda_2 \nabla L_1(\theta_t^{1,i}) + \sigma_1 |\theta_t^{1, i} - m_{L_1}^{\alpha}[\rho_t^N]| dB_t^{1, i} + \sigma_2 |\nabla L_1 (\theta_t^{1,i})| d\widetilde{B}_t^{1,i}, \\
d \theta_t^{2, j} &= -\lambda_1 (\theta_t^{2, j} - m_{L_2}^{\alpha}[\rho_t^N])dt -\lambda_2 \nabla L_2(\theta_t^{2,j})  + \sigma_1 |\theta_t^{2, j} - m_{L_2}^{\alpha}[\rho_t^N] | dB_t^{2, j} + \sigma_2 |\nabla L_2 (\theta_t^{2,j})| d\widetilde{B}_t^{2,j}.
\end{align*}
\end{subequations}
Based on the above expression, we can formally postulate a mean-field SDE system characterizing the time evolution of particles as $N_1, N_2 \rightarrow \infty$. Precisely, we consider the system of two SDEs:
\begin{subequations}
\label{eqn:mean_field_SDE}
\begin{equation}
d \overline{\theta}_t^1 = - \lambda_1( \overline{\theta}_t^1 - m_{L_1}^{\alpha}[\rho_t]) dt - \lambda_2 \nabla L_1(\overline{\theta}_t^1)dt + \sigma_1 |\overline{\theta}_t^1 - m_{L_1}^{\alpha}[\rho_t] | dB_t^1 + \sigma_2 |\nabla L_1 (\overline{\theta}_t^1)| d\widetilde{B}_t^{1}, \label{eqn: mean-field eqn 1} 
\end{equation}
\begin{equation}
d \overline{\theta}_t^2 = - \lambda_1( \overline{\theta}_t^2 - m_{L_2}^{\alpha}[\rho_t]) dt - \lambda_2 \nabla L_2(\overline{\theta}_t^2)dt + \sigma_1 |\overline{\theta}_t^2 - m_{L_2}^{\alpha}[\rho_t]  | dB_t^2 + \sigma_2 |\nabla L_2 (\overline{\theta}_t^2)| d\widetilde{B}_t^{2}, \label{eqn: mean-field eqn 2}
\end{equation}
\end{subequations}
where
\begin{equation*}
m_{L_k}^{\alpha}[\rho_t] = \int \theta d\eta_{k, t}^{\alpha}, \quad \eta_{k, t}^{\alpha} = w_{L_k}^{\alpha}\rho_t / \|w_{L_k}^{\alpha}\|_{\mathbb{L}^1(\rho_t)}, \quad \rho_t^k = Law(\overline{\theta}_t^k), \quad \rho_t = w_1 \rho_t^1 + w_2 \rho_t^2,
\end{equation*}
subject to the independent initial conditions $\overline{\theta}_0^1 \sim \rho_0^1$ and $\overline{\theta}_0^2 \sim \rho_0^2$; in the above, $B^1, \widetilde{B}^1, B^2,\widetilde{B}^2$ are independent $d$-dimensional Brownian motions. Equations \eqref{eqn: mean-field eqn 1} and \eqref{eqn: mean-field eqn 2} describe, respectively, the effective time evolutions of individual particles of class $1$ and $2$ in the regime of a large number of particles. The weight $w_k$ represents the asymptotic proportion of particles of type $k$ in the system. Finally, $\rho_t^1$ and $\rho_t^2$ represent, respectively, the distributions of particles of class $1$ and $2$ at time $t$. Notice that the system \eqref{eqn:mean_field_SDE} is coupled through the distribution $\rho_t$ of agents of both types. In Section \ref{sec: M-F limit}, we discuss in precise mathematical terms the relationship between the finite system of interacting particles \eqref{eqn: particle system} and the mean-field limit system described in \eqref{eqn:mean_field_SDE}; see also Theorem \ref{thm:mainMeanField} below.

The system of Fokker-Planck equations corresponding to \eqref{eqn:mean_field_SDE} reads:
\begin{subequations}\label{eqn: FP system}
\begin{align}
    \partial_t \rho_t^1 &:= \Delta (\kappa_t^1 \rho_t^1) + \nabla \cdot (\mu_t^1 \rho_t^1), \qquad \lim_{t \rightarrow 0} \rho_t^1 = \rho_0^1 \label{eqn: FP eqn 1}\\
    \partial_t \rho_t^2 &:= \Delta (\kappa_t^2 \rho_t^2) + \nabla \cdot (\mu_t^2 \rho_t^2), \qquad \lim_{t \rightarrow 0} \rho_t^2 = \rho_0^2,\label{eqn: FP eqn 2}
\end{align}
\end{subequations}
where
\begin{equation*}
    \mu_t^{k} := \lambda_1(\theta - m_{L_k}^{\alpha}[\rho_t]) + \lambda_2 \nabla L_k(\theta), \quad \kappa_t^k := \frac{\sigma_1^2}{2} |\theta - m_{L_k}^{\alpha}[\rho_t]|^2 + \frac{\sigma_2^2}{2} |\nabla L_k(\theta)|^2, \quad \text{for $k=1,2$.}
\end{equation*}
This is a non-linear system of equations that describes the time evolution of the distributions of agents of each class in the mean-field limit. Notice that we can indeed describe the law of the joint system of agents in the mean field limit just with the laws of the two marginals $\rho^1$ and $\rho^2$, thanks to the independence of the processes $\overline{\theta}^1$ and $ \overline{\theta}^2$, which we discuss in Remark \ref{rem:IndependenceSDEs}. In section \ref{sec: identification of limit measure}, however, it will be more convenient to work with the law of the joint process $(\overline{\theta}^1_t, \overline{\theta}^2_t)$ explicitly. We do this to facilitate our study of the convergence of the finite particle system toward the mean-field. These details will be presented in due course.     \nc

In the sequel we interpret the Fokker-Planck system \eqref{eqn: FP system} in the weak sense.

\begin{definition}
For $k=1,2$, let $\rho_0^k \in \mathcal{P}(\mathbb{R}^d)$. Let $ T > 0$ be a given time horizon. We say that $\rho^1, \rho^2 \in \mathcal{C}([0, T], \mathcal{P}(\mathbb{R}^d))$ satisfy, in the weak sense, the Fokker-Planck equation \eqref{eqn: FP system} for the time interval $[0, T]$ and with initial conditions $(\rho_0^1,\rho^2_0)$ if for $\forall\phi \in \mathcal{C}_c^{\infty}(\mathbb{R}^d)$, all $\forall t \in (0, T)$ and $k=1,2$, we have
\begin{equation*}
\begin{aligned}
\frac{d}{dt}\int \phi(\theta) d\rho_t^k(\theta) = &\int \left( -\lambda_1 (\theta - m_{L_k}^{\alpha}[\rho_t]) - \lambda_2 \nabla L_k(\theta) \right) \cdot \nabla \phi (\theta) d\rho_t^k(\theta)\\
&+ \int \left( \frac{\sigma_1^2}{2} |\theta - m_{L_k}^{\alpha}[\rho_t]|^2 + \frac{\sigma_2^2}{2} |\nabla L_k(\theta)|^2 \right) \Delta \phi(\theta) d\rho_t^k(\theta),
\end{aligned}
\end{equation*}
and $\lim_{t \rightarrow 0} \rho_t^k = \rho_0^k$ (in the sense of weak convergence of probability measures) for $k=1,2$.
\end{definition}

 Our main theoretical results, presented in the next section, are split into three key theorems. First, we discuss the well-posedness of the mean-field SDE system \eqref{eqn:mean_field_SDE} and its corresponding Fokker-Planck system \eqref{eqn: FP system}. Second, we state a result that establishes the connection between the evolution of empirical measures of the finite particle system \eqref{eqn: particle system} and the Fokker-Planck equation \eqref{eqn: FP system}. Finally, we discuss the long-time behavior properties of the mean-field PDE and show that, under some mild assumptions on initialization and the correct tuning of parameters, each of the distributions $\rho_t^1 $ and $\rho_{t}^2$ concentrates around the global minimizers of $L_1$ and $L_2$, respectively, within a certain time interval. The practical implication of the combination of these theoretical results is the following: by considering the system \eqref{eqn: particle system} with sufficiently large $N_1$ and $N_2$, and assuming appropriate initialization, particles of class $1$ will concentrate around the global minimizer of the loss function $L_1$, while particles of class $2$ will do the same around the global minimizer of $L_2$. In Section \ref{sec: alg}, we use our mathematical model to motivate a new algorithm for cluster-based federated learning. In Section \ref{sec: experiment}, we show through numerical experimentation that the proposed algorithm can indeed produce high-performing learning models for groups of users with similar data sets.

\subsection{Main theoretical results}
\label{sec:MainResults}
In all our theoretical analysis, we make the following assumptions on the loss functions $L_1, L_2$.


\begin{assumption}\label{assump: Lips of L_k}
For $k=1,2$, the loss function $L_k: \mathbb{R}^d \rightarrow \mathbb{R}$ is bounded from below with $\underline{L_k} := \inf L_k$. Moreover, there exist constants $M_{L_k}, C_{L_k}, M_{\nabla L_k}, C_{\nabla L_k} > 0$ such that for all $\theta, \hat{\theta} \in \mathbb{R}^d$,
\begin{subequations}
\begin{align}
|L_k(\theta) - L_k(\hat{\theta})| &\leq M_{L_k} (|\theta| + |\hat{\theta}|)|\theta - \hat{\theta}| \label{eqn: loc lip L_k} \\
L_k(\theta) - \underline{L_k} &\leq C_{L_k}(1 + |\theta|^2) \label{eqn: quadratic growth L_k} \\
|\nabla L_k(\theta) - \nabla L_k(\hat{\theta})| & \leq M_{\nabla L_k} |\theta - \hat{\theta}| \label{eqn: lip grad L_k} \\
|\nabla L_k(\theta)| &\leq C_{\nabla L_k} \label{eqn: bounded grad L_k}
\end{align}
\end{subequations}
\end{assumption}
\noindent {In simple words, in \eqref{eqn: loc lip L_k} and \eqref{eqn: quadratic growth L_k} we assume the loss functions $L_k$ are locally Lipschitz and bounded above by quadratic functions. We also assume the gradients of $L_k$ are Lipschitz and bounded in \eqref{eqn: lip grad L_k} and \eqref{eqn: bounded grad L_k}.}
In addition to Assumption \ref{assump: Lips of L_k}, we consider loss functions $L_k$ that either are (1) bounded from above. In particular, $L_k$ has the upper bound $\overline{L_k}:= \sup L_k$; or (2) quadratic growth at infinity, i.e. there exist constants $M > 0$ and $c_{q_k} > 0$ such that $L_k(\theta) - \underline{L_k} \geq c_{q_k}|\theta|^2$ for all $|\theta| \geq M$.

\begin{theorem}[Well Posedness of mean-field equations]\label{thm: well-posedness of mean-field system}
Let $L_1, L_2$ satisfy Assumption \ref{assump: Lips of L_k} and be  either bounded or have quadratic growth at infinity, and let $\rho_0^1, \rho_0^2 \in \mathcal{P}_4(\mathbb{R}^d)$. Then there exist unique nonlinear processes $\overline{\theta}^1, \overline{\theta}^2 \in \mathcal{C}([0, T], \mathbb{R}^d), T > 0$, satisfying
\begin{subequations}
\begin{align*}
d \overline{\theta}_t^1 &= - \lambda_1( \overline{\theta}_t^1 - m_{L_1}^{\alpha}[\rho_t]) dt - \lambda_2 \nabla L_1(\overline{\theta}_t^1)dt + \sigma_1 |\overline{\theta}_t^1 - m_{L_1}^{\alpha}[\rho_t]  | dB_t^1 + \sigma_2 |\nabla L_1 (\overline{\theta}_t^1)| d\widetilde{B}_t^{1},\\
d \overline{\theta}_t^2 &= - \lambda_1( \overline{\theta}_t^2 - m_{L_2}^{\alpha}[\rho_t]) dt - \lambda_2 \nabla L_2(\overline{\theta}_t^2)dt + \sigma_1 |\overline{\theta}_t^2 - m_{L_2}^{\alpha}[\rho_t]  | dB_t^2 + \sigma_2 |\nabla L_2 (\overline{\theta}_t^2)| d\widetilde{B}_t^{2},\\
\rho_t &= w_1 \rho_t^1 + w_2 \rho_t^2, \qquad w_1 + w_2 = 1
\end{align*}
\end{subequations}
in the strong sense, with $\rho_t^1=\text{Law}(\overline{\theta}_t^1), \rho_t^2=\text{Law}(\overline{\theta}_t^2), \rho_t \in \mathcal{C}([0, T], \mathcal{P}_2(\mathbb{R}^d))$ satisfying the corresponding Fokker-Planck equations (\ref{eqn: FP eqn 1}) and (\ref{eqn: FP eqn 2}) in the weak sense, with $\lim_{t \rightarrow 0} \rho_t^k = \rho_0^k \in \mathcal{P}_2(\mathbb{R}^d)$ for $k=1,2$.
\end{theorem}

Knowing that both the mean-field SDE system \eqref{eqn:mean_field_SDE} and its corresponding Fokker-Planck system are well-posed, we can now state precisely the connection between the finite particle system \eqref{eqn: particle system} and the mean-field system. 
We introduce some notation first. Given $\rho_0^1, \rho_0^2 \in \mathcal{P}(\R^d)$, and given $N_1, N_2 \in \N$, we use $\rho_0^{1 \otimes N_1}$ to denote the product measure of $\rho_0^1 $ with itself $N_1$ times, and, likewise, use $\rho_0^{2 \otimes N_2}$ to denote the product measure of $\rho_0^2$ with itself $N_2$ times.

\begin{theorem}[Convergence to mean-field limit]
\label{thm:mainMeanField}
Let $L_1, L_2$ satisfy Assumption \ref{assump: Lips of L_k} and have quadratic growth at infinity, and suppose that $\rho_0^1, \rho_0^2\in \mathcal{P}_4(\mathbb{R}^d)$. Let $w_1, w_2 > 0$ be such that $w_1 + w_2 = 1$, and let $N_1, N_2 \geq 1$, $N = N_1 + N_2$ be such that $\frac{N_1}{N} \rightarrow w_1$, $\frac{N_2}{N} \rightarrow w_2$ as $N_1, N_2 \rightarrow \infty$. Assume that $\{\bm{\theta}_t^{1, N}\}, \{\bm{\theta}_t^{2, N}\}$ is the unique solution to the particle system \eqref{eqn: particle system} with $\rho_0^{1 \otimes N_1} \otimes \rho_0^{2 \otimes N_2}$-distributed initial data $\{\bm{\theta}_0^{1, N}\}, \{\bm{\theta}_0^{2, N}\}$. 

Then 
for $k=1,2$ and every $\zeta>0$ we have
\[ \lim_{N \rightarrow \infty} \mathbb{P} \left( \sup_{t \in [0,T]} d_{LP} ( \rho_t^{k,N}, \rho_t^k ) \geq \zeta   \right) =0,   \]
where $d_{LP}$ is the Levy-Prokhorov metric between probability measures over $\R^d$, and $t\in [0,T] \mapsto (\rho_t^1, \rho_t^2)$ is the unique solution to \eqref{eqn: FP system} with initial conditions $\rho_0^1$ and $\rho_0^2$.
\nc 
\end{theorem}

\begin{remark}
The implication in Theorem \ref{thm:mainMeanField} is that the empirical measure of positions of particles of type $k$ in the finite particle system converges toward the distribution of positions of particles of type $k$ in the mean-field uniformly in time in probability. 
\end{remark}

For our next result, we impose additional assumptions on the loss functions $L_1, L_2$.

\begin{assumption}\label{assump: global conv assump}
For $k=1,2$, we assume that the function $L_k \in \mathcal{C}(\R^d)$ satisfies
\begin{enumerate}[label=(\Roman*)]
   \item \label{assump: 2.I} there exists $\theta_k^* \in \mathbb{R}^d$ such that $L_k(\theta_k^*) = \inf_{\theta \in \mathbb{R}^d} L_k(\theta) =: \underline{L_k}$,
    \item \label{assump: 2.II} There exist $L_{\infty}^k, R_0^k, \eta_k > 0$, and $\nu_k \in (0, \frac{1}{2} ]$ such that 
    \begin{align}
    |\theta - \theta_k^*| &\leq \frac{1}{\eta_k} \big(L_k(\theta) - \underline{L_k}\big)^{\nu_k} \qquad \text{for all $\theta \in B_{R_0^k}(\theta_k^*)$} \label{eqn: local coercivity},\\
    L_{\infty}^k &< L_k(\theta) - \underline{L_k} \qquad \text{for all $\theta \in \big(B_{R_0^k}(\theta_k^*)\big)^c$}. \label{eqn: farfield assump}
    \end{align}
\end{enumerate}
\end{assumption}
Assumption \ref{assump: global conv assump} is similar to assumptions used in \cite{fornasier2021consensus}.  The first part of the assumption, i.e., Assumption \ref{assump: 2.I}, states that the minimum value $\underline{L_k}$ of the objective function is reached at $\theta_k^*$. The second part, i.e., Assumption \ref{assump: 2.II}, specifies certain required properties of the objective functions' landscapes for our theory to hold.  Specifically, inequality \eqref{eqn: local coercivity} imposes lower bounds on the local growth of $L_k$ around the global minimizer $\theta_k^*$, while condition \eqref{eqn: farfield assump} rules out the possibility that $L_k(\theta) \approx \underline{L_k}$ for some $\theta$ outside a neighborhood of $\theta_k^*$.

\begin{theorem}[Concentration of mean-field around global minimizers]\label{thm: large time behavior}
For $k=1,2$, suppose $L_k \in \mathcal{C}(\R^d)$ satisfy Assumptions \ref{assump: Lips of L_k} and \ref{assump: global conv assump}. Moreover, let $\rho_0^k \in \mathcal{P}_4(\R^d)$ be such that $\rho_0^k (B_r(\theta_k^*)) > 0$ for all $r > 0$. Define $\mathcal{V}(\rho_t^k) := \frac{1}{2} \int |\theta - \theta_k^*|^2 d\rho_t^k(\theta)$. For any $\varepsilon \in (0, \mathcal{V}(\rho_0^1) + \mathcal{V}(\rho_0^2))$, $\tau \in (0, 1)$, parameters $\lambda_1, \lambda_2, \sigma_1, \sigma_2 > 0$ satisfying $2\lambda_1 > 2\lambda_2M + d\sigma_1^2 + d\sigma_2^2 M^2$, where $M:= \max\{M_{\nabla L_1}, M_{\nabla L_2}\}$, and the time horizon
\begin{equation}\label{eqn: time horizon T*}
    T^* := \frac{1}{(1-\tau) (2\lambda_1 - 2\lambda_2 M - d\sigma_1^2 - d\sigma_2^2 M^2)} \log \left(\frac{\mathcal{V}(\rho_0^1) + \mathcal{V}(\rho_0^2)}{\varepsilon} \right),
\end{equation}
there exists $\alpha_0 >0$, which depends on $\varepsilon$ and $\tau$ only, such that for all $\alpha > \alpha_0$, if $\rho^1, \rho^2 \in \mathcal{C}([0, T], \mathcal{P}_4(\R^d))$ are the weak solutions to the Fokker-Planck equations \eqref{eqn: FP eqn 1} and \eqref{eqn: FP eqn 2}, respectively, on the time interval $[0, T^*]$ with initial conditions $\rho_0^1, \rho_0^2$, we have $\min_{t \in [0, T^*]} \big(\mathcal{V}(\rho_t^1) + \mathcal{V}(\rho_t^2)\big) \leq \varepsilon$. Furthermore, up until $\mathcal{V}(\rho_t^1) + \mathcal{V}(\rho_t^2)$ reaches the prescribed accuracy $\varepsilon$ for the first time, we have the exponential decay
\begin{equation}
\mathcal{V}(\rho_t^1) + \mathcal{V}(\rho_t^2) \leq \big(\mathcal{V}(\rho_0^1) + \mathcal{V}(\rho_0^2) \big) \exp\big(-(1 - \tau) (2\lambda_1 - 2\lambda_2 M - d\sigma_1^2 - d\sigma_2^2 M^2)t \big).
\end{equation}
\end{theorem}

\begin{remark}
The parameter $\lambda_1$ in \eqref{eqn: particle system} determines the strength of the force driving particles toward their respective consensus points.  Similarly, $\lambda_2$ and $\sigma_1, \sigma_2$ characterize the strength of the gradient and noise terms, respectively.
In Theorem \ref{thm: large time behavior}, we require the parameters $\lambda_1, \lambda_2, \sigma_1, \sigma_2 > 0$ to satisfy $2\lambda_1 > 2\lambda_2M + d\sigma_1^2 + d\sigma_2^2 M^2$. 
This requirement ensures that the consensus inducing terms dominate the other terms in the dynamics, which is crucial for the system to reach consensus around the global minima of the loss functions. This, however, is an assumption that we impose for theoretical purposes, as in fact a stronger drift toward consensus translates to more communication rounds between agents in applications. Nevertheless, as we will see in our numerical experiments in section \ref{sec: experiment}, our proposed FedCBO algorithm, introduced in the next section, continues to induce consensus among cluster members even when reasonable communication constraints are imposed. 
\end{remark}

\subsection{The FedCBO algorithm}\label{sec: alg}
To implement the system \eqref{eqn: particle system} into a practical algorithm for federated learning, we need to make a series of adjustments. Firstly, we discretize the proposed continuous-time system, this can be done using an Euler-Maruyama discretization of \eqref{eqn: particle system}. Secondly, the resulting discretized scheme must be adapted to fit the conventional federated training protocol where the number of communication rounds among users are restricted. The resulting algorithm, which we name FedCBO, is the combination of these adjustments. We provide more details next.

Let $\{\theta_{0}^{1,i}\}_{i=1}^{N_1}, \{\theta_{0}^{2,j}\}_{j=1}^{N_2}$ be sampled from given distributions $\rho_0^1, \rho_0^2 \in \mathcal{P}(\R^d)$, respectively. Since class memberships are not given, it is reasonable to assume that $\rho_0^1$ and $\rho_0^2$ are the same, but this assumption is not required. Consider the iterates
\begin{subequations}
\begin{align}
\theta_{n+1}^{1, i} &\leftarrow \theta_{n}^{1,i} - \lambda_1 \gamma (\theta_n^{1,i} - m_n^1) - \lambda_2 \gamma \nabla L_1(\theta_n^{1,i}) + \sigma_1 \sqrt{\gamma} |\theta_{n}^{1,i} - m_n^1| z_n^{1,i} + \sigma_2 \sqrt{\gamma} |\nabla L_1(\theta_n^{1,i})| \widetilde{z}_{n}^{1,i}, \label{eqn: class 1 particles discretized}\\
\theta_{n+1}^{2, j} &\leftarrow \theta_{n}^{2,j} - \lambda_1 \gamma (\theta_n^{2,j} - m_n^2) - \lambda_2 \gamma \nabla L_2(\theta_n^{2,j}) + \sigma_1 \sqrt{\gamma} |\theta_{n}^{2,j} - m_n^2| z_n^{2,j} + \sigma_2 \sqrt{\gamma} |\nabla L_2(\theta_n^{2,j})| \widetilde{z}_{n}^{2,j}, \label{eqn: class 2 particles discretized}
\end{align}
\end{subequations}
for $n=0, 1, 2, \dots$. Here, $\gamma$ is the discretization step size; $z_n^{k, i}, \widetilde{z}_n^{k,i}$ for $k=1,2$ are independent normal random vectors $N(0, I_{d \times d})$; $m_n^k, k = 1,2$ are the weighted averages of $\{\theta_{n}^{1,i}\}_{i=1}^{N_1}, \{\theta_{n}^{2,j}\}_{j=1}^{N_2}$ defined by
\begin{subequations}
\begin{align}
m_n^1\big[\{\theta_n^{1,i}\}, \{\theta_n^{2,j}\}\big] &= \frac{\sum_{i=1}^{N_1} \theta_{n}^{1,i} w_{L_1}^{\alpha} (\theta_n^{1,i}) + \sum_{j=1}^{N_2} \theta_n^{2,j} w_{L_1}^{\alpha}(\theta_n^{2,j})}{\sum_{i=1}^{N_1}  w_{L_1}^{\alpha} (\theta_n^{1,i}) + \sum_{j=1}^{N_2} w_{L_1}^{\alpha}(\theta_n^{2,j})}, \label{eqn: class 1 consensus point at time n}\\
m_n^2\big[\{\theta_n^{1,i}\}, \{\theta_n^{2,j}\}\big] &= \frac{\sum_{i=1}^{N_1} \theta_{n}^{1,i} w_{L_2}^{\alpha} (\theta_n^{1,i}) + \sum_{j=1}^{N_2} \theta_n^{2,j} w_{L_2}^{\alpha}(\theta_n^{2,j})}{\sum_{i=1}^{N_1}  w_{L_2}^{\alpha} (\theta_n^{1,i}) + \sum_{j=1}^{N_2} w_{L_2}^{\alpha}(\theta_n^{2,j})} \label{eqn: class 2 consensus point at time n},
\end{align}
\end{subequations}
where $w_{L_k}^{\alpha}(\theta) = \exp(-\alpha L_k(\theta))$.
Given a fixed integer $\tau > 0$, by summing over \eqref{eqn: class 1 particles discretized} and \eqref{eqn: class 2 particles discretized} $\tau$ times,  we can rewrite \eqref{eqn: class 1 particles discretized} and \eqref{eqn: class 2 particles discretized} (omitting noise terms for simplicity\footnote{As mentioned in \cite{carrillo2021consensus}, the impact of noise terms on the algorithm's performance is not significant when training a neural network on the MNIST dataset.}) as
\begin{subequations}
\begin{align}
\theta_{(n+1)\tau}^{1,i} &\leftarrow \theta_{n\tau}^{1,i} - \lambda_1 \gamma \sum_{q=0}^{\tau-1}(\theta_{n\tau + q}^{1,i} - m_{n\tau + q}^1) - \lambda_2 \gamma \sum_{q=0}^{\tau-1} \nabla L_1 (\theta_{n\tau + q}^{1,i}), \label{eqn: class 1 particles original updating formula} \\
\theta_{(n+1)\tau}^{2,j} &\leftarrow \theta_{n\tau}^{2,j} - \lambda_1 \gamma \sum_{q=0}^{\tau-1}(\theta_{n\tau + q}^{2,j} - m_{n\tau + q}^2) - \lambda_2 \gamma \sum_{q=0}^{\tau-1} \nabla L_1 (\theta_{n\tau + q}^{2,j}), \label{eqn: class 2 particles original updating formula}
\end{align}
\end{subequations}
However, the above update rule would require the computation of the consensus points $m_{n\tau + q}^1$ and $m_{n\tau + q}^2$ at each iterate. 
This would result in an excessive amount of communication among users and server, a situation that must be avoided in practical settings. Indeed, note that an agent would need to download the parameters of all other users participating in the training to compute its corresponding $m_{n \tau + q}^k$. If this communication is done too often, it could quickly become prohibitively expensive. 
To accommodate our algorithm to this practical constraint, we consider a splitting scheme that approximates the update formula \eqref{eqn: class 1 particles original updating formula} and \eqref{eqn: class 2 particles original updating formula} in the following way: for $n=0,1,2, \cdots, $
\begin{subequations}\label{eqn: splitting scheme}
\begin{align}
&\widehat{\theta}_{n\tau}^{1,i} \leftarrow \theta_{n\tau}^{1, i}, \quad \widehat{\theta}_{n\tau}^{2,j} \leftarrow \theta_{n\tau}^{2, j},\\
&\widehat{\theta}_{n\tau + q + 1}^{1,i} \leftarrow \widehat{\theta}_{n\tau + q}^{1,i} - \lambda_2 \gamma \nabla L_1(\widehat{\theta}_{n\tau + q}^{1,i}), \quad \widehat{\theta}_{n\tau + q + 1}^{2,j} \leftarrow \widehat{\theta}_{n\tau + q}^{2,j} - \lambda_2 \gamma \nabla L_1(\widehat{\theta}_{n\tau + q}^{2,j}) \qquad \text{for $q=0, \dots, \tau-1$,} \label{eqn: local gradient descent}\\
&\theta_{(n+1)\tau}^{1,i} \leftarrow \widehat{\theta}_{(n+1) \tau}^{1,i} - \lambda_1 \gamma \big(\widehat{\theta}_{(n+1)\tau}^{1,i} - m_{(n+1)\tau}^1 \big), \quad  \theta_{(n+1)\tau}^{2,j} \leftarrow \widehat{\theta}_{(n+1) \tau}^{2,j} - \lambda_1 \gamma \big(\widehat{\theta}_{(n+1)\tau}^{2,j} - m_{(n+1)\tau}^2 \big),\label{eqn: local aggregation step}
\end{align}
\end{subequations}
where the consensus points $m_{(n+1) \tau}^k := m_{(n+1) \tau}^k \big[ \{\widehat{\theta}_{(n+1)\tau}^{1,i}\},\{\widehat{\theta}_{(n+1)\tau}^{2,j}\}\big]$ for $k=1,2$ are the weighted average of $\{\widehat{\theta}_{(n+1)\tau}^{1,i}\}_{i=1}^{N_1}, \{\widehat{\theta}_{(n+1)\tau}^{2,j}\}_{j=1}^{N_2}$ defined as in \eqref{eqn: class 1 consensus point at time n} and \eqref{eqn: class 2 consensus point at time n}.
In simple terms, at each communication round we first update models through gradient descent $\tau$ times \eqref{eqn: local gradient descent} and then compute the consensus points once \eqref{eqn: local aggregation step}.
For the above scheme to resemble \eqref{eqn: class 1 particles original updating formula} as much as possible, we set a larger value for $\lambda_1$ than for $\lambda_2$.
In the standard terminology in federated training, \eqref{eqn: local gradient descent} can be interpreted as a local update of each user's model parameters through $\tau$ epochs of local gradient descent, while \eqref{eqn: local aggregation step} can be viewed as the aggregation step. One interesting feature of our update rules is that the model aggregation does not occur at the global server. Instead, agents may download other user's models and aggregate them through \eqref{eqn: local aggregation step} locally. Thus, the server can be assumed to be completely oblivious to not only class memberships but also to the actual values of all user parameters. This feature makes our FedCBO approach a rather decentralized approach to federated learning.

\begin{algorithm}[htb]
\setstretch{1.25}
\caption{FedCBO}
\label{alg: FedCBO}
\begin{algorithmic}[1]
\REQUIRE
Initialized model $\theta_0^j \in \R^d, j \in [N]$; Number of iterations $T$; Number of local gradient steps $\tau$; Number of models downloaded $M$; CBO system hyperparameters $\lambda_1, \lambda_2, \alpha$; Discretization step size $\gamma$; Initialized sampling likelihood $P_0 \in \R^{N \times (N-1)}$;
\FOR{$n=0, \cdots, T-1$}
\STATE $G_n \leftarrow$ random subset of agents (participating devices);\\
\STATE \textbf{LocalUpdate}($\theta_n^j, \tau, \lambda_2, \gamma$) for $j \in G_n$;\\
\STATE \textbf{LocalAggregation}(agent $j$) for $j \in G_n$;
\ENDFOR
\ENSURE $\theta_T^j$ for $j \in [N]$.

\underline{\textbf{LocalUpdate($\widehat{\theta}_0, \tau, \lambda_2, \gamma$)}} at $j$-th agent\\
\FOR{$q=0, \cdots, \tau - 1$}
\STATE (stochastic) gradient descent $\widehat{\theta}_{q+1} \leftarrow \widehat{\theta}_q - \lambda_2 \gamma \nabla L_j(\widehat{\theta}_q)$;
\ENDFOR
\RETURN $\widehat{\theta}_{\tau}$;
\end{algorithmic}
\end{algorithm}

\begin{algorithm}[ht]
\setstretch{1.50}
\caption{LocalAggregation(agent $j$)}
\label{alg: local aggregation}
\begin{algorithmic}[1]
\REQUIRE Agent $j$'s model $\theta_n^j \in \R^d$; Participating devices at $n$ iteration $G_n$; 
Sampling likelihood $P_n^j \in \R^{N-1}$; 
CBO system hyperparameters $\lambda_1, \alpha$; 
Discretization step size $\gamma$; 
Random sample proportion $\varepsilon \in (0,1)$;
Number of models downloaded $M$;\\
\STATE $A_n \leftarrow$ $\bm{\varepsilon}$\textbf{-greedySampling}($P_n^j, G_n, M$);\\
\STATE Agent $j$ downloads models $\theta_n^i$ for $i \in A_n$;\\
\STATE Evaluate models $\theta_n^i$ on agent $j$'s data set respectively and denote the corresponding loss as $L_j^i$;\\
\STATE Calculate consensus point $m_j$ by
\vspace{-12pt}
\begin{equation}\label{eqn: calculate consensus point}
m_j \leftarrow \frac{1}{\sum_{i \in A_n} \mu_j^i} \sum_{i \in A_n} \theta_n^i \mu_j^i, \qquad \text{with} \;\; \mu_j^i = \exp(-\alpha L_j^i)
\end{equation}
\vspace{-15pt}
\STATE Update agent $j$'s model by
\vspace{-12pt}
\begin{equation}\label{eqn: local aggregation discrete}
\theta_{n+1}^j \leftarrow \theta_n^j - \lambda_1 \gamma (\theta_n^j - m_j),
\end{equation}
\vspace{-15pt}
\STATE Update sampling likelihood $P_n^j$ by
\vspace{-12pt}
\begin{equation}\label{eqn: update sampling likelihood}
P_{n+1}^{j, i} \leftarrow P_{n}^{j, i} + (L_j^j - L_j^i), \qquad \text{for} \;\; i \in A_n
\end{equation}
\vspace{-20pt}
\ENSURE $\theta_{n+1}^j, P_{n+1}^j$

\underline{$\bm{\varepsilon}$-\textbf{greedySampling}}($P_n^j, G_n, M$)\\
\STATE Randomly sample $\varepsilon * M$ number of agents from $G_n$, denoted as $A_n^1$;\\
\STATE Select $(1-\varepsilon) * M$ numbers of agents in $G_n \backslash A_n^1$ with top value $P_j^{j, i}, i \in G_n \backslash A_n^1$, denoted as $A_n^2$;
\RETURN $A_n = A_n^1 \cup A_n^2$
\end{algorithmic}
\end{algorithm}

We are ready to present the FedCBO algorithm (Algorithm \ref{alg: FedCBO}) in precise terms. At the $n$-th iteration of FedCBO, the central server selects a subset of participating agents $G_n \subseteq [N]$; in practice, the server can select this group among the agents that are currently online and available. Each selected agent $j \in G_n$ performs local SGD updates on its model $\theta_n^j$ using its personal data set. After the local update, each participating agent $j \in G_n$ begins local aggregation (Algorithm \ref{alg: local aggregation}). 
In particular, agent $j$ first selects a subset of agents $A_n \subseteq G_n$ using, for example, a $\varepsilon$-greedy sampling strategy \cite{zhang2020personalized} (see Remark \ref{remark: greedy sampling strategy} for details) and then downloads their models {(see Remark \ref{remark: privacy} for a discussion on data privacy vulnerabilities of this and other federated learning schemes).} 
Agent $j$ then evaluates all downloaded models $\theta_j^i, i \in A_n$, on its local dataset and obtains their corresponding losses $L_j^i$. 
Using the losses $L_j^i$, agent $j$ calculates the consensus point $m_j$ following \eqref{eqn: calculate consensus point} and updates its own model $\theta_n^j$ following equation \eqref{eqn: local aggregation discrete}. 
Finally, agent $j$ updates its sampling likelihood vector $P_{n}^j$ according to \eqref{eqn: update sampling likelihood} for future communication rounds. 
As had already been suggested above, the model aggregation step in FedCBO is different from the one in most conventional federated learning algorithms. In FedCBO, models are aggregated locally on each device, whereas conventional federated learning algorithms average models at the global server.

\begin{remark}[$\varepsilon$-greedy sampling strategy]\label{remark: greedy sampling strategy}
For each agent j, we use $\varepsilon$-greedy sampling scheme as in \cite{zhang2020personalized} to select which models to download from other agents. In particular, we maintain a matrix $P$ consisting of row vectors $p^j = (p^{j, 1}, \dots, p^{j, N})$, where $p^{j, i}$ measures the likelihood of agent $j$ downloading model $\theta^i$. Initially, we set $P$ to be the zero matrix, i.e., each model has an equal chance of being selected by any other agent. During each federated iteration, we update $P$ by \eqref{eqn: update sampling likelihood}. Since the number of allowed downloaded models $M$ is much smaller than the total number of agents $N$, we may benefit from extra exploration by randomly selecting $\varepsilon$ proportion of agents and then selecting the remaining proportion of agents based on the top sampling likelihoods according to $P$.
After a few iterations, the likelihood matrix $P$ should become more accurate in identifying similar agents. Therefore, we gradually decrease the value of $\varepsilon$ to control the random exploration rate.
\end{remark}

{ \begin{remark}\label{remark: privacy}
    Given the data privacy constraints motivating federated learning methods, individual agents must not share their local data with the global server or other agents.
    In standard federated training protocols, agents typically exchange either the gradients or the parameters of the models that were trained on their respective local data sets. 
    However, it should be noted that the sharing of gradients is not entirely secure, as it is possible for the private training data to be retrieved  from publicly shared gradients \cite{zhu2020deep,geiping2020inverting,zhao2020idlg,yin2021see,huang2021evaluating,li2022auditing}.
    In contrast, it is more challenging to reconstruct training data information from
    model parameters than it is from gradients, especially when there are limited query times.
    Therefore, in our FedCBO algorithm, it is relatively safe to allow agents to download models of other agents even when there are potential training data inversion attacks. 
\end{remark}
}

\subsection{Experiments}\label{sec: experiment}
In this section, we present an empirical study of our proposed FedCBO algorithm 
and assess its performance in relation to other state of the art learning methodologies designed for the clustered federated learning setting\footnote{Implementation of our experiments is open sourced at \url{https://github.com/SixuLi/FedCBO}.}. 

\textbf{Dataset Setup:} We follow the approach used in \cite{ghosh2020efficient} to create a clustered FL setting that is based on the standard MNIST dataset \cite{lecun1998}. Precisely, we begin with the original MNIST dataset containing $60,000$ training images and $10,000$ test images. We augment this dataset by applying $0$, $90$, $180$, and $270$ degrees of rotation to each image, producing in this way $k=4$ clusters, each of them corresponding to one of the $4$ rotation angles.
For training, we randomly partition the total number of training images $60000k$ into $N$ agent machines so that each agent holds $n = \frac{60000k}{N}$ images, all coming \textit{from the same rotation angle}. For inference, we do not split the test data. Therefore, the model from each local agent will be evaluated on $10000$ rotated test images according to the cluster to which the agent belongs to.
A few examples of the rotated MNIST dataset are shown in Fig. \ref{fig: rotated mnist}.

\begin{figure}[!htb]
    \centering
    \includegraphics[width=0.8\textwidth]{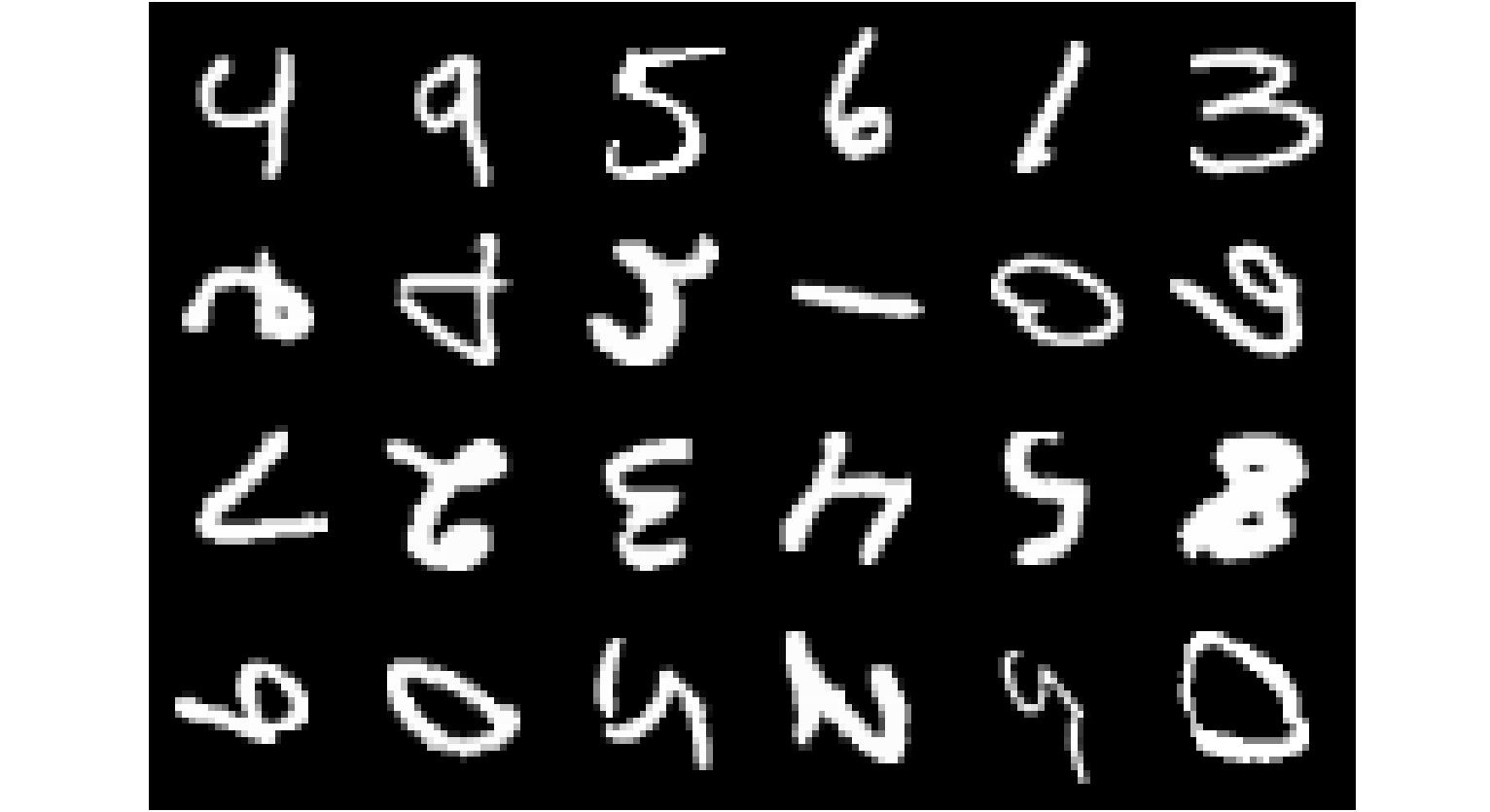}
    \caption{Examples of rotated MNIST dataset. Each row contains a collection of samples from one particular rotation.}
    \label{fig: rotated mnist}
\end{figure}

\textbf{Baselines \& Implementations:} We compare our FedCBO algorithm with three baseline algorithms: IFCA \cite{ghosh2020efficient}, FedAvg \cite{mcmahan2017communication}, and a \textit{local model training} scheme. 
We use fully connected neural networks with a single hidden layer of size $200$ and ReLU activation as base model.
We set the total number of agents $N=1200$ and the number of communication rounds $T=100$.
In each communication round, all agents participate in the training, i.e., $|G_n| = N$ for all $n$.
When an agent trains its own model on its local dataset (local update step in each round), we run $\tau=10$ epochs of stochastic gradient descent (SGD) with a learning rate of $\gamma = 0.1$ and momentum of $0.9$.
In the following, we provide some implementation details for each baseline algorithm:
\begin{itemize}
    \item \textbf{FedCBO:} We set the model download budget $M=200$.
    We choose the hyperparameters $\lambda_1 = 10$, $\lambda_2=1$ and $\alpha=10$.
    For the $\varepsilon$-greedy sampling, we use the decay scheme $\varepsilon(n) = \max\{0.5 - 0.01n, 0.1\}$, i.e., the initial random sample proportion $\varepsilon=0.5$. This parameter is decreased by $0.01$ at each communication round until it reaches the threshold $0.1$.

    \item \textbf{IFCA \cite{ghosh2020efficient}:} We set the number of models initialized at the global server to equal the number of underlying clusters ($k=4$) as suggested in \cite{ghosh2020efficient} for a fair comparison. This should provide the best result for the IFCA algorithm. 

    \item \textbf{FedAvg \cite{mcmahan2017communication}:} The algorithm tries to train a single global model that works for all the local distributions. Hence, in the model aggregation step, the local models trained by the agents are averaged to obtain the updated global model.

    \item \textbf{Local model training:} Agents train their own model using only their local data and with no communication with the global server or to any of the other agents. To ensure a fair comparison, each agent trains its model for a total of $T*\tau = 1000$ epochs.
\end{itemize}

\begin{remark}
Notice that in FedCBO we do not have to input the number of underlying clusters $k$, in contrast to IFCA where we need to input this value or an estimate thereof.  
\end{remark}

For FedCBO and the local model training scheme, we perform inference by testing the trained model on the test data with the same distribution as their training data (i.e., data points with the same rotation). For IFCA and FedAvg, following \cite{ghosh2020efficient}, we run inference on all learned models ($k$ models for IFCA and one model for FedAvg) on each data distribution and calculate the accuracy of the model that produces the smallest loss value.
We conduct experiments with $5$ different random seeds for all the algorithms and report the average accuracy and standard deviation.

\textbf{Experimental Results: } 
The test results are summarized in Table \ref{table: alg comparisons}. 
We observe that our FedCBO algorithm outperforms the three baseline methods.
Although the IFCA algorithm can gradually estimate the cluster identities of users correctly and average over users' models that are estimated to belong to the same clusters, it gives models the same weights during the model aggregation step, thus preventing the algorithm from further utilizing the relative similarities between different models.
In contrast, as we run the FedCBO algorithm, we observe that during the (local) model aggregation steps, agents successfully select models from other users having the same data distribution (as discussed in Remark \ref{remark: SR}) and assign different importance (weights) to the downloaded models using \eqref{eqn: calculate consensus point} during the aggregation steps.
In this way, each user can better utilize the most beneficial models from others.
We attribute this averaging scheme \eqref{eqn: calculate consensus point} as one of the reasons why our FedCBO method outperforms the IFCA algorithm.
As pointed out in \cite{ghosh2020efficient}, the FedAvg baseline performs worse than FedCBO and IFCA as it tries to fit heterogeneous data using one model and thus cannot provide cluster-wise predictions.
Since each agent only stores a small amount of data, the local model training scheme can easily overfit to the local dataset. This explains why it produces the worst performance among all other methodologies.

\begin{table}[!tb]
	\caption{Test accuracy $\pm$ standard deviation \% on rotated MNIST.}
	\label{table: alg comparisons}
	\centering
	\renewcommand\arraystretch{1.2}
	\resizebox{0.8\textwidth}{!}{
		\begin{sc}
			\begin{tabular}{cccc}
				\toprule
                    FedCBO & IFCA & FedAvg & Local \\
				\midrule
				$\bm{96.51 \pm 0.04}$ & $94.44 \pm 0.01$ & $85.50 \pm 0.19$ & $81.27 \pm 0.02$\\
				\bottomrule
			\end{tabular}
		\end{sc}
  }
\end{table}

\begin{remark}\label{remark: SR}
To verify the correctness of the sampling scheme in FedCBO, we define the successful selection rate (SR) for agent $j$ at iteration $n$ as follows:
\begin{equation}
    \text{SR}_n^j := \frac{\text{Number of selected agents in the same cluster as agent $j$}}{\text{Total number of selected agents}},
\end{equation}
where the total number of selected agents equals the model download budget $M$. During the FedCBO algorithm, we calculate the average successful selection rate $\text{SR}_n := \frac{1}{N} \sum_{j=1}^{N} \text{SR}_n^j$ at each communication round $n$, which corresponds to the blue curve in Fig. \ref{fig: SR}.
Meanwhile, when implementing the $\varepsilon$-greedy sampling, we set the random exploration proportion $\varepsilon$ to $0.5$ at $n=0$ and use a decay scheme of $\varepsilon(n) = \max\{0.5 - 0.01 n, 0.1\}$. 
Hence we can calculate the oracle expected successful selection rate at each round: this is shown as the orange curve in Fig. \ref{fig: SR}.
We note that the empirical average successful SR (blue curve) is very close to the best expected successful SR (orange curve). 
This indicates that our FedCBO algorithm can successfully identify the agents with the same data distributions.
{We leave the task of designing better sampling strategies to close the gap between empirical successful SR and oracle successful SR to future work. }
\end{remark}

\begin{figure}[!htb]
    \centering
    \includegraphics[width=0.8\textwidth]{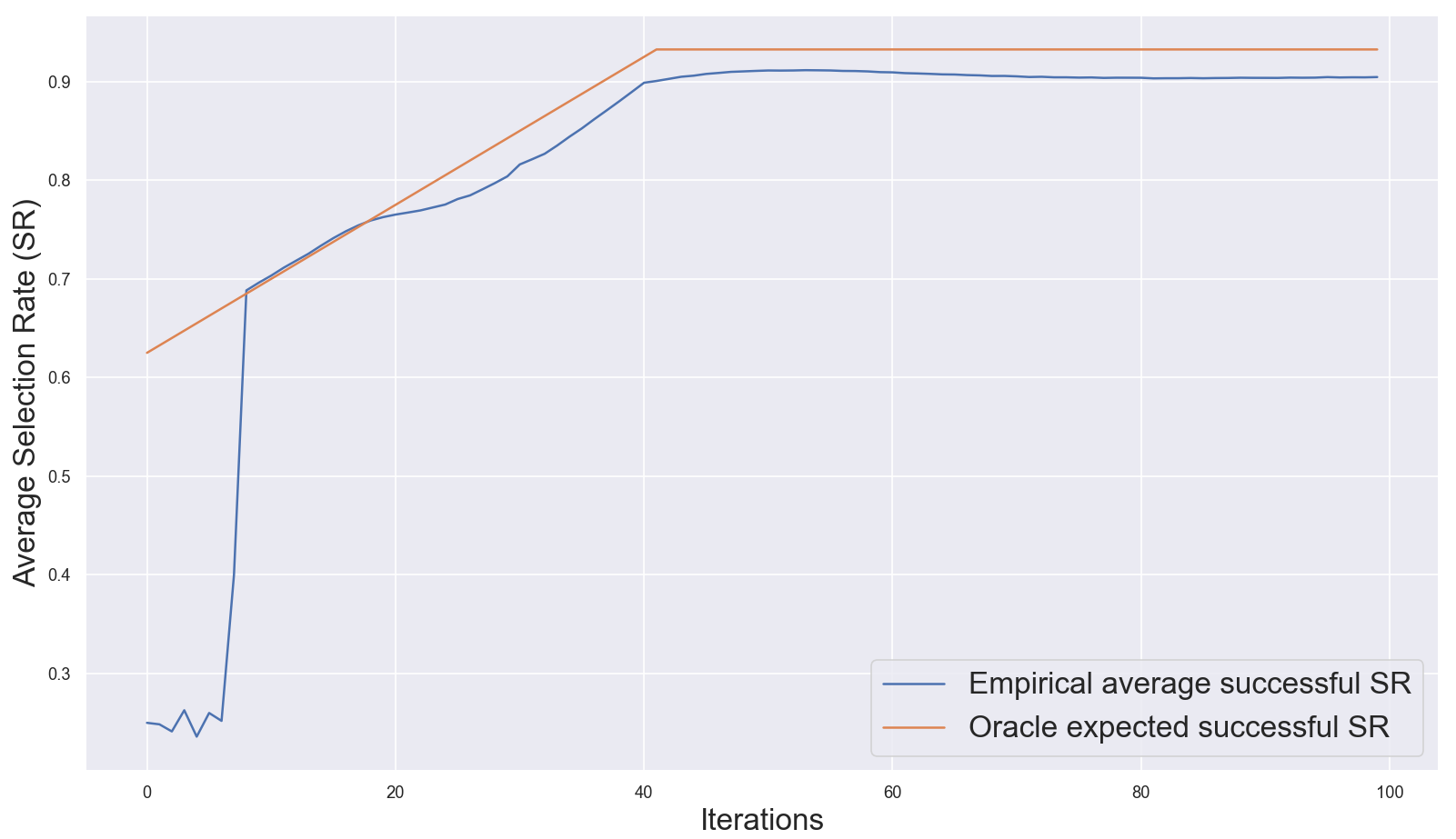}
    \caption{Average successful selection rate (SR) at each communication round.}
    \label{fig: SR}
\end{figure}

\section{Well-posedness of mean-field equations}
\label{sec:WellPosedness}

In this section we prove Theorem \ref{thm: well-posedness of mean-field system}. We present the details in the case in which the loss functions are assumed to be bounded. The proof for the quadratic growth case is similar, and we refer the reader to the Appendix for more details. 

\begin{lemma}\label{lemma: difference between m_Lk}
Let $L_1, L_2$ satisfy Assumption \ref{assump: Lips of L_k} and let $\nu^1, \nu^2 \in \mathcal{C}([0,T], \mathcal{P}_2(\mathbb{R}^d))$ be such that $\sup_{t \in [0,T]}\int |\theta|^4 d\nu_t^1\leq K$, $\sup_{t\in [0,T]}\int |\theta|^4 d\nu_t^2 \leq K$.  Let us denote by $\overline{L}_1, \overline{L}_2$ the supremum of each of the loss functions. Let $w_1, w_2 > 0$ be such that $w_1 + w_2 = 1$, and let $\nu := w_1\nu^1 + w_2\nu^2$. Then, for all $s, t \in (0,T)$, the following stability estimates hold
\begin{equation}
|m_{L_k}^{\alpha}[\nu_t] - m_{L_k}^{\alpha}[\nu_s]| \leq C\bigg(w_1 W_2(\nu_t^1, \nu_s^1) + w_2W_2(\nu_t^2, \nu_s^2)\bigg),
\end{equation}
for $k=1,2$ and for a constant $C$ that depends only on $\alpha, M_{L_1}, M_{L_2}$ and $K$.
\end{lemma}
\begin{proof}
For $k=1,2$ and $s, t \in [0,T]$, one can compute the difference
\begin{equation*}
\begin{aligned}
|m_{L_k}^{\alpha}[\nu_t] - m_{L_k}^{\alpha}[\nu_s]| &= |\frac{\int \theta w_{L_k}^{\alpha} d\nu_t}{\|w_{L_k}^{\alpha}\|_{\mathbb{L}^1(\nu_t)}} - \frac{\int \theta w_{L_k}^{\alpha} d\nu_s}{\|w_{L_k}^{\alpha}\|_{\mathbb{L}^1(\nu_s)}}|\\
&\leq \frac{1}{\|w_{L_k}^{\alpha}\|_{\mathbb{L}^1(\nu_t)}} |\int \theta w_{L_k}^{\alpha}d\nu_t - \int \theta w_{L_k}^{\alpha}d\nu_s| + |\int \theta w_{L_k}^{\alpha} d\nu_s| |\frac{1}{\|w_{L_k}^{\alpha}\|_{\mathbb{L}^1(\nu_t)}} - \frac{1}{\|w_{L_k}^{\alpha}\|_{\mathbb{L}^1(\nu_s)}}|\\
&=: A_1 + A_2.
\end{aligned}
\end{equation*}
To estimate $A_1$, note that
\begin{equation}\label{eqn: estimation of part of A_1}
|\int \theta w_{L_k}^{\alpha}d\nu_t - \int \theta w_{L_k}^{\alpha}d\nu_s| \leq w_1 |\int \theta w_{L_k}^{\alpha}d\nu_t^1 - \int \theta w_{L_k}^{\alpha}d\nu_s^1| + w_2 |\int \theta w_{L_k}^{\alpha}d\nu_t^2 - \int \theta w_{L_k}^{\alpha}d\nu_s^2|.
\end{equation}
Let us take a look at the first term on the right-hand side.
\begin{equation*}
|\int \theta w_{L_k}^{\alpha}(\theta)d\nu_t^1(\theta) - \int \hat{\theta} w_{L_k}^{\alpha}(\hat{\theta})d\nu_s^1(\hat{\theta})| = |\int \int \big[\theta w_{L_k}^{\alpha}(\theta) - \hat{\theta}w_{L_k}^{\alpha}(\hat{\theta})\big]d\pi(\theta, \hat{\theta})|,
\end{equation*}
where $\pi$ is an arbitrary coupling between $\nu_t^1$ and $\nu_s^1$. Denote $h(\theta) - h(\hat{\theta}) := \theta w_{L_k}^{\alpha}(\theta) - \hat{\theta} w_{L_k}^{\alpha}(\hat{\theta})$, then by Assumption \ref{assump: Lips of L_k} on $L_k$, we have
\begin{equation*}
\begin{aligned}
|h(\theta) - h(\hat{\theta})| &= |w_{L_k}^{\alpha}(\theta)(\theta - \hat{\theta}) + \hat{\theta}\big[w_{L_k}^{\alpha}(\theta) - w_{L_k}^{\alpha}(\hat{\theta})\big]|\\
&\leq \exp(-\alpha \underline{L_k})|\theta - \hat{\theta}| + \alpha \exp(-\alpha \underline{L_k}) M_{L_k} |\hat{\theta}|(|\theta| + |\hat{\theta}|)|\theta - \hat{\theta}|.
\end{aligned}
\end{equation*}
Therefore, choosing $\pi$ to be an optimal coupling between $\nu_t^1$ and  $\nu_s^1$ for the $W_2$-distance, we obtain
\begin{equation*}
\begin{aligned}
|\int \theta w_{L_k}^{\alpha}(\theta)d\nu_t^1(\theta) - \int \hat{\theta} w_{L_k}^{\alpha}(\hat{\theta})d\nu_s^1(\hat{\theta})| &\leq \iint \exp(-\alpha \underline{L_k})\big(1 + \alpha M_{L_k}|\hat{\theta}|(|\theta| + |\hat{\theta}|)\big)|\theta - \hat{\theta}|)d\pi\\
\leq &\bigg(\iint  \exp(-2\alpha \underline{L_k})\big(1 + \alpha M_{L_k}|\hat{\theta}|(|\theta| + |\hat{\theta}|)\big)^2 d\pi\bigg)^{\frac{1}{2}}\bigg(\iint |\theta - \hat{\theta}|^2 d\pi\bigg)^{\frac{1}{2}}\\
\leq &\exp(-\alpha \underline{L_k}) \alpha M_{L_k} p_K\bigg(\iint |\theta - \hat{\theta}|^2 d\pi\bigg)^{\frac{1}{2}}\\
= &\exp(-\alpha \underline{L_k}) \alpha M_{L_k} p_K W_2(\nu_t^1, \nu_s^1) \,.
\end{aligned}
\end{equation*}
Here $p_K$ is a polynomial in $K$. Similarly, we deduce
\begin{equation*}
|\int \theta w_{L_k}^{\alpha}d\nu_t^2 - \int \theta w_{L_k}^{\alpha}d\nu_s^2| \leq \exp(-\alpha \underline{L_k}) \alpha M_{L_k} p_K W_2(\nu_t^2, \nu_s^2)\,.
\end{equation*}
Therefore, we conclude
\begin{equation*}
\begin{aligned}
|\int \theta w_{L_k}^{\alpha}d\nu_t - \int \theta w_{L_k}^{\alpha}d\nu_s| &\leq w_1 |\int \theta w_{L_k}^{\alpha}d\nu_t^1 - \int \theta w_{L_k}^{\alpha}d\nu_s^1| + w_2 |\int \theta w_{L_k}^{\alpha}d\nu_t^2 - \int \theta w_{L_k}^{\alpha}d\nu_s^2|\\
&\leq \exp(-\alpha \underline{L_k}) \alpha M_{L_k} p_K\big(w_1 W_2(\nu_t^1, \nu_s^1) + w_2W_2(\nu_t^2, \nu_s^2)\big)\,,
\end{aligned}
\end{equation*}
and
\begin{equation*}
A_1 := \frac{1}{\|w_{L_k}^{\alpha}\|_{\mathbb{L}^1(\nu_t)}} |\int \theta w_{L_k}^{\alpha}d\nu_t - \int \theta w_{L_k}^{\alpha}d\nu_s| \leq \exp(\alpha(\overline{L_k} - \underline{L_k}))\alpha M_{L_k} p_K \big(w_1 W_2(\nu_t^1, \nu_s^1) + w_2W_2(\nu_t^2, \nu_s^2)\big)\,.
\end{equation*}
For the second term $A_2$, we have
\begin{equation*}
A_2:= |\int \theta w_{L_k}^{\alpha} d\nu_s| |\frac{1}{\|w_{L_k}^{\alpha}\|_{\mathbb{L}^1(\nu_t)}} - \frac{1}{\|w_{L_k}^{\alpha}\|_{\mathbb{L}^1(\nu_s)}}| = \frac{\int \theta w_{L_k}^{\alpha}d\nu_s}{\|w_{L_k}^{\alpha}\|_{\mathbb{L}^1(\nu_t)} \|w_{L_k}^{\alpha}\|_{\mathbb{L}^1(\nu_s)}} |\int w_{L_k}^{\alpha}d\nu_t - \int w_{L_k}^{\alpha}d\nu_s|.
\end{equation*}
Similar to the estimation in \eqref{eqn: estimation of part of A_1}, we have
\begin{equation*}
|\int w_{L_k}^{\alpha}d\nu_t - \int w_{L_k}^{\alpha}d\nu_s| \leq \exp(-\alpha\underline{L_k})\alpha M_{L_k}q_K\bigg(w_1W_2(\nu_t^1, \nu_s^1) + w_2W_2(\nu_t^2, \nu_s^2)\bigg),
\end{equation*}
where $q_K$ is a polynomial on $K$. Then,
\begin{equation*}
\begin{aligned}
A_2 &= \frac{\int \theta w_{L_k}^{\alpha}d\nu_s}{\|w_{L_k}^{\alpha}\|_{\mathbb{L}^1(\nu_t)} \|w_{L_k}^{\alpha}\|_{\mathbb{L}^1(\nu_s)}} |\int w_{L_k}^{\alpha}d\nu_t - \int w_{L_k}^{\alpha}d\nu_s|\\
&\leq C_{K_k}\exp(\alpha \overline{L_k}) \exp(-\alpha\underline{L_k})\alpha M_{L_k}q_K\bigg(w_1W_2(\nu_t^1, \nu_s^1) + w_2W_2(\nu_t^2, \nu_s^2)\bigg)\\
&=  C_{K_k}\exp(\alpha(\overline{L_k} - \underline{L_k}))\alpha M_{L_k}q_K\bigg(w_1W_2(\nu_t^1, \nu_s^1) + w_2W_2(\nu_t^2, \nu_s^2)\bigg).
\end{aligned}
\end{equation*}
Combining the estimates for $A_1$ and $A_2$, we obtain
\begin{equation*}
|m_{L_k}^{\alpha}[\nu_t] - m_{L_k}^{\alpha}[\nu_s]| \leq C\bigg(w_1W_2(\nu_t^1, \nu_s^1) + w_2W_2(\nu_t^2, \nu_s^2)\bigg),
\end{equation*}
where $C:= \max \bigg\{ \exp(\alpha(\overline{L_1} - \underline{L_1}))\alpha M_{L_1}(p_K + C_{K_1}q_K), \exp(\alpha(\overline{L_2} - \underline{L_2}))\alpha M_{L_2}(p_K + C_{K_2}q_K)\bigg\}$. 
\end{proof}

\begin{proof}[Proof of Theorem \ref{thm: well-posedness of mean-field system}]
The proof is split into several steps.
\\

\noindent \textbf{Step 1:} For a given pair $u^1, u^2 \in \mathcal{C}([0, T], \mathbb{R}^d)$, we may apply standard theory of SDEs to conclude that there exists a unique strong solution to the SDE
\begin{subequations}
\begin{equation}\label{eqn: mean-field eqn w.r.t u^1}
dY_t^1 = -\lambda_1(Y_t^1 - u_t^1)dt - \lambda_2 \nabla L_1(Y_t^1)dt + \sigma_1 |Y_t^1 - u_t^1| dB_t^1 + \sigma_2 |\nabla L_1 (Y_t^1)|d\widetilde{B}_t^1 
\end{equation}
\begin{equation}\label{eqn: mean-field eqn w.r.t u^2}
dY_t^2 = -\lambda_1(Y_t^2 - u_t^2)dt - \lambda_2 \nabla L_2(Y_t^2)dt + \sigma_1 |Y_t^2 - u_t^2| dB_t^2 + \sigma_2 |\nabla L_2 (Y_t^2)|d\widetilde{B}_t^2,
\end{equation}
\end{subequations}
for initial conditions $Y^1_0 \sim \rho_0^1$ and $Y_0^2 \sim \rho_2$ (independent of each other), where $\rho_0^1, \rho_0^2 \in \mathcal{P}_4(\mathbb{R}^d)$. Let $\nu_t^1 = \text{Law}(Y_t^1)$ and $\nu_t^2 = \text{Law}(Y_t^2)$ be the laws of $Y_t^1$ and $Y_t^2$, respectively. Since the variables $Y^1, Y^2$ take values in $\mathcal{C}([0,T], \mathbb{R}^d)$, it follows that $\nu^1, \nu^2 \in \mathcal{C}([0,T], \mathcal{P}(\mathbb{R}^d))$. Moreover, $\nu^1_t, \nu^2_t$ satisfy the following system of Fokker-Planck equations (in weak form):
\begin{subequations}
\begin{equation}\label{eqn: F-P eqn for u^1}
\frac{d}{dt}\int \varphi d\nu_t^1 = \int \bigg[ \big( -\lambda_1(\theta - u_t^1) - \lambda_2 \nabla L_1(\theta) \big) \cdot \nabla \varphi + \big( \frac{\sigma_1^2}{2} |\theta - u_t^1|^2 + \frac{\sigma_2^2}{2} |\nabla L_1(\theta)|^2 \big) \Delta \varphi \bigg] d\nu_t^1 ,
\end{equation}
\begin{equation}\label{eqn: F-P eqn for u^2}
\frac{d}{dt}\int \varphi d\nu_t^2 = \int \bigg[ \big( -\lambda_1(\theta - u_t^2) - \lambda_2 \nabla L_2(\theta) \big) \cdot \nabla \varphi + \big( \frac{\sigma_1^2}{2} |\theta - u_t^2|^2 + \frac{\sigma_2^2}{2} |\nabla L_2(\theta)|^2 \big) \Delta \varphi \bigg] d\nu_t^2,
\end{equation}
\end{subequations}
for all $\varphi \in \mathcal{C}_c^2(\mathbb{R}^d)$. Let us consider the product space $\mathcal{C}([0, T], \mathbb{R}^d) \times \mathcal{C}([0, T], \mathbb{R}^d)$ endowed with the norm $\|\cdot\|$ defined as
\begin{equation*}
    \|(f, g)\|:= \|f\|_{\infty} + \|g\|_{\infty},
\end{equation*}
where $f, g \in \mathcal{C}([0, T], \mathbb{R}^d)$ and $\|f\|_{\infty} := \sup_{t \in [0,T]} |f_t|$. Let $P_k: \mathcal{C}([0, T], \mathbb{R}^d) \times \mathcal{C}([0, T], \mathbb{R}^d) \rightarrow \mathcal{C}([0, T], \mathbb{R}^d) $ be the projection onto the $k$-th coordinate for $k=1,2$. That is, for any $(f, g) \in \mathcal{C}([0,T], \mathbb{R}^d) \times \mathcal{C}([0,T], \mathbb{R}^d)$, $P_1 (f,g) = f$ and $P_2(f,g) = g$.  Notice that $m_{L_k}[\nu] \in \mathcal{C}([0,T], \mathbb{R}^d)$ for $k=1,2$ and
thus we can define the map
\begin{equation*}
    \begin{aligned}
    \mathcal{T}: \mathcal{C}([0,T], \mathbb{R}^d) \times \mathcal{C}([0,T], \mathbb{R}^d) &\longrightarrow \mathcal{C}([0,T], \mathbb{R}^d) \times \mathcal{C}([0,T], \mathbb{R}^d),\\
    (u^1, u^2) &\longrightarrow \mathcal{T}(u^1, u^2) = (m_{L_1}^{\alpha}[\nu], m_{L_2}^{\alpha}[\nu]),
    \end{aligned}
\end{equation*}
where $\nu = w_1\nu^1 + w_2 \nu^2$.

Next, we show that $\mathcal{T}$ has a unique fixed point.\\ 

\noindent \textbf{Step 2:} First, we show that $\mathcal{T}$ is compact, i.e., any bounded sequence $\{(f_n, g_n)\}_{n \in \N}$ in $\mathcal{C}([0,T], \mathbb{R}^d) \times \mathcal{C}([0,T], \mathbb{R}^d)$ is precompact. 
Let $(\varphi_n, \psi_n) := \mathcal{T}(f_n, g_n)$.
It is sufficient to show that each of the sequences $\{ \varphi_n\}_{n \in \N}$ and $\{ \psi_n\}_{n \in \N}$ is precompact. We show the details for the sequence $\{ \varphi_n\}_{n \in \N}$. 
Since $\rho_0^1, \rho_0^2 \in \mathcal{P}_4(\mathbb{R}^d)$, standard theory of SDEs (see, e.g., Chapter $7$ of \cite{arnold1974stochastic}) provides a fourth-order moment estimate for solutions to \eqref{eqn: mean-field eqn w.r.t u^1} and \eqref{eqn: mean-field eqn w.r.t u^2} of the form
\begin{equation*}
\mathbb{E}|Y_{t}^1|^{\red 4 \nc } \leq (1 + \mathbb{E}|Y_{0}^1|^4)\exp(ct)
\qquad \text{and} \qquad 
\mathbb{E}|Y_{t}^2|^{\red 4} \leq (1 + \mathbb{E}|Y_{0}^2|^4)\exp(ct),
\end{equation*}
for some constant $c > 0$ only depending on $\|u^1\|_{\infty}$ and $\|u^2\|_{\infty}$. In particular,
$$
\sup_{t \in [0, T]} \int |x|^4 d\nu^1_t , \sup_{t \in [0, T]} \int |x|^4 d\nu^2_t \leq K
$$ 
for some $K < \infty$. On the other hand, for any $t > s, t,s \in (0,T)$, the Itô isometry and Cauchy-Schwarz inequality yield
\begin{equation*}
\begin{aligned}
\mathbb{E} |Y_t^1 - Y_s^1|^2 &= \mathbb{E} \bigg[\int_{s}^{t} \big( -\lambda_1 (Y_{\tau}^1 - u_{\tau}^1) - \lambda_2 \nabla L_1 (Y_{\tau}^1) \big) d\tau + \int_s^t \sigma_1 |Y_{\tau}^1 - u_{\tau}^1| dB_{\tau}^1 + \int_s^t \sigma_2 |\nabla L_1 (Y_{\tau}^1)| d\widetilde{B}_{\tau}^1 \bigg]^2\\
&\leq  2\mathbb{E} \bigg[ \int_s^t \big( -\lambda_1 (Y_{\tau}^1 - u_{\tau}^1) - \lambda_2 \nabla L_1(Y_{\tau}^1) \big) d\tau \bigg]^2 + 2\mathbb{E} \bigg[ \int_s^t \big( \sigma^2 |Y_{\tau}^1 - u_{\tau}^1|^2 + \sigma_2^2 |\nabla L_1 (Y_{\tau}^1)|^2 \big) d\tau \bigg]\\
&\leq 2\lambda_1^2 |t-s| \mathbb{E} \big[\int_s^t |Y_{\tau}^1 - u_{\tau}^1|^2 d\tau \big] + 2\lambda_2^2 |t-s| \mathbb{E} \big[\int_s^t |\nabla L_1(Y_{\tau}^1)|^2 d\tau \big]\\
&\quad + \mathbb{E} \bigg[ \int_s^t \big( \sigma^2 |Y_{\tau}^1 - u_{\tau}^1|^2 + \sigma_2^2 |\nabla L_1 (Y_{\tau}^1)|^2 \big) d\tau \bigg]\\
&\leq 4\lambda_1^2(K + \|u^1\|_{\infty}^2)T |t-s| + 2\lambda_2^2 C_{\nabla L_1}^2 T |t-s| + 2\sigma_1^2 (K + \|u^1\|_{\infty}^2) |t-s| + \sigma_2^2 C_{\nabla L_1}^2 |t-s|\\
&=: C_1|t-s|.
\end{aligned}
\end{equation*}
Similarly, $\mathbb{E} |Y_t^2 - Y_s^2|^2 \leq C_2 |t-s|$ for a constant $C_2 > 0$. Therefore, $W_2(\nu_t^1, \nu_s^1) \leq c_1|t-s|^{\frac{1}{2}}$ and $W_2(\nu_t^2, \nu_s^2) \leq c_2|t-s|^{\frac{1}{2}}$, for some constants $c_1, c_2 >0$ only depending on $\|u^1\|_{\infty}$ and $\|u^2\|_{\infty}$. Applying Lemma \ref{lemma: difference between m_Lk}, we obtain
\begin{equation*}
    |m_{L_1}^{\alpha}[\nu_t] - m_{L_1}^{\alpha}[\nu_s]| \leq C(w_1c_1 + w_2c_2)|t-s|^{\frac{1}{2}},
\end{equation*}
which proves that  $t \rightarrow m_{L_1}^{\alpha}[\nu_t]$ is Hölder continuous with exponent $\frac{1}{2}$. From this we can conclude that $\{\varphi_n\}_{n \in \N}$ is precompact due to the compact embedding $\mathcal{C}^{0, 1/2}([0, T], \mathbb{R}^d) \hookrightarrow \mathcal{C}([0, T], \mathbb{R}^d)$, where $\mathcal{C}^{0, 1/2}([0, T], \mathbb{R}^d)$ is the space of $\frac{1}{2}$-H\"{o}lder continuous functions from $[0,T]$ into $\R^d$.\\

\noindent \textbf{Step 3:} Next, we verify the conditions in the Leray-Schauder fixed point theorem. For that purpose, suppose the pair $(u^1, u^2) \in \mathcal{C}([0, T], \mathbb{R}^d) \times \mathcal{C}([0,T], \mathbb{R}^d)$ satisfies $(u^1, u^2) = \tau \mathcal{T}(u^1, u^2)$ for some $\tau \in [0,1]$. In particular, there exist $\nu_1, \nu_2 \in \mathcal{C}([0,T], \mathcal{P}_2(\mathbb{R}^d))$ satisfying  \eqref{eqn: F-P eqn for u^1} and \eqref{eqn: F-P eqn for u^2}, respectively, such that $(u^1, u^2) = \tau (m_{L_1}^{\alpha}[\nu], m_{L_2}^{\alpha}[\nu])$, where $\nu = w_1 \nu_1 + w_2\nu_2$. Due to the boundedness assumption on $L_1$, we have for all $t \in (0, T)$
\begin{equation}\label{eqn: estimation of u_1t bounded case}
|u_{t}^1|^2 = \tau^2 |m_{L_1}[\nu]|^2 \leq \tau^2 \exp(\alpha(\overline{L_1} - \underline{L_1}))\int |x|^2d\nu_t = \tau^2 \exp(\alpha(\overline{L_1} - \underline{L_1})) \bigg(w_1 \int |x|^2d\nu_t^1 + w_2 \int |x|^2d\nu_t^2\bigg).
\end{equation}
A computation of the second moment of $\nu_t^1$ gives
\begin{equation*}
\begin{aligned}
\frac{d}{dt} \int |\theta|^2 d\nu_t^1 &= \int \bigg[ \big( -\lambda_1 (\theta - u_t^1) - \lambda_2 \nabla L_1 (\theta) \big) \cdot 2\theta + 2d\big( \frac{\sigma_1^2}{2} |\theta - u_t^1|^2 + \frac{\sigma_2^2}{2} |\nabla L_1(\theta)|^2 \big) \bigg] d\nu_t^1\\
&= \int \bigg[ -2\lambda_1 |\theta|^2 + 2\lambda_1 \theta \cdot u_t^1 - 2\lambda_2 \theta \cdot \nabla L_1(\theta) + d\sigma_1^2 (|\theta|^2 - 2\theta \cdot u_t^1 + |u_t^1|^2) + d\sigma_2^2 |\nabla L_1(\theta)|^2 \bigg] d\nu_t^1\\
&\leq \int \bigg[(d\sigma_1^2 - 2\lambda_1 + |\gamma| + \lambda_2)|\theta|^2 + (d\sigma_1^2 + |\gamma|) |u_t^1|^2 + (\lambda_2 + d\sigma_2^2) |\nabla L_1(\theta)|^2 \bigg] d\nu_t^1\\
&\leq (d\sigma_1^2 - 2\lambda_1 + |\gamma| + \lambda_2) \int |\theta|^2 d\nu_t^1 + (d\sigma_1^2 + |\gamma|) |u_t^1|^2 + (\lambda_2 + d\sigma_2^2) C_{\nabla L_1}^2\\
&\leq (d\sigma_1^2 + |\gamma| + \lambda_2)\big(1 + \exp(\alpha(\overline{L_1} - \underline{L_1})) \big) \bigg(\int |\theta|^2 d\nu_t^1 + \int |\theta|^2 d\nu_t^2 \bigg) + (\lambda_2 + d\sigma_2^2) C_{\nabla L_1}^2,
\end{aligned}
\end{equation*}
where $\gamma := \lambda_1 - d\sigma_1^2$. Similarly,
\begin{equation*}
\frac{d}{dt} \int |\theta|^2 d\nu_t^2 \leq (d\sigma_1^2 + |\gamma| + \lambda_2)\big(1 + \exp(\alpha(\overline{L_2} - \underline{L_2})) \big) \bigg(\int |\theta|^2 d\nu_t^1 + \int |\theta|^2 d\nu_t^2 \bigg) + (\lambda_2 + d\sigma_2^2) C_{\nabla L_2}^2.
\end{equation*}
Adding the above two inequalities we conclude that
\begin{equation*}
\frac{d}{dt} \bigg(\int |\theta|^2 d\nu_t^1 + \int |\theta|^2 d\nu_t^2 \bigg) \leq C_1 \bigg(\int |\theta|^2 d\nu_t^1 + \int |\theta|^2 d\nu_t^2 \bigg) + C_2,
\end{equation*}
for some constants $C_1, C_2 > 0$. Using Grönwall's inequality we obtain
\begin{equation*}
\int |\theta|^2 d\nu_t^1 + \int |\theta|^2 d\nu_t^2 \leq \bigg(\int |\theta|^2 d\nu_0^1 + \int |\theta|^2 d\nu_0^2 \bigg) \exp(C_1 t) + \frac{C_2}{C_1}\big(\exp(C_1 t) -1\big).
\end{equation*}
Then, from \eqref{eqn: estimation of u_1t bounded case}, we conclude that there is a constant $q_1 > 0$ such that $\|u^1\|_{\infty} < q_1$. A similar bound  holds for $\|u^2\|_{\infty}$, i.e., there is a constant $q_2 > 0$ such that $\|u^2\|_{\infty} < q_2$. Hence, $\|(u^1, u^2)\| = \|u^1\|_{\infty} + \|u^2\|_{\infty} < q_1 + q_2$. We may now invoke the Leray-Schauder fixed point theorem (Section 9.2 in \cite{evans2010partial}) to conclude that there exists a fixed point $(u^1, u^2)$ for the mapping $\mathcal{T}$ and thereby a solution of \eqref{eqn: mean-field eqn w.r.t u^1} and \eqref{eqn: mean-field eqn w.r.t u^2}.\\

\noindent \textbf{Step 4:} As for uniqueness, we first note that a fixed point $(u^1, u^2)$ of $\mathcal{T}$ must satisfy $\|(u^1, u^2)\| < q$. Hence, the fourth-order moment estimates provided in \textbf{Step 2} hold and $\sup_{t \in [0,T]} \int |x|^4d\nu_t^k \leq K < \infty$ for $k=1,2$. Now suppose we have two fixed points $(u^1, u^2)$ and $(\hat{u}^1, \hat{u}^2)$ with
\begin{equation*}
\|(u^1, u^2)\|, \|(\hat{u}^1, \hat{u}^2)\| < q, \qquad \sup_{t \in [0,T]} \int |x|^4d\nu_t^k, \sup_{t \in [0,T]} \int |x|^4d\hat{\nu}_t^k \leq K \quad \text{for $k=1,2$,}
\end{equation*}
and consider their corresponding processes $(Y_{t}^1, Y_{t}^2), (\hat{Y}_{t}^1, \hat{Y}_{t}^2)$, which satisfy \eqref{eqn: mean-field eqn w.r.t u^1} and \eqref{eqn: mean-field eqn w.r.t u^2} with the same Brownian motions for both $k=1,2$. Taking the differences $z_{t}^k := Y_{t}^k - \hat{Y}_{t}^k$ for $k=1,2$, we obtain
\begin{equation*}
\begin{aligned}
z_t^k &= z_0^k + \int_0^t \big( -\lambda_1 z_s^k + \lambda_1 (u_s^k - \hat{u}_s^k) - \lambda_2 ( \nabla L_k(Y_s^k) - \nabla L_k(\hat{Y}_s^k) ) \big) ds \\
&\quad + \sigma_1 \int_0^t \big(|Y_s^k - u_s^k| - |\hat{Y}_s^k - \hat{u}_s^k| \big) dB_s^k + \sigma_2 \int_0^t \big(|\nabla L_k (Y_s^k)| - |\nabla L_k(\hat{Y}_s^k)| \big) d\widetilde{B}_s^k .
\end{aligned}
\end{equation*}
Squaring both sides, taking expectations, and using Itô isometry we obtain 
\begin{equation*}
\begin{aligned}
\mathbb{E}|z_t^k|^2 &= \mathbb{E} \bigg[ z_0^k + \int_0^t \big( -\lambda_1 z_s^k + \lambda_1 (u_s^k - \hat{u}_s^k) - \lambda_2 ( \nabla L_k(Y_s^k) - \nabla L_k(\hat{Y}_s^k) ) \big) ds \\
&\qquad + \sigma_1 \int_0^t \big(|Y_s^k - u_s^k| - |\hat{Y}_s^k - \hat{u}_s^k| \big) dB_s^k + \sigma_2 \int_0^t \big(|\nabla L_k (Y_s^k)| - |\nabla L_k(\hat{Y}_s^k)| \big) d\widetilde{B}_s^k \bigg]^2\\
&\leq 2 \mathbb{E}|z_0^k|^2 + 2t \mathbb{E} \bigg[ \int_0^t \big( -\lambda_1 z_s^k + \lambda_1 (u_s^k - \hat{u}_s^k) -\lambda_2 (\nabla L_k (Y_s^k) - \nabla L_k (\hat{Y}_s^k)) \big)^2 ds \bigg]\\
&\quad + 2\sigma_1^2 \mathbb{E} \bigg[ \int_0^t \big(|Y_s^k - u_s^k| - |\hat{Y}_s^k - \hat{u}_s^k| \big)^2 ds \bigg] + 2\sigma_2^2 \mathbb{E}\bigg[\int_0^t \big( |\nabla L_k(Y_s^k)| - |\nabla L_k(\hat{Y}_s^k)| \big)^2 ds \bigg]\\
&\leq 2 \mathbb{E}|z_0^k|^2 + 6t\lambda_1^2 \int_0^t \mathbb{E}|z_s^k|^2 ds + 6t \lambda_1^2 \int_0^t |u_s^k - \hat{u}_s^k|^2 ds + 6t \lambda_2^2 \int_0^t \mathbb{E} \big[ \nabla L_k (Y_s^k) - \nabla L_k (\hat{Y}_s^k) \big]^2 ds \\
&\quad + 2\sigma_1^2 \E \bigg[\int_0^t |(Y_s^k - \hat{Y}_s^k) - (u_s^k - \hat{u}_s^k)|^2 ds \bigg] + 2\sigma_2^2 \E \bigg[ \int_0^t |\nabla L_k(Y_s^k) - \nabla L_k(\hat{Y}_s^k)|^2 ds \bigg]\\
&\leq 2 \mathbb{E}|z_0^k|^2 + (6\lambda_1^2 t + 4\sigma_1^2) \int_0^t \mathbb{E}|z_s^k|^2 ds + (6\lambda_1^2 t + 4\sigma_1^2) \int_0^t |u_s^k - \hat{u}_s^k|^2 ds\\
&\quad + 6\lambda_2^2 M_{\nabla L_1}^2 \int_0^t \mathbb{E}|z_s^k|^2 ds + 2\sigma_2^2 M_{\nabla L_1}^2 \int_0^t \mathbb{E}|z_s^k|^2 ds\\
&= 2 \mathbb{E}|z_0^k|^2 + \big[ 6\lambda_1^2 t + 4\sigma_1^2 + M_{\nabla L_1}^2 (6\lambda_2^2 + 2\sigma_2^2) \big] \int_0^t \mathbb{E}|z_s^k|^2 ds + (6\lambda_1^2 t + 4\sigma_1^2) \int_0^t |u_s^k - \hat{u}_s^k|^2 ds .
\end{aligned}
\end{equation*}
By Lemma \ref{lemma: difference between m_Lk}, for $k=1,2$, we get
\begin{equation}\label{eqn: bound for difference between m_Lk}
\begin{aligned}
|u_s^k - \hat{u}_s^k|^2 = |m_{L_k}[\nu_s] - m_{L_k}[\hat{\nu}_s]|^2 
&\leq C^2 \big( w_1 W_2(\nu_s^1, \hat{\nu}_s^1) + w_2 W_2 (\nu_s^2, \hat{\nu}_s^2) \big)^2\\
&\leq C^2 \big( w_1 \sqrt{\E |z_s^1|^2} + w_2 \sqrt{\E |z_s^2|^2} \big)^2\\
&\leq 2C^2\big( \E |z_s^1|^2 + \E |z_s^2|^2 \big) .
\end{aligned}
\end{equation}
We further obtain
\begin{align*}
\E |z_t^1|^2 &\leq 2\E |z_0^1|^2 + \widetilde{C}_{1} \int_0^t \E|z_s^1|^2 ds + \widetilde{C}_2 \int \E|z_s^2|^2 ds \\
\E |z_t^2|^2 &\leq 2\E |z_0^2|^2 + \widetilde{C}_{1} \int_0^t \E|z_s^2|^2 ds + \widetilde{C}_2 \int \E|z_s^1|^2 ds,
\end{align*}
where $\widetilde{C}_1 = (1 + 2C^2)(6\lambda_1^2 t + 4 \sigma_1^2) + (6\lambda_2^2 + 2\sigma_2^2) \max\{M_{\nabla L_1}^2, M_{\nabla L_2}^2\}$ and $\widetilde{C}_2 = 2C^2 (6\lambda_1^2 t + 4\sigma_1^2)$. Combining the above two inequalities we deduce
\begin{equation*}
\E|z_t^1|^2 + \E|z_t^2|^2 \leq 2\big(\E|z_0^1|^2 + \E|z_0^2|^2 \big) + (\widetilde{C}_1 + \widetilde{C}_2) \int_0^t \big( \E|z_s^1|^2 + \E|z_s^2|^2\big) ds .
\end{equation*}
Then, by Grönwall's inequality and the fact that $\E|z_0^1|^2 = \E|z_0^2|^2 = 0$, we infer that $\E|z_t^1|^2 + \E|z_t^2|^2 = 0$ for all $t \in [0, T]$. From inequality \eqref{eqn: bound for difference between m_Lk}, we obtain $\|u^1 - \hat{u}^1\|_{\infty} = \|u^2 - \hat{u}^2\|_{\infty} = 0$, i.e., $(u^1, u^2) \equiv (\hat{u}^1, \hat{u}^2)$, proving in this way the uniqueness.
\end{proof}

\begin{remark}
\label{rem:IndependenceSDEs}
Notice that the stochastic processes $Y^1$ and $Y^2$ in \eqref{eqn: mean-field eqn w.r.t u^1} and \eqref{eqn: mean-field eqn w.r.t u^2} are independent from each other for any input functions $(u^1, u^2)$. In turn, since \eqref{eqn: mean-field eqn 1} and \eqref{eqn: mean-field eqn 2} are realized as $(Y^1, Y^2)$ for a specific choice of $(u^1, u^2)$ (i.e., for a fixed point of $\mathcal{T}$), we conclude that the processes $\overline{\theta}^1, \overline{\theta}^2$ from \eqref{eqn: mean-field eqn 1} and \eqref{eqn: mean-field eqn 2} are independent as stochastic processes. Notice, however, that both SDEs share parameters, e.g., the distribution $\rho$ appearing in both the drift and diffusion terms of the equations.
\end{remark}

\section{Large time behavior of mean-field equation and consensus formation}\label{sec: long term behavior}

\begin{proof}[Proof of Theorem \ref{thm: large time behavior}]
Let $M:= \max \{ M_{\nabla L_1}, M_{\nabla L_2}\}$. Using Lemma \ref{lemma: evolution of the variances} we get
\begin{equation}\label{eqn: upper bound for V1+V2}
\begin{aligned}
\frac{d}{dt} \big( \mathcal{V}(\rho_t^1) + \mathcal{V}(\rho_t^2) \big) &\leq -(2\lambda_1 - 2\lambda_2 M - d\sigma_1^2 - d\sigma_2^2 M^2) \big(\mathcal{V}(\rho_t^1) + \mathcal{V}(\rho_t^2) \big)\\
&\quad + \sqrt{2} (\lambda_1 + d\sigma_1^2) \bigg( \sqrt{\mathcal{V}(\rho_t^1)} |m_{L_1}^{\alpha}[\rho_t] - \theta_1^*| + \sqrt{\mathcal{V}(\rho_t^2)} |m_{L_2}^{\alpha}[\rho_t] - \theta_2^*| \bigg) \\
&\quad + \frac{d\sigma_1^2}{2} \big( |m_{L_1}^{\alpha}[\rho_t] - \theta_1^*|^2 + |m_{L_2}^{\alpha}[\rho_t] - \theta_2^*|^2 \big)\\
&\leq -(2\lambda_1 - 2\lambda_2 M - d\sigma_1^2 - d\sigma_2^2 M^2) \big(\mathcal{V}(\rho_t^1) + \mathcal{V}(\rho_t^2) \big)\\
&\quad + \sqrt{2} (\lambda_1 + d\sigma_1^2) \big( |m_{L_1}^{\alpha}[\rho_t] - \theta_1^*| + |m_{L_2}^{\alpha}[\rho_t] - \theta_2^*| \big) \sqrt{\mathcal{V}(\rho_t^1) + \mathcal{V}(\rho_t^2)} \\
&\quad + \frac{d\sigma_1^2}{2} \big( |m_{L_1}^{\alpha}[\rho_t] - \theta_1^*|^2 + |m_{L_2}^{\alpha}[\rho_t] - \theta_2^*|^2 \big).
\end{aligned}
\end{equation}
Let $T_{\alpha} \geq 0$ be given by
\begin{equation}\label{eqn: time horizon T}
T_{\alpha}:= \sup\big\{ t\geq 0: \mathcal{V}(\rho_{t'}^1) + \mathcal{V}(\rho_{t'}^2) > \varepsilon, |m_{L_1}^{\alpha}[\rho_{t'}] - \theta_1^*| + |m_{L_2}^{\alpha}[\rho_{t'}] - \theta_2^*| < C(t') \;\, \forall t' \in [0, t]\big\},
\end{equation}
where 
\begin{equation}\label{eqn: auxilary function C_1(t), C_2(t)}
C(t):= C \sqrt{\mathcal{V}(\rho_t^1) + \mathcal{V}(\rho_t^2)}
\end{equation}
with
\begin{equation}\label{eqn: aux constant C}
C:= \min \bigg\{\frac{\tau}{2} \frac{(2\lambda_1 - 2\lambda_2 M - d\sigma_1^2 - d\sigma_2^2 M^2)}{\sqrt{2} (\lambda_1 + d\sigma_1^2)}, \sqrt{\tau \frac{(2\lambda_1 - 2\lambda_2 M - d\sigma_1^2 - d\sigma_2^2 M^2)}{d\sigma_1^2}} \bigg\}.
\end{equation}
Then, combining \eqref{eqn: upper bound for V1+V2} with \eqref{eqn: time horizon T}, for all $t \in [0, T_{\alpha}]$ we have
\begin{equation}\label{eqn: time-evolution bound}
\frac{d}{dt} \big(\mathcal{V}(\rho_t^1) + \mathcal{V}(\rho_t^2) \big) \leq -(1-\tau) (2\lambda_1 - 2\lambda_2 M - d\sigma_1^2 -d\sigma_2^2 M^2) \big(\mathcal{V}(\rho_t^1) + \mathcal{V}(\rho_t^2) \big) < 0.
\end{equation}
This implies that the sum $\mathcal{V}(\rho_t^1) + \mathcal{V}(\rho_t^2)$ is decreasing in time. 
Moreover, Grönwall's inequality implies the upper bound
\begin{equation}
\mathcal{V}(\rho_t^1) + \mathcal{V}(\rho_t^2) \leq \big(\mathcal{V}(\rho_0^1) + \mathcal{V}(\rho_0^2) \big) \exp\big( -(1-\tau) (2\lambda_1 - 2\lambda_2 M - d\sigma_1^2 -d\sigma_2^2 M^2)t \big), \qquad \text{for $t \in [0,T_{\alpha}]$}.
\label{eqn: time-evolution bound_2}
\end{equation}
Accordingly, the decay in time of the sum $\mathcal{V}(\rho_t^1) + \mathcal{V}(\rho_t^2)$ implies that the auxiliary function $C(t)$ decreases as well.
Hence, for $k=1,2$,
\begin{equation}\label{eqn: bound by auxiliary function C(t)}
\max_{t \in [0, T_{\alpha}]} |m_{L_k}^{\alpha}[\rho_t] - \theta_k^*| \leq \max_{t \in [0, T_{\alpha}]} |m_{L_1}^{\alpha}[\rho_t] - \theta_1^*| + |m_{L_2}^{\alpha}[\rho_t] - \theta_2^*| \leq \max_{t \in [0, T_{\alpha}]} C(t) \leq C \sqrt{\mathcal{V}(\rho_0^1) + \mathcal{V}(\rho_0^2)}.
\end{equation}
Also, note that
\begin{equation}\label{eqn: aux bound 1}
\begin{aligned}
\int |\theta - \theta_1^*| d\rho_{T_{\alpha}}(\theta) &= w_1 \int |\theta - \theta_1^*| d\rho_{T_{\alpha}}^1(\theta) + w_2 \int |\theta - \theta_1^*| d\rho_{T_{\alpha}}^2(\theta) \\
&\leq w_1 \sqrt{2\mathcal{V}(\rho_{T_{\alpha}}^1)} + w_2 \sqrt{\int |\theta - \theta_1^*|^2 d\rho_{T_{\alpha}}^2}\\
&\leq w_1 \sqrt{2\mathcal{V}(\rho_{T_{\alpha}}^1)} + w_2 \sqrt{4\mathcal{V}(\rho_{T_{\alpha}}^2) + 2|\theta_1^* - \theta_2^*|^2}\\
&\leq 2\sqrt{\mathcal{V}(\rho_{T_{\alpha}}^1)} + 2\sqrt{\mathcal{V}(\rho_{T_{\alpha}}^2)} + \sqrt{2} |\theta_1^* - \theta_2^*|\\
&\leq 2\sqrt{2} \sqrt{\mathcal{V}(\rho_{T_{\alpha}}^1) + \mathcal{V}(\rho_{T_{\alpha}}^2)} + \sqrt{2} |\theta_1^* - \theta_2^*|\\
&\leq 2\sqrt{2} \sqrt{\mathcal{V}(\rho_0^1) + \mathcal{V}(\rho_0^2)} + \sqrt{2} |\theta_1^* - \theta_2^*|,
\end{aligned}
\end{equation}
and, similarly,
\begin{equation}\label{eqn: aux bound 2}
\int |\theta - \theta_2^*| d\rho_{T_{\alpha}}(\theta) \leq 2\sqrt{2} \sqrt{\mathcal{V}(\rho_0^1) + \mathcal{V}(\rho_0^2)} + \sqrt{2} |\theta_1^* - \theta_2^*|.
\end{equation}
To conclude that $\mathcal{V}(\rho_{T_{\alpha}}^1) + \mathcal{V}(\rho_{T_{\alpha}}^2) \leq \varepsilon$, it remains to analyze the following three different cases.\\
\noindent \textbf{Case $T_{\alpha} \geq T^*$:} If $T_{\alpha} \geq T^*$, we can use the definition of $T^*$ in \eqref{eqn: time horizon T*} and the bound for $\mathcal{V}(\rho_t^1) + \mathcal{V}(\rho_t^2)$ in \eqref{eqn: time-evolution bound_2} to conclude that $\mathcal{V}(\rho_{T^*}^1) + \mathcal{V}(\rho_{T^*}^2) \leq \varepsilon$. Hence, from the definition of $T_{\alpha}$ in \eqref{eqn: time horizon T}, we find that $\mathcal{V}(\rho_{T_{\alpha}}^1) + \mathcal{V}(\rho_{T_{\alpha}}^2) = \varepsilon$ and $T_{\alpha} = T^*$.\\
\noindent \textbf{Case $T_{\alpha} < T^*$ and $\mathcal{V}(\rho_{T_{\alpha}}^1) + \mathcal{V}(\rho_{T_{\alpha}}^2) = \varepsilon$:} Nothing needs to be discussed in this case.\\
\noindent \textbf{Case $T_{\alpha} < T^*$, $\mathcal{V}(\rho_{T_{\alpha}}^1) + \mathcal{V}(\rho_{T_{\alpha}}^2) > \varepsilon$, and $|m_{L_1}^{\alpha}[\rho_{T_{\alpha}}] - \theta_1^*| + |m_{L_2}^{\alpha}[\rho_{T_{\alpha}}] - \theta_2^*| \geq C(T_{\alpha})$:} We will show there exists $\alpha_0 > 0$ so that for any $\alpha > \alpha_0$ we have
\begin{equation}\label{eqn: upper bound by auxiliary function C(T_alpha)}
|m_{L_1}^{\alpha}[\rho_{T_{\alpha}}] - \theta_1^*| + |m_{L_2}^{\alpha}[\rho_{T_{\alpha}}] - \theta_2^*| < C(T_{\alpha}),
\end{equation}
which would contradict $|m_{L_1}^{\alpha}[\rho_{T_{\alpha}}] - \theta_1^*| + |m_{L_2}^{\alpha}[\rho_{T_{\alpha}}] - \theta_2^*| \geq C(T_{\alpha})$. In other words, we prove that the last case never happens if we choose $\alpha$ sufficiently large. 
To show \eqref{eqn: upper bound by auxiliary function C(T_alpha)}, we define
\begin{equation*}
q_1:= \frac{1}{2}\min \bigg\{\left(\frac{\eta_1}{4} C\sqrt{\varepsilon} \right)^{\frac{1}{\nu_1}}, L_{\infty}^1 \bigg\} \qquad \text{and} \qquad r_1:= \max_{s \in [0, R_0^1]}\bigg\{\max_{\theta \in B_s(\theta_1^*)} L_1(\theta) - \underline{L_1} \leq q_1  \bigg\},
\end{equation*}
where $\underline{L_1} := \inf_{\theta \in \R^d} L_1(\theta)$, and $\eta_1, \nu_1, L_{\infty}^1$ come from assumption \ref{assump: 2.II} and $C$ is defined in \eqref{eqn: aux constant C}.
By construction, these choices satisfy $r_1 \leq R_0^1$ and $q_1 + \sup_{\theta \in B_{r_1}(\theta_1^*)} L_1(\theta) - \underline{L_1} \leq 2q_1 \leq L_{\infty}^1$. 
Furthermore, we note $q_1 > 0$, and by the continuity of $L_1$, there exists $s_{q_1} > 0$ such that $L_1(\theta) - \underline{L_1} \leq q_1$ for all $\theta \in B_{s_{q_1}} (\theta_1^*)$, thus yielding $r_1 > 0$. Therefore, we can apply Lemma \ref{lemma: quantitative laplace principle} with $q_1$ and $r_1$ as above to get
\begin{equation}\label{eqn: difference between consensus point and global minimizer}
\begin{aligned}
|m_{L_1}^{\alpha}[\rho_{T_{\alpha}}] - \theta_1^*| &\leq \frac{\big(q_1 + \sup_{\theta \in B_{r_1}(\theta_1^*)}L_1(\theta) - \underline{L_1} \big)^{\nu_1}}{\eta_1} + \frac{\exp(-\alpha q_1)}{\rho_{T_{\alpha}}(B_{r_1}(\theta_1^*))} \int |\theta - \theta_1^*| d\rho_{T_{\alpha}}(\theta)\\
&\leq \frac{(2q_1)^{\nu_1}}{\eta_1} + \frac{\exp(-\alpha q_1)}{\rho_{T_{\alpha}}(B_{r_1}(\theta_1^*))} \int |\theta - \theta_1^*| d\rho_{T_{\alpha}}(\theta)\\
&\leq \frac{\big[\big(\frac{\eta_1}{4}C\sqrt{\varepsilon} \big)^{\frac{1}{\nu_1}} \big]^{\nu_1}}{\eta_1} + \frac{\exp(-\alpha q_1)}{\rho_{T_{\alpha}}(B_{r_1}(\theta_1^*))} \int |\theta - \theta_1^*| d\rho_{T_{\alpha}}(\theta)\\
&= \frac{C}{4}\sqrt{\varepsilon} + \frac{\exp(-\alpha q_1)}{\rho_{T_{\alpha}}(B_{r_1}(\theta_1^*))} \int |\theta - \theta_1^*| d\rho_{T_{\alpha}}(\theta).
\end{aligned}
\end{equation}
Similarly, by choosing 
\begin{equation*}
q_2 := \frac{1}{2} \min \bigg\{\big(\frac{\eta_2}{4} C\sqrt{\varepsilon} \big)^{\frac{1}{\nu_2}}, L_{\infty}^2  \bigg\} \qquad \text{and} \qquad r_2:= \max_{s \in [0, R_0^2]}\bigg\{\max_{\theta \in B_s(\theta_1^*)} L_2(\theta) - \underline{L_2} \leq q_2 \bigg\},
\end{equation*}
we have
\begin{equation*}
|m_{L_2}^{\alpha}[\rho_{T_{\alpha}}] - \theta_2^*| \leq \frac{C}{4}\sqrt{\varepsilon} + \frac{\exp(-\alpha q_2)}{\rho_{T_{\alpha}}(B_{r_2}(\theta_2^*))} \int |\theta - \theta_2^*| d\rho_{T_{\alpha}}(\theta).
\end{equation*}
Combining with inequalities \eqref{eqn: aux bound 1} and \eqref{eqn: aux bound 2}, we further obtain
\begin{equation}\label{eqn: bound of the sum}
\begin{aligned}
|m_{L_1}^{\alpha}[\rho_{T_{\alpha}}] - \theta_1^*| + |m_{L_2}^{\alpha}[\rho_{T_{\alpha}}] - \theta_2^*| & \leq \frac{C}{2} \sqrt{\varepsilon} + \frac{\exp(-\alpha q_1)}{\rho_{T_{\alpha}}(B_{r_1}(\theta_1^*))} \int |\theta - \theta_1^*| d\rho_{T_{\alpha}}(\theta)\\
&\qquad \quad \; \; \, + \frac{\exp(-\alpha q_2)}{\rho_{T_{\alpha}}(B_{r_2}(\theta_2^*))} \int |\theta - \theta_2^*| d\rho_{T_{\alpha}}(\theta)\\
\leq \frac{C}{2} \sqrt{\varepsilon} + \bigg(\frac{\exp(-\alpha q_1)}{\rho_{T_{\alpha}}(B_{r_1}(\theta_1^*))} &+  \frac{\exp(-\alpha q_2)}{\rho_{T_{\alpha}}(B_{r_2}(\theta_2^*))} \bigg) \bigg( 2\sqrt{2} \sqrt{\mathcal{V}(\rho_0^1) + \mathcal{V}(\rho_0^2)} + \sqrt{2} |\theta_1^* - \theta_2^*| \bigg)\\
\leq \frac{C}{2} \sqrt{\varepsilon} + \exp(-\alpha q) \bigg(\frac{1}{\rho_{T_{\alpha}}(B_{r_1}(\theta_1^*))} &+  \frac{1}{\rho_{T_{\alpha}}(B_{r_2}(\theta_2^*))} \bigg) \bigg( 2\sqrt{2} \sqrt{\mathcal{V}(\rho_0^1) + \mathcal{V}(\rho_0^2)} + \sqrt{2} |\theta_1^* - \theta_2^*| \bigg), 
\end{aligned}
\end{equation}
with $q:= \min \big\{q_1, q_2 \big\}$.
By \eqref{eqn: bound by auxiliary function C(t)} we have the bound $G_{\alpha, k} := \max_{t \in [0, T_{\alpha}]} |m_{L_k}^{\alpha}[\rho_t] - \theta_k^*| \leq C\sqrt{\mathcal{V}(\rho_0^1) + \mathcal{V}(\rho_0^2)} := G$, which implies that all assumptions of Lemma \ref{lemma: lower bound for the mass around global minimizer} are satisfied. Therefore, by Lemma \ref{lemma: lower bound for the mass around global minimizer}, for $k=1,2$ and mollifiers $\phi_{r_k}^k$ defined in \eqref{eqn: mollifier}, there exist $a_k := a_k^l + a_k^g > 0$ such that
\begin{equation*}
\rho_{T_{\alpha}}^k \big( B_{r_k}(\theta_k^*) \big) \geq \int \phi_{r_k}^k (\theta) d\rho_0^k(\theta) \exp(-a_k T_{\alpha}),
\end{equation*}
where
\begin{equation*}
a_k^l := \max \bigg\{ h_1^l + h_2^l \frac{G_{\alpha, k}}{r_k} + h_3^l \frac{G_{\alpha, k}^2}{r_k^2}, h_4^l \bigg\} \qquad \text{and} \qquad  a_k^g := \max \bigg\{ h_1^g M_{\nabla L_k} + h_2^g M_{\nabla L_k}^2, h_3^g \bigg\},
\end{equation*}
with $h_1^l, h_2^l, h_3^l, h_4^l$ and $h_1^g, h_2^g, h_3^g$ only depending on $\lambda_1, \lambda_2, \sigma_1, \sigma_2$ and $d$.
Now we let $\widetilde{a} := \widetilde{a}^l + \widetilde{a}^g$, where
\begin{equation*}
\widetilde{a}^l := \max\bigg\{ h_1^l + h_2^l \frac{G}{r} + h_3^l \frac{G^2}{r^2}, h_4^l \bigg\} \qquad \text{and} \qquad \widetilde{a}^g := \max \bigg\{ h_1^g M + h_2^g M^2, h_3^g \bigg\},
\end{equation*}
with $r:= \min\{r_1, r_2\}$ and $M := \max \{ M_{\nabla L_1}, M_{\nabla L_2} \}$. Then
\begin{equation*}
\begin{aligned}
\rho_{T_{\alpha}}(B_{r_1}(\theta_1^*)) &= w_1 \rho_{T_{\alpha}}^1 (B_{r_1}(\theta_1^*)) + w_2 \rho_{T_{\alpha}^2} (B_{r_1}(\theta_1^*))\\
&\geq w_1 \rho_{T_{\alpha}}^1 (B_{r_1}(\theta_1^*)) \\
&\geq w_1 \int \phi_{r_1}^1(\theta) d\rho_0^1(\theta) \exp(-a_1 T_{\alpha})\\
&\geq w_1 \int \phi_{r_1}^1(\theta) d\rho_0^1(\theta) \exp(-\widetilde{a} T^*) > 0  ,
\end{aligned}
\end{equation*}
since $\widetilde{a} \geq a_1$ and $T^* \geq T_{\alpha}$, and, similarly,
\begin{equation*}
\rho_{T_{\alpha}}(B_{r_2}(\theta_2^*)) \geq w_2 \int \phi_{r_2}^2(\theta) d\rho_0^2(\theta) \exp(-\widetilde{a} T^*) > 0.
\end{equation*} 
Denote $K:= \min \bigg\{ w_1 \int \phi_{r_1}^1 (\theta) d\rho_0^1(\theta), w_2 \int \phi_{r_2}^2 (\theta) d\rho_0^2 (\theta) \bigg\}$. Then by using $\alpha > \alpha_0$ with
\begin{equation*}
\alpha_0 := \frac{\widetilde{a} T^* - \log \bigg( \frac{CK\sqrt{\varepsilon}}{2\sqrt{2} \sqrt{\mathcal{V}(\rho_0^1) + \mathcal{V}(\rho_0^2)} + \sqrt{2} |\theta_1^* - \theta_2^*|} \bigg)}{q},
\end{equation*}
the second term in \eqref{eqn: bound of the sum} is strictly smaller than $\frac{C}{2} \sqrt{\varepsilon}$. 
That is,
\begin{equation*}
\begin{aligned}
&\exp(-\alpha q) \bigg(\frac{1}{\rho_{T_{\alpha}}(B_{r_1}(\theta_1^*))} +  \frac{1}{\rho_{T_{\alpha}}(B_{r_2}(\theta_2^*))} \bigg) \bigg( 2\sqrt{2} \sqrt{\mathcal{V}(\rho_0^1) + \mathcal{V}(\rho_0^2)} + \sqrt{2} |\theta_1^* - \theta_2^*|\bigg)\\
&\quad < \exp(-\alpha_0 q) \frac{2}{K \exp (-\widetilde{a} T^*)} \bigg( 2\sqrt{2} \sqrt{\mathcal{V}(\rho_0^1) + \mathcal{V}(\rho_0^2)} + \sqrt{2} |\theta_1^* - \theta_2^*|\bigg)\\
&\quad = \frac{C}{2} \sqrt{\varepsilon}.
\end{aligned}
\end{equation*}
It follows from  \eqref{eqn: bound of the sum} that
\begin{equation*}
|m_{L_1}^{\alpha}[\rho_{T_{\alpha}}] - \theta_1^*| + |m_{L_2}^{\alpha}[\rho_{T_{\alpha}}] - \theta_2^*| < C\sqrt{\varepsilon} < C\sqrt{\mathcal{V}(\rho_{T_{\alpha}}^1) + \mathcal{V}(\rho_{T_{\alpha}}^2)},
\end{equation*}
contradicting in this way \eqref{eqn: upper bound by auxiliary function C(T_alpha)}.
\end{proof}

\section{Mean-field Limit of CBO for Clustered Federated Learning}\label{sec: M-F limit}
\subsection{Tightness of empirical measures}
First we present some uniform moment estimates for the particle system \eqref{eqn: particle system}. These estimates are a direct consequence of Lemma \ref{lemma: moment estimate}.
\begin{lemma}\label{lemma: moment esitmate corollary}
Let $L_1, L_2$ satisfy Assumption \ref{assump: Lips of L_k} and have quadratic growth at infinity. For $N_1, N_2 \geq 1$ and $N = N_1 + N_2$, let $\{\bm{\theta}_t^{1, N}\}, \{\bm{\theta}_t^{2, N}\}$ be the unique solution to the particle system \eqref{eqn: particle system} with $\rho_0^{\otimes N}$-distributed initial data $\{\bm{\theta}_0^{1, N}\}, \{\bm{\theta}_0^{2, N}$\}. Then there exists a constant $K > 0$ independent of $N$ such that
\begin{align*}
& \sup_{i = 1, 2, \dots, N_1} \bigg\{\sup_{t \in [0, T]} \mathbb{E} \big[|\theta_t^{1, (i, N)}|^2 + |\theta_t^{1, (i, N)}|^4 \big] + \sup_{t \in [0, T]} \mathbb{E} \big[ |m_{L_1}^{\alpha}[\rho_t^N]|^2 + |m_{L_1}^{\alpha}[\rho_t^N]|^4\big] \bigg\} \leq K
\end{align*}
and
\begin{align*}
& \sup_{j = 1, 2, \dots, N_2} \bigg\{\sup_{t \in [0, T]} \mathbb{E} \big[|\theta_t^{2, (j, N)}|^2 + |\theta_t^{2, (j, N)}|^4 \big] + \sup_{t \in [0, T]} \mathbb{E} \big[ |m_{L_2}^{\alpha}[\rho_t^N]|^2 + |m_{L_2}^{\alpha}[\rho_t^N]|^4\big] \bigg\} \leq K .
\end{align*}
\end{lemma}

In what follows, we treat $\theta^{1, (i, N)}, \theta^{2, (j, N)} $ as random elements, taking values in $\mathcal{C}([0, T], \mathbb{R}^d)$. These are random elements defined over a rich enough probability space $(\Omega, \mathcal{F} , \mathbb{P})$. Precisely,  $ \theta^{1, (i, N)}, \theta^{2, (j, N)} : \Omega \rightarrow \mathcal{C}([0, T], \mathbb{R}^d)$ are measurable maps. We will thus be able to interpret $\rho^{1, N} = \frac{1}{N_1}\sum_{i=1}^{N_1} \delta_{\theta^{1, (i, N)}}, \rho^{2, N} = \frac{1}{N_2} \sum_{j=1}^{N_2} \delta_{\theta^{2, (j, N)}}$, $\rho^N = \frac{N_1}{N} \rho^{1, N} + \frac{N_2}{N} \rho^{2, N}: \Omega \rightarrow \mathcal{P}\big(\mathcal{C}([0, T], \mathbb{R}^d)\big)$ as random measures.

Let $\mathcal{L}(\rho^N):= Law(\rho^N) \in \mathcal{P}(\mathcal{P}(\mathcal{C}([0, T]; \mathbb{R}^d)))$ be the law of the random variable $\rho^N$ and similarly define $\mathcal{L}(\rho^{1, N})$ and $\mathcal{L}(\rho^{2, N})$. We prove next that $\{\mathcal{L}(\rho^{1, N})\}_{N \geq 2}, \{\mathcal{L}(\rho^{2, N})\}_{N \geq 2}$ and $\{\mathcal{L}(\rho^N)\}_{N \geq 2}$ are tight. As is frequent in probability theory, we will simply say that the sequences of random variables $\{\rho^{1, N}\}_{N\geq 2}, \{\rho^{2, N}\}_{N \geq 2}$, and $\{\rho^N\}_{N \geq 2}$ are tight.

\begin{theorem}\label{thm: tightness}
Under the same assumptions as in Lemma \ref{lemma: moment esitmate corollary}, the sequences $\{\mathcal{L}(\rho^{1,N})\}_{N \geq 2}$, $\{\mathcal{L}(\rho^{2,N})\}_{N \geq 2}$ are tight in $\mathcal{P}(\mathcal{P}(\mathcal{C}([0, T]; \mathbb{R}^d)))$.
\end{theorem}

\begin{proof}[Proof of Theorem \ref{thm: tightness}]
According to \cite[Proposition~2.2 (ii)]{sznitman1991topics}, due to the exchangeability of the particle system it is sufficient to show that $\{\mathcal{L}(\theta^{1, (1, N)})\}$ and $\{\mathcal{L}(\theta^{2, (1, N)})\}$ are tight in $\mathcal{P}(\mathcal{C}([0, T]; \mathbb{R}^d))$. In other words, it is sufficient to show that the family of laws, indexed by $N$, of the trajectories of a single particle in the system is tight. We shall do this by verifying the two conditions in Aldous criteria (Lemma \ref{lemma: Aldous criteria}) for $\{\mathcal{L}(\theta^{1, (1, N)})\}$. The proof for $\{\mathcal{L}(\theta^{2, (1, N)})\}$ is similar and thus we omit the details.\\
\noindent \textbf{Checking (Con1):} Given $\varepsilon > 0$, consider the compact subset $U_{\varepsilon} := \{\theta \in \R^d \: : \: |\theta|^2 \leq \frac{K}{\varepsilon}\}$. Then, by Markov's inequality,
\begin{equation*}
        \mathcal{L}(\theta_t^{1, (1, N)})\big((U_{\varepsilon})^c \big) = \mathbb{P}\bigg(|\theta_t^{1, (1, N)}|^2 > \frac{K}{\varepsilon} \bigg) \leq \frac{\varepsilon \mathbb{E}[|\theta_t^{1, (1, N)}|^2]}{K} \leq \varepsilon \quad \forall N \geq 2,
\end{equation*}
where we have used Lemma \ref{lemma: moment esitmate corollary} in the last inequality. This means that for each fixed $t \in [0, T]$, the sequence $\{\mathcal{L}(\theta_t^{1, (1, N)})\}_{N \geq 2}$ is tight, verifying in this way condition (\textit{Con1}) in Lemma \ref{lemma: Aldous criteria}.\\
\noindent \textbf{Checking (Con2):} Let $\beta$ be a $\sigma (\theta_s^{1, (1,N)}; s \in [0,T])$-stopping time taking discrete values and such that $\beta + \delta_0 \leq T$, for some $\delta_0$ chosen below. Let $\delta \in [0, \delta_0]$. Recalling \eqref{eqn: sde for particle 1}, we deduce
\begin{equation*}
\begin{aligned}
\theta_{\beta + \delta}^{1, (1, N)} - \theta_{\beta}^{1, (1, N)} &= -\int_{\beta}^{\beta + \delta} \lambda_1 (\theta_s^{1, (1,N)} - m_{L_1}^{\alpha}[\rho_s^N])ds - \int_{\beta}^{\beta + \delta} \lambda_2 \nabla L_1 (\theta_s^{1,(1,N)})ds \\
&\quad + \int_{\beta}^{\beta + \delta} \sigma_1 |\theta_s^{1, (1,N)} - m_{L_1}^{\alpha}[\rho_s^N]| dB_s^{1,1} + \int_{\beta}^{\beta + \delta} \sigma_2 |\nabla L_1 (\theta_s^{1, (1, N)})| d\widetilde{B}_s^{1,1}.
\end{aligned}
\end{equation*}
By Cauchy-Schwarz inequality and Lemma \ref{lemma: moment esitmate corollary}, we have
\begin{equation*}
\begin{aligned}
\mathbb{E}\bigg[ |\int_{\beta}^{\beta + \delta} \lambda_1 \big(\theta_s^{1, (1, N)} - m_{L_1}^{\alpha}[\rho_s^N] \big)ds|^2 \bigg] &\leq \lambda_1^2 \delta \int_{0}^{T} \mathbb{E}\bigg[ |\theta_s^{1, (1, N)} - m_{L_1}^{\alpha}[\rho_s^N]|^2 \bigg]ds\\
&\leq 2\lambda_1^2 \delta T \bigg(\sup_{t \in [0, T]} \mathbb{E}[|\theta_t^{1, (1, N)}|^2] + \sup_{t \in [0, T]} \mathbb{E}[|m_{L_1}^{\alpha}[\rho_s^N]|^2]\bigg)\\
&\leq 2\lambda_1^2K T\delta.
\end{aligned}
\end{equation*}
Using Assumption \ref{assump: Lips of L_k}, we get
\begin{equation*}
\E \bigg[ |\int_{\beta}^{\beta + \delta} \lambda_2 \nabla L_1(\theta_s^{1, (1, N)}) ds|^2 \bigg] \leq \lambda_2^2 \delta \int_0^T |\nabla L_1(\theta_s^{1, (1, N)})|^2 ds \leq \lambda_2^2 C_{\nabla L_1}^2 T\delta.
\end{equation*}
Further, we apply Itô isometry to get
\begin{equation*}
\begin{aligned}
\mathbb{E}\bigg[|\sigma_1 \int_{\beta}^{\beta + \delta} |\theta_s^{1, (1, N)} - m_{L_1}^{\alpha}[\rho_s^N]|dB_s^{1, 1}|^2 \bigg] &= \sigma_1^2 \mathbb{E}\bigg[ \int_{\beta}^{\beta + \delta} |\theta_s^{1, (1, N)} - m_{L_1}^{\alpha}[\rho_s^N]|^2 ds \bigg]\\
& \leq \sigma_1^2 \delta^{\frac{1}{2}} \mathbb{E} \bigg[\big(\int_0^T |\theta_s^{1, (1, N)} - m_{L_1}^{\alpha}[\rho_s^N]|^4 ds \big)^{\frac{1}{2}} \bigg]\\
&\leq \sigma_1^2 \delta^{\frac{1}{2}} \bigg( \int_0^T \mathbb{E}[|\theta_s^{1, (1, N)} - m_{L_1}^{\alpha}[\rho_s^N]|^4]ds \bigg)^{\frac{1}{2}}\\
&\leq \sigma_1^2 (8K)^{\frac{1}{2}}  T^{\frac{1}{2}}\delta^{\frac{1}{2}}
\end{aligned}
\end{equation*}
and
\begin{equation*}
\E \bigg[ |\sigma_2 \int_{\beta}^{\beta + \delta} |\nabla L_1 (\theta_s^{1, (1,N)})| d\widetilde{B}_s^{1,1}|^2\bigg] = \sigma_2^2 \E \bigg[ \int_{\beta}^{\beta + \delta} |\nabla L_1 (\theta_s^{1, (1,N)})|^2 ds \bigg] \leq \sigma_2^2 C_{\nabla L_1}^2 \delta.
\end{equation*}
Combining the above four estimates we obtain
\begin{equation*}
\E |\theta_{\beta + \delta}^{1, (1, N)} - \theta_{\beta}^{1, (1,N)}|^2 \leq C(\lambda_1, \lambda_2, \sigma_1, \sigma_2, T, K)(\delta^{\frac{1}{2}} + \delta).
\end{equation*}
Hence, for any $\varepsilon > 0$ and $\eta > 0$, there exist some $\delta_0 > 0$ such that for all $N \geq 2$, it holds that
\begin{equation*}
\sup_{\delta \in [0, \delta_0]} \mathbb{P} \big( |\theta_{\beta + \delta}^{1, (1, N)} - \theta_{\beta}^{1, (1,N)}|^2 \geq \eta \big) \leq \sup_{\delta \in [0, \delta_0]} \frac{\E|\theta_{\beta + \delta}^{1, (1, N)} - \theta_{\beta}^{1, (1, N)}|^2}{\eta} \leq \varepsilon.
\end{equation*}
This justifies condition (\textit{Con2}) in Lemma \ref{lemma: Aldous criteria} and completes the proof.
\end{proof}

\begin{lemma}\label{lemma: weak convergence of subsequence} Under the same assumptions as in Theorem \ref{thm: tightness} the following statements hold:\\
\begin{itemize}
\item[(1)] There exists a subsequence of $\{(\rho^{1, N}, \rho^{2, N})\}_{N \geq 2}$ (not relabeled for simplicity) and a random variable $(\tilde \rho^1,\tilde  \rho^2): \Omega \rightarrow \mathcal{P}(\mathcal{C}([0, T], \mathbb{R}^d)) \times \mathcal{P}(\mathcal{C}([0, T], \mathbb{R}^d))$ such that
\begin{equation*}
    (\rho^{1, N}, \rho^{2, N}) \rightharpoonup (\tilde \rho^1, \tilde \rho^2) \quad \text{in law as} \,\, N \rightarrow +\infty.
\end{equation*}
\item[(2)] Let $\rho^N = \frac{N_1}{N}\rho^{1, N} + \frac{N_2}{N}\rho^{2, N}$, where $N_1$ and $N_2$ satisfy $\frac{N_1}{N_2} \rightarrow \frac{w_1}{w_2}$ as $N = N_1 + N_2 \rightarrow +\infty$ and $w_1, w_2 >0$, $w_1 + w_2 = 1$. Then 
\begin{equation*}
    \rho^N \rightharpoonup w_1 \tilde  \rho^1 + w_2 \tilde \rho^2 \quad \text{in law as} \,\, N \rightarrow +\infty.
\end{equation*}
\end{itemize}
\end{lemma}

\begin{proof}
We know that $\{(\rho^{1, N}, \rho^{2, N})\}_{N \geq 2}$ is tight given that $\{\rho^{1, N}\}_{N \geq 2}$ and $\{\rho^{2, N}\}_{N \geq 2}$ are tight according to Theorem \ref{thm: tightness}. Assertion $(1)$ can thus be directly deduced from Prokhorov's theorem. 

Regarding assertion (2), it is straightforward to verify that $(\frac{N_1}{N} \rho^{1, N}, \frac{N_2}{N} \rho^{2, N})$ converges to $(w_1 \tilde \rho^1, w_2 \tilde  \rho^2)$ in distribution --notice that these random variables take values in $\M_+(\mathcal{C}([0,T], \R^d)) \times \M_+(\mathcal{C}([0,T], \R^d)) $, i.e., each coordinate is a positive measure over $\mathcal{C}([0,T], \R^d)$-- provided that $(\rho^{1, N}, \rho^{2, N}) \rightharpoonup (\tilde \rho^1, \tilde  \rho^2)$, and $\frac{N_1}{N_2} \rightarrow \frac{w_1}{w_2}$, $N = N_1 + N_2$. Assertion (2) follows from the continuous mapping theorem.
\end{proof}

\subsection{Identification of the limit measures via PDEs}\label{sec: identification of limit measure}
In this section we rewrite the mean-field SDE equations \eqref{eqn:mean_field_SDE} in the following form:
\begin{equation}\label{eqn: mean-field sde}
\begin{aligned}
d\left( \begin{array}{c}
     \overline{\theta}_t^1  \\
     \overline{\theta}_t^2 
\end{array}
\right) &= - \left( \begin{array}{c}
     \lambda_1 (\overline{\theta}_t^1 - m_{L_1}^{\alpha}[\rho_t]) + \lambda_2 \nabla L_1 (\overline{\theta}_t^1)  \\
      \lambda_1 (\overline{\theta}_t^1 - m_{L_2}^{\alpha}[\rho_t]) + \lambda_2 \nabla L_2 (\overline{\theta}_t^2)
\end{array}
\right)dt \\
&\quad + \sigma_1 \left(\begin{array}{cc}
   \text{diag}(|\overline{\theta}_t^1 - m_{L_1}^{\alpha} [\rho_t]|) & 0 \\
   0 & \text{diag}(|\overline{\theta}_t^2 - m_{L_2}^{\alpha} [\rho_t]|)
\end{array}
\right) \left( \begin{array}{c}
    dB_t^1     \\
    dB_t^2
   \end{array} \right) \\
&\quad + \sigma_2 \left(\begin{array}{cc}
   \text{diag}(|\nabla L_1 (\overline{\theta}_t^1)|) & 0 \\
   0 & \text{diag}(|\nabla L_2(\overline{\theta}_t^2) |)
\end{array}
\right) \left( \begin{array}{c}
    d\widetilde{B}_t^1     \\
    d\widetilde{B}_t^2
   \end{array} \right),
\end{aligned}
\end{equation}
where, recall,  $\overline{\theta}_0^1$ and $\overline{\theta}_0^2$ are independent initial conditions. Recall also that $\rho^1_t = \text{Law}(\overline{\theta}^1_t), \rho^2_t = \text{Law}(\overline{\theta}^2_t)$, $\rho_t = w_1 \rho^1_t + w_2 \rho^2_t$. In the above, we use $\text{diag}(v)$ to denote the diagonal matrix in $\R^{d\times d}$ with diagonal entries equal to the coordinates of the vector $v$. For each $t\in [0,T]$, let  $\mu_t= \text{Law} ((\overline{\theta}_t^1, \overline{\theta}_t^2))$. We will view $\mu$ as an element of $\mathcal{C}([0,T], \mathcal{P}(\R^d \times \R^d))$. We notice that it must satisfy the joint Fokker-Planck equation: 
\begin{equation}\label{eqn: FP eqn for joint distribution}
\begin{aligned}
\partial_t \mu_t(\theta) &= \nabla_{(\theta_1, \theta_2)} \cdot \bigg( \big(\lambda_1(\theta_1 - m_{L_1}^{\alpha}[\rho_t]) + \lambda_2 \nabla L_1 (\theta_1), \lambda_1(\theta_2 - m_{L_2}^{\alpha}[\rho_t]) + \lambda_2 \nabla L_2 (\theta_2)\big) \mu_t \bigg)\\
&\quad + \Delta_{\theta_1} \bigg( \big(\frac{\sigma_1^2}{2} |\theta_1 - m_{L_1}^{\alpha}[\rho_t]|^2 + \frac{\sigma_2^2}{2} |\nabla L_1(\theta_1)|^2 \big) \mu_t \bigg) + \Delta_{\theta_2} \bigg( \big(\frac{\sigma_1^2}{2} |\theta_2 - m_{L_2}^{\alpha}[\rho_t]|^2 + \frac{\sigma_2^2}{2} |\nabla L_2(\theta_2)|^2 \big) \mu_t \bigg),
\end{aligned}
\end{equation}
where $\theta = (\theta_1, \theta_2) \in \R^{2d},  \Delta_{\theta_1} := \sum_{i=1}^d \frac{\partial^2}{\partial \theta_i^2}$ and $\Delta_{\theta_2} := \sum_{i=d+1}^{2d} \frac{\partial^2}{\partial \theta_i^2}$, and $\nabla_{(\theta_1, \theta_2)}\cdot$ denotes the divergence in both coordinates $\theta_1, \theta_2$. Equation \eqref{eqn: FP eqn for joint distribution} must be interpreted as in Definition \ref{def:JointMeanFieldFP} below. 

\begin{remark}
It is important to observe that, since  $\overline{\theta}^1$ and $\overline{\theta}^2$ are  independent according to Remark \ref{rem:IndependenceSDEs}, the distribution $\mu_t$ satisfies $\mu_t = \rho^1_t \otimes \rho^2_t \in \mathcal{P}( \R^d \times \R^d  )$ for $\rho_t^1$ and $\rho_t^2$ as in \eqref{eqn: FP system}. Although it may seem redundant to introduce $\mu$ in this form, we do so for the convenience of our analysis below.     \label{rem:JointFP}
\end{remark}

\begin{definition}
We say  $\tilde \mu \in  \mathcal{C}([0,T], \mathcal{P}( \R^d \times \R^d) )$ is a weak solution to the PDE \eqref{eqn: FP eqn for joint distribution} if
\begin{itemize}
\item[(i)] For all $\phi \in \mathcal{C}_b(\mathbb{R}^d \times \mathbb{R}^d)$ and $t_n \rightarrow t$ we have
\begin{equation}\label{eqn: continuous in time}
\int_{\mathbb{R}^d \times \mathbb{R}^d} \phi(\theta_1, \theta_2) d\tilde \mu_{t_n}(\theta_1, \theta_2) \rightarrow \int_{\mathbb{R}^d \times \mathbb{R}^d} \phi(\theta_1, \theta_2) d\tilde \mu_t(\theta_1, \theta_2);
\end{equation}

\item[(ii)] The following holds
\begin{equation*}
\begin{aligned}
&\langle \varphi(\theta_1, \theta_2), \tilde \mu_t(d\theta_1, d\theta_2) \rangle - \langle \varphi(\theta_1, \theta_2), \tilde \mu_0(d\theta_1, d\theta_2) \rangle \\ 
&\quad+  \int_0^t \langle \big( \lambda_1(\theta_1 - m_{L_1}^{\alpha}[\rho_s]) + \lambda_2 \nabla L_1 (\theta_1) \big) \cdot \nabla_{\theta_1} \varphi\\
&\qquad \qquad \qquad  + \big( \lambda_1(\theta_2 - m_{L_2}^{\alpha}[\rho_s]) + \lambda_2 \nabla L_2(\theta_2) \big) \cdot \nabla_{\theta_2} \varphi, \tilde \mu_s(d\theta_1, d\theta_2)  \rangle ds\\
& \quad - \int_0^t \langle \big( \frac{\sigma_1^2}{2} |\theta_1 - m_{L_1}^{\alpha}[\rho_s]|^2 + \frac{\sigma_2^2}{2} |\nabla L_1(\theta_1)|^2 \big) \Delta_{\theta_1} \varphi \\
&\qquad \qquad \qquad + \big(\frac{\sigma_1^2}{2}|\theta_2 - m_{L_2}^{\alpha}[\rho_s]|^2 + \frac{\sigma_2^2}{2} |\nabla L_2 (\theta_2)|^2 \big) \Delta_{\theta_2}\varphi, \tilde \mu_s(d\theta_1, d\theta_2) \rangle ds = 0,
\end{aligned}
\end{equation*}
for all $\varphi \in \mathcal{C}_c^2(\mathbb{R}^d \times \mathbb{R}^d)$ and all $t \in [0,T]$. Here, $\rho_t := w_1 \rho_t^1 + w_2 \rho_t^2$, and $\rho_t^1 := \pi_{1\sharp} \tilde \mu_t$, $\rho_t^2 := \pi_{2\sharp} \tilde \mu_t$ are the marginals of $\tilde \mu_t$.  In the above and in the remainder, we use $\langle \cdot, \cdot \rangle$ to denote the standard duality pairing between test functions $\varphi$ and measures.
\end{itemize}
\label{def:JointMeanFieldFP}
\end{definition}
For each $\varphi \in \mathcal{C}_c^2(\R^d \times \R^d)$ and fixed $t \in (0,T)$ we define the functional $F_\varphi$ over $\mathcal{C}([0, T], \mathcal{P}(\R^d \times \R^d)) \times [0,1] $ given by
\begin{equation*}
\begin{aligned}
F_{\varphi}(\nu,   \tilde{w} \nc) &:=    |\tilde{w} - w_1| \nc  +\langle \varphi(\theta_1, \theta_2), \nu_t(d\theta_1, d\theta_2) \rangle - \langle \varphi(\theta_1, \theta_2), \nu_0(d\theta_1, d\theta_2) \rangle \\ 
&\quad+  \int_0^t \langle \big( \lambda_1(\theta_1 - m_{L_1}^{\alpha}[ \rho_{s,\nu, \tilde w} \nc]) + \lambda_2 \nabla L_1 (\theta_1) \big) \cdot \nabla_{\theta_1} \varphi\\
&\qquad \qquad \qquad  + \big( \lambda_1(\theta_2 - m_{L_2}^{\alpha}[ \rho_{s,\nu, \tilde w} \nc]) + \lambda_2 \nabla L_2(\theta_2) \big) \cdot \nabla_{\theta_2} \varphi, \nu_s(d\theta_1, d\theta_2)  \rangle ds\\
& \quad - \int_0^t \langle \big( \frac{\sigma_1^2}{2} |\theta_1 - m_{L_1}^{\alpha}[ \rho_{s,\nu, \tilde w} \nc]|^2 + \frac{\sigma_2^2}{2} |\nabla L_1(\theta_1)|^2 \big) \Delta_{\theta_1} \varphi \\
&\qquad \qquad \qquad + \big(\frac{\sigma_1^2}{2}|\theta_2 - m_{L_2}^{\alpha}[ \rho_{s,\nu, \tilde w} \nc]|^2 + \frac{\sigma_2^2}{2} |\nabla L_2 (\theta_2)|^2 \big) \Delta_{\theta_2}\varphi, \nu_s(d\theta_1, d\theta_2) \rangle ds,
\end{aligned}
\end{equation*}
for $\nu \in \mathcal{C}([0, T], \mathcal{P}(\R^d \times \R^d))$ and $\tilde w \in [0,1]$.  In the above,  $\rho^1_{s, \nu,\tilde w}:= \pi_{1 \sharp} \nu_s$, $\rho^2_{s, \nu, \tilde w}:= \pi_{2 \sharp} \nu_s$, and $\rho_{s,\nu, \tilde w} := \tilde{w}  \rho^1_{s, \nu,\tilde w} + (1- \tilde{w}) \rho^2_{s, \nu,\tilde w} $.  \nc

We have the following estimate.
\begin{lemma}\label{lemma: bound of second moment for empirical measure}
Let $L_1, L_2$ satisfy Assumption \ref{assump: Lips of L_k} and have quadratic growth at infinity, and let $\mu_0^1, \mu_0^2 \in \mathcal{P}_4(\mathbb{R}^d)$. Assume that, for every $N \geq 2$,  $\{\bm{\theta}_t^{1, N}\}$, $\{\bm{\theta}_t^{2,N}\}$ is the unique solution to the particle system \eqref{eqn: particle system} with $(\mu_0^1)^{\otimes N_1} \otimes (\mu_0^2)^{\otimes N_2}$-distributed initial data $\{\bm{\theta}_0^{1,N}\}, \{\bm{\theta}_0^{2,N}\}$. Then there exists a constant $C > 0$ depending only on $\sigma_1, \sigma_2, K, T$, and $\|\nabla \varphi\|_{\infty}$ such that
\begin{equation*}
    \mathbb{E}\big[|F_{\varphi}(\mu^N,  \frac{N_1}{N} \nc)|^2 \big] \leq \frac{C}{N_1N_2} +  2 | \frac{N_1}{N} - w_1  |^2 \nc ,
\end{equation*}
where $\mu^N_s := \rho^{1, N}_s \otimes \rho^{2, N}_s \in \mathcal{P}(\R^d \times \R^d)$, and the empirical measures $\rho_{s}^{1,N}, \rho_{s}^{2,N}$ are as in \eqref{eq:EmpiricalMeasuresParticleSystem}. $\mu^N$ can thus be interpreted as a random variable taking values in $\mathcal{C}([0,T], \mathcal{P}(\R^d \times \R^d))$.
\end{lemma}
\begin{proof}
Using the definition of $F_{\varphi}$ one has
\begin{equation*}
\begin{aligned}
F_{\varphi} (\mu^N,  \frac{N_1}{N} \nc)  -  | \frac{N_1}{N} - w_1  | \nc &= \frac{1}{N_1 N_2} \sum_{i=1}^{N_1} \sum_{j=1}^{N_2} \varphi (\theta_t^{1,i}, \theta_t^{2,j}) - \frac{1}{N_1 N_2} \sum_{i=1}^{N_1} \sum_{j=1}^{N_2} \varphi (\theta_0^{1,i}, \theta_0^{2,j})\\
&\quad + \int_0^j \frac{1}{N_1 N_2} \sum_{i=1}^{N_1} \sum_{j=1}^{N_2} \bigg[ \big(\lambda_1 (\theta_s^{1,i} - m_{L_1}^{\alpha}[\rho_s^N]) + \lambda_2 \nabla L_1 (\theta_s^{1,i}) \big) \cdot \nabla_{\theta_1} \varphi (\theta_s^{1,i}, \theta_s^{2,j})\\
&\qquad \qquad \qquad \qquad \qquad  + \big(\lambda_1 (\theta_s^{2,j} - m_{L_2}^{\alpha}[\rho_s^N]) + \lambda_2 \nabla L_2 (\theta_s^{2,j}) \big) \cdot \nabla_{\theta_2} \varphi (\theta_s^{1,i}, \theta_s^{2,j}) \bigg]ds\\
&\quad - \int_0^t \frac{1}{N_1 N_2} \sum_{i=1}^{N_1} \sum_{j=1}^{N_2} \bigg[ \big( \frac{\sigma_1^2}{2} |\theta_s^{1,i} - m_{L_1}^{\alpha}[\rho_s^N]|^2 + \frac{\sigma_2^2}{2} |\nabla L_1(\theta_s^{1,i})|^2 \big) \Delta_{\theta_1} \varphi(\theta_s^{1,i}, \theta_s^{2,j})\\
&\qquad \qquad \qquad \qquad \qquad + \big( \frac{\sigma_1^2}{2} |\theta_s^{2,j} - m_{L_2}^{\alpha}[\rho_s^N]|^2 + \frac{\sigma_2^2}{2} |\nabla L_2(\theta_s^{2,j})|^2 \big) \Delta_{\theta_2} \varphi(\theta_s^{1,i}, \theta_s^{2,j}) \bigg] ds\\
&= \frac{1}{N_1 N_2} \sum_{i=1}^{N_1} \sum_{j=1}^{N_2} \bigg\{\varphi (\theta_t^{1,i}, \theta_t^{2,j}) - \varphi (\theta_0^{1,i}, \theta_0^{2,j})\\
&\qquad \qquad \qquad \qquad + \int_0^t \bigg[ \big(\lambda_1 (\theta_s^{1,i} - m_{L_1}^{\alpha}[\rho_s^N]) + \lambda_2 \nabla L_1 (\theta_s^{1,i}) \big) \cdot \nabla_{\theta_1} \varphi (\theta_s^{1,i}, \theta_s^{2,j})\\
&\qquad \qquad \qquad \qquad \qquad \quad  + \big(\lambda_1 (\theta_s^{2,j} - m_{L_2}^{\alpha}[\rho_s^N]) + \lambda_2 \nabla L_2 (\theta_s^{2,j}) \big) \cdot \nabla_{\theta_2} \varphi (\theta_s^{1,i}, \theta_s^{2,j}) \bigg]ds\\
&\qquad \qquad \qquad \qquad -\int_0^t \bigg[ \big( \frac{\sigma_1^2}{2} |\theta_s^{1,i} - m_{L_1}^{\alpha}[\rho_s^N]|^2 + \frac{\sigma_2^2}{2} |\nabla L_1(\theta_s^{1,i})|^2 \big) \Delta_{\theta_1} \varphi(\theta_s^{1,i}, \theta_s^{2,j})\\
&\qquad \qquad \qquad \qquad \qquad \quad + \big( \frac{\sigma_1^2}{2} |\theta_s^{2,j} - m_{L_2}^{\alpha}[\rho_s^N]|^2 + \frac{\sigma_2^2}{2} |\nabla L_2(\theta_s^{2,j})|^2 \big) \Delta_{\theta_2} \varphi(\theta_s^{1,i}, \theta_s^{2,j}) \bigg] ds \bigg\}.
\end{aligned}
\end{equation*}
On the other hand, for each $i \in [N_1], j \in [N_2]$, from Itô formula we get
\begin{equation*}
\begin{aligned}
\varphi(\theta_s^{1, i}, \theta_t^{2,j}) - \varphi (\theta_0^{1,i}, \theta_0^{2,j}) &= -\int_0^t \bigg[ \big(\lambda_1 (\theta_s^{1,i} - m_{L_1}^{\alpha}[\rho_s^N]) + \lambda_2 \nabla L_1 (\theta_s^{1,i}) \big) \cdot \nabla_{\theta_1} \varphi (\theta_s^{1,i}, \theta_s^{2,j})\\
&\qquad \qquad  + \big(\lambda_1 (\theta_s^{2,j} - m_{L_2}^{\alpha}[\rho_s^N]) + \lambda_2 \nabla L_2 (\theta_s^{2,j}) \big) \cdot \nabla_{\theta_2} \varphi (\theta_s^{1,i}, \theta_s^{2,j}) \bigg]ds\\
&\quad + \int_0^t \bigg[ \big( \frac{\sigma_1^2}{2} |\theta_s^{1,i} - m_{L_1}^{\alpha}[\rho_s^N]|^2 + \frac{\sigma_2^2}{2} |\nabla L_1(\theta_s^{1,i})|^2 \big) \Delta_{\theta_1} \varphi(\theta_s^{1,i}, \theta_s^{2,j})\\
&\qquad \qquad + \big( \frac{\sigma_1^2}{2} |\theta_s^{2,j} - m_{L_2}^{\alpha}[\rho_s^N]|^2 + \frac{\sigma_2^2}{2} |\nabla L_2(\theta_s^{2,j})|^2 \big) \Delta_{\theta_2} \varphi(\theta_s^{1,i}, \theta_s^{2,j}) \bigg] ds\\
+ \sigma_1 &\int_0^t \left( \begin{array}{cc}
   \text{diag}(|\theta_s^{1,i} - m_{L_1}^{\alpha} [\rho_s^N]|)  &  0 \\
   0  &  \text{diag}(|\theta_s^{2,j} - m_{L_2}^{\alpha} [\rho_s^N]|)  
\end{array} \right) \nabla \varphi (\theta_s^{1,i}, \theta_s^{2,j}) \cdot \left( \begin{array}{c}
      dB_s^{1,i}   \\
      dB_s^{2,j}  
   \end{array} \right)\\
+ \sigma_2 &\int_0^t \left( \begin{array}{cc}
   \text{diag}(|\nabla L_1 (\theta_s^{1,i})|)  &  0 \\
   0  &  \text{diag}(|\nabla L_2 (\theta_s^{2,j})|)  
\end{array} \right) \nabla \varphi (\theta_s^{1,i}, \theta_s^{2,j}) \cdot \left( \begin{array}{c}
      d\widetilde{B}_s^{1,i}   \\
      d\widetilde{B}_s^{2,j}  
   \end{array} \right).
\end{aligned}
\end{equation*}
This implies that
\begin{equation*}
\begin{aligned}
&F_{\varphi}(\mu^N, N_1/N)  -  | \frac{N_1}{N} - w_1  | \nc  \\
&= \frac{1}{N_1 N_2} \sum_{i=1}^{N_1}\sum_{j=1}^{N_2} \bigg[ \sigma_1 \int_0^t \left( \begin{array}{cc}
   \text{diag}(|\theta_s^{1,i} - m_{L_1}^{\alpha} [\rho_s^N]|)  &  0 \\
   0  &  \text{diag}(|\theta_s^{2,j} - m_{L_2}^{\alpha} [\rho_s^N]|)  
\end{array} \right) \nabla \varphi (\theta_s^{1,i}, \theta_s^{2,j}) \cdot \left( \begin{array}{c}
      dB_s^{1,i}   \\
      dB_s^{2,j}  
   \end{array} \right)\\
&\qquad \qquad \qquad \qquad + \sigma_2 \int_0^t \left( \begin{array}{cc}
   \text{diag}(|\nabla L_1 (\theta_s^{1,i})|)  &  0 \\
   0  &  \text{diag}(|\nabla L_2 (\theta_s^{2,j})|)  
\end{array} \right) \nabla \varphi (\theta_s^{1,i}, \theta_s^{2,j}) \cdot \left( \begin{array}{c}
      d\widetilde{B}_s^{1,i}   \\
      d\widetilde{B}_s^{2,j}  
   \end{array} \right)\bigg].
\end{aligned}
\end{equation*}
Then, by Itô isometry, we conclude that
\begin{equation*}
\begin{aligned}
&\E[ \left|F_{\varphi}(\mu^N, N_1/N) -  |N_1/ N  - w_1 | \nc \right|^2] \\
&= \frac{1}{N_1^2 N_2^2} \E \bigg[ \bigg|\sum_{i,j} \bigg[ \sigma_1 \int_0^t \left( \begin{array}{cc}
   \text{diag}(|\theta_s^{1,i} - m_{L_1}^{\alpha} [\rho_s^N]|)  &  0 \\
   0  &  \text{diag}(|\theta_s^{2,j} - m_{L_2}^{\alpha} [\rho_s^N]|)  
\end{array} \right) \nabla \varphi (\theta_s^{1,i}, \theta_s^{2,j}) \cdot \left( \begin{array}{c}
      dB_s^{1,i}   \\
      dB_s^{2,j}  
   \end{array} \right)\\
&\qquad \qquad \qquad + \sigma_2 \int_0^t \left( \begin{array}{cc}
   \text{diag}(|\nabla L_1 (\theta_s^{1,i})|)  &  0 \\
   0  &  \text{diag}(|\nabla L_2 (\theta_s^{2,j})|)  
\end{array} \right) \nabla \varphi (\theta_s^{1,i}, \theta_s^{2,j}) \cdot \left( \begin{array}{c}
      d\widetilde{B}_s^{1,i}   \\
      d\widetilde{B}_s^{2,j}  
   \end{array} \right)\bigg] \bigg|^2 \bigg]\\
&=\frac{1}{N_1^2 N_2^2} \sum_{i,j} \bigg\{  \sigma_1^2 \E \bigg[ \bigg| \int_0^t \left( \begin{array}{cc}
   \text{diag}(|\theta_s^{1,i} - m_{L_1}^{\alpha} [\rho_s^N]|)  &  0 \\
   0  &  \text{diag}(|\theta_s^{2,j} - m_{L_2}^{\alpha} [\rho_s^N]|)  
\end{array} \right) \nabla \varphi (\theta_s^{1,i}, \theta_s^{2,j}) \cdot \left( \begin{array}{c}
      dB_s^{1,i}   \\
      dB_s^{2,j}  
   \end{array} \right) \bigg|^2 \bigg] \\
&\qquad \qquad \qquad + \sigma_2^2 \E \bigg[ \bigg|\int_0^t \left( \begin{array}{cc}
   \text{diag}(|\nabla L_1 (\theta_s^{1,i})|)  &  0 \\
   0  &  \text{diag}(|\nabla L_2 (\theta_s^{2,j})|)  
\end{array} \right) \nabla \varphi (\theta_s^{1,i}, \theta_s^{2,j}) \cdot \left( \begin{array}{c}
      d\widetilde{B}_s^{1,i}   \\
      d\widetilde{B}_s^{2,j}  
   \end{array} \right)  \bigg|^2 \bigg] \bigg\}\\
&=\frac{1}{N_1^2 N_2^2} \sum_{i,j} \bigg\{ \sigma_1^2 \E \bigg[ \int_0^t \bigg\| \left( \begin{array}{cc}
   \text{diag}(|\theta_s^{1,i} - m_{L_1}^{\alpha} [\rho_s^N]|)  &  0 \\
   0  &  \text{diag}(|\theta_s^{2,j} - m_{L_2}^{\alpha} [\rho_s^N]|)  
\end{array} \right) \nabla \varphi (\theta_s^{1,i}, \theta_s^{2,j}) \bigg\|_F^2 ds \bigg]\\
&\qquad \qquad \qquad + \sigma_2^2 \E \bigg[ \int_0^t \bigg\| \left( \begin{array}{cc}
   \text{diag}(|\nabla L_1 (\theta_s^{1,i})|)  &  0 \\
   0  &  \text{diag}(|\nabla L_2 (\theta_s^{2,j})|)  
\end{array} \right) \nabla \varphi (\theta_s^{1,i}, \theta_s^{2,j}) \bigg\|_F^2 ds \bigg]\bigg\}\\
&= \frac{1}{N_1^2 N_2^2} \sum_{i,j} \E \bigg[ \int_0^t \bigg( \big( \sigma_1^2 |\theta_s^{1,i} - m_{L_1}^{\alpha} [\rho_s^N]|^2 + \sigma_2^2 |\nabla L_1 (\theta_s^{1,i})|^2 \big) |\nabla_{\theta_1} \varphi (\theta_s^{1,i}, \theta_s^{2,j})|^2 \\
&\qquad \qquad \qquad \qquad \qquad + \big( \sigma_1^2 |\theta_s^{2,j} - m_{L_2}^{\alpha} [\rho_s^N]|^2 + \sigma_2^2 |\nabla L_2 (\theta_s^{2,j})|^2 \big) |\nabla_{\theta_2} \varphi (\theta_s^{1,i}, \theta_s^{2,j})|^2 \bigg)ds \bigg]\\
&\leq C(\sigma_1, \sigma_2, K, T, C_{\nabla L_1}, C_{\nabla L_2}, \|\nabla \varphi\|_{\infty}) \frac{1}{N_1 N_2},
\end{aligned}
\end{equation*}
where we have used Lemma \ref{lemma: moment esitmate corollary} and Assumption \ref{assump: Lips of L_k} in the last inequality. From the above we deduce the desired estimate.
\end{proof}

\noindent By Skorokhod's lemma (see \cite[Theorem~6.7]{billingsley2013convergence}) and Lemma \ref{lemma: weak convergence of subsequence} we may change the probability space $(\Omega, \mathcal{F}, \mathbb{P})$ and, in this new probability space, assume that the sequence of random variables $\{(\rho^{1,N}, \rho^{2,N})\}_{N \geq 2}$ converges toward some random variable $(\tilde \rho^1, \tilde \rho^2)$ almost surely; notice that these are random variables taking values in  $\mathcal{P}(\mathcal{C}([0, T],\R^d)) \times \mathcal{P}(\mathcal{C}([0, T],\R^d))$ (see more detailed explanations in Remark \ref{remark: details on applying Skorokhod's lemma} in Appendix). In particular, if we let $\tilde{\mu}$ be the random variable taking values in $\mathcal{C}([0,T] , \mathcal{P}(\R^d \times \R^d))$ defined by $\tilde \mu : t \in [0,T]  \mapsto \tilde{\mu}_t:= \tilde{\rho}_t^1 \otimes \tilde{\rho}_t^2$, then $\mu^N$ converges a.s. toward $\tilde{\mu}$, as random variables taking values in $\mathcal{C}([0,T] , \mathcal{P}(\R^d \times \R^d))$. Our goal will be to show that $\tilde{\mu}$ is actually deterministic and equal to $\mu$, where $\mu: t \in [0,T] \mapsto \rho_t^1 \otimes \rho_t^2$, and $\rho_t^1, \rho^2_t$ satisfy the Fokker-Planck equation \eqref{eqn: FP system}.

Notice that, from the above, for all $t \in [0, T]$ and $\phi \in \mathcal{C}_b(\mathbb{R}^d \times \mathbb{R}^d)$,
\begin{equation}\label{eqn: a.s. convergence}
    \lim_{N \rightarrow +\infty} |\langle \phi, \mu_t^N - \tilde \mu_t \rangle| + |m_{L_1}^{\alpha}[\rho_t^N] - m_{L_1}^{\alpha}[\tilde \rho_t]| + |m_{L_2}^{\alpha}[\rho_t^N] - m_{L_2}^{\alpha}[\tilde \rho_t]|= 0 \quad \text{a.s.}
\end{equation}
Indeed, according to Assumption \ref{assump: Lips of L_k}, one has $\theta e^{-\alpha L_k(\theta)}, e^{-\alpha L_k(\theta)} \in \mathcal{C}_b(\mathbb{R}^d)$ for $k = 1, 2$. 
{On the other hand, $\rho_t^N = \frac{N_1}{N}\rho_t^{1, N} + \frac{N_2}{N} \rho_t^{2, N}$ converges almost surely toward $\tilde \rho_t = w_1\tilde \rho_t^1 + w_2 \tilde \rho_t^2$ (see Remark \ref{remark: convergence a.s.} below).}
Then, by continuous mapping theorem, we have
\begin{equation*}
\lim_{N \rightarrow +\infty} m_{L_k}^{\alpha}[\rho_t^N] = \lim_{N \rightarrow +\infty} \frac{\langle \theta e^{-\alpha L_k(\theta)}, \rho_t^N(d\theta)\rangle}{\langle e^{-\alpha L_k(\theta)}, \rho_t^N(d\theta)\rangle} = \frac{\langle \theta e^{-\alpha L_k(\theta)}, \tilde \rho_t(d\theta)\rangle}{\langle e^{-\alpha L_k(\theta)}, \tilde \rho_t(d\theta)\rangle} = m_{L_k}^{\alpha}[\tilde \rho_t] \quad \text{a.s.} \quad \text{for}\,\, k=1,2.
\end{equation*}
For each $A > 0$, let us take $\phi = |\cdot|^4 \wedge A \in \mathcal{C}_b(\mathbb{R}^d \times \mathbb{R}^d)$. It follows from \eqref{eqn: a.s. convergence} that
\begin{equation*}
\begin{aligned}
\mathbb{E}\big[\int_{\mathbb{R}^d \times \mathbb{R}^d} (|(\theta_1, \theta_2)|^4 \wedge A)\tilde  \mu_t(d\theta_1, d\theta_2) \big] &= \mathbb{E}\big[\lim_{N \rightarrow +\infty} \int_{\mathbb{R}^d \times \mathbb{R}^d} (|(\theta_1, \theta_2)|^4 \wedge A) \mu_t^N(d\theta_1, d\theta_2) \big]\\
&\leq \liminf_{N \rightarrow + \infty} \frac{1}{N} \big( \sum_{i=1}^{N_1} \mathbb{E}[|\theta_t^{1, (i, N)}|^4] + \sum_{j=1}^{N_2}\mathbb{E}[|\theta_t^{2, (j, N)}|^4] \big) \leq K,
\end{aligned}
\end{equation*}
where we have used Lemma \ref{lemma: moment esitmate corollary} and the fact that, thanks to Remark \ref{remark: details on applying Skorokhod's lemma}, we can assume that $\mu^{N}$, which is a random variable defined over the probability space provided by Skorohod's lemma, can continue to be represented in terms of empirical measures of trajectories of diffusions of the type \eqref{eqn: particle system} for Brownian motions defined over the new probability space.

Letting $A \rightarrow +\infty$ in the above, we deduce from the monotone convergnce theorem that
\begin{equation}\label{eqn: fourth moment of mu_t}
    \sup_{t \in [0, T]} \mathbb{E}\big[ \int_{\mathbb{R}^d \times \mathbb{R}^d} (|\theta_1|^4 + |\theta_2|^4) \tilde  \mu_t(d\theta_1, d\theta_2) \big] \leq K.
\end{equation}
Similarly, we also have
\begin{equation}\label{eqn: fourth moment of rho_t}
\sup_{t \in [0, T]} \mathbb{E} \big[\int_{\mathbb{R}^d} |\theta|^4 \tilde \rho_t(d\theta) \big] \leq K.
\end{equation}
Lemma \ref{lemma: bound for the 2nd-monent of weighted measure} then implies that
\begin{equation}\label{eqn: fourth moment of consensus point}
    \mathbb{E}[|m_{L_k}^{\alpha}[\tilde \rho_t]|^4] < +\infty,
\end{equation}
for $k=1,2$ and all $t \in [0, T]$. From the almost sure convergences of $\langle \phi, \mu_t^N - \tilde \mu_t \rangle$, and the uniform estimates in Lemma \ref{lemma: moment esitmate corollary} and \eqref{eqn: fourth moment of consensus point} we deduce that
\begin{equation}\label{eqn: convergence in second moment}
\lim_{N \rightarrow +\infty} \mathbb{E}\big[ |\langle \phi, \mu_t^N - \tilde  \mu_t \rangle|^2 + |m_{L_1}^{\alpha}[\rho_t^N] - m_{L_1}^{\alpha}[\tilde \rho_t]|^2 + |m_{L_2}^{\alpha}[\rho_t^N] - m_{L_2}^{\alpha}[\tilde \rho_t]|^2 \big] = 0.
\end{equation}

\begin{remark}\label{remark: convergence a.s.}
For $k=1,2$, $\rho^{k, N}$ converges to $\tilde \rho^k$ almost surely as random variables valued in $\mathcal{P}(\mathcal{C}([0,T],\R^d))$ with the topology of weak convergence. In particular, there exist measurable sets $\Omega^k \subseteq \Omega$ such that $\mathbb{P}(\Omega \backslash \Omega^k) = 0$ and  
\begin{equation*}
    \lim_{N \rightarrow + \infty} \int \phi(x) d\rho_t^{k, N}(\omega) = \int \phi(x) d\tilde \rho_t^k(\omega),
\end{equation*}
for any $\phi \in \mathcal{C}_b(\mathbb{R}^d)$, any $t\in [0,T]$, and any $\omega \in \Omega^k$. 
Let $\widetilde{\Omega} := \Omega^1 \cap \Omega^2$. Then
\begin{equation*}
\begin{aligned}
\lim_{N \rightarrow +\infty} \int \phi(x) d\big( \frac{N_1}{N} \rho_t^{1, N} + \frac{N_2}{N} \rho_t^{2, N} \big)(\omega) &= \lim_{N \rightarrow +\infty} \frac{N_1}{N} \int \phi(x) d\rho_t^{1, N}(\omega) + \frac{N_2}{N} \int \phi(x) d\rho_t^{2, N}(\omega)\\
&= \int \phi(x) d\big(w_1 \tilde  \rho_t^1 + w_2 \tilde \rho_t^2 \big)(\omega),
\end{aligned}
\end{equation*}
for any $\phi \in \mathcal{C}_b(\mathbb{R}^d)$, $t\in [0,T]$ and $\omega \in \Omega$. Note that $\mathbb{P}(\Omega \backslash \widetilde{\Omega}) = 0$. From the above we infer that
\begin{equation*}
    \rho_t^N = \frac{N_1}{N}\rho_t^{1, N} + \frac{N_2}{N} \rho_t^{2, N} \overset{\text{a.s.}}{\longrightarrow} \tilde \rho_t = w_1 \tilde \rho_t^1 + w_2 \tilde \rho_t^2,
\end{equation*}
where the convergence must be interpreted in the weak sense.
\end{remark}

We are ready to prove our Theorem \ref{thm:mainMeanField}.

\begin{proof}[Proof of Theorem \ref{thm:mainMeanField}]
We may continue working on the above common probability space $(\Omega, \mathcal{F}, \mathbb{P})$ where the convergence of $(\rho^{1,N}, \rho^{2,N})$ toward $(\tilde \rho^1 , \tilde \rho^2)$ is in the a.s. sense. It is immediate that, almost surely, $\tilde \mu_t$ is continuous in time in the sense of \eqref{eqn: continuous in time}. 

Now, 
\begin{equation*}
\begin{aligned}
\E | F_{\varphi} (\mu^N, N_1/ N) - F_{\varphi}(\tilde \mu, w_1)| &-  | N_1/N - w_1|  \leq \E \bigg| \bigg( \big\langle \varphi(\theta_1, \theta_2), \mu_t^N(d\theta_1, d\theta_2) \big\rangle - \big\langle \varphi(\theta_1, \theta_2), \mu_0^N(d\theta_1, d\theta_2) \big\rangle \\
&\qquad \quad + \int_0^t \big\langle \big( \lambda_1 (\theta_1 - m_{L_1}^{\alpha}[\rho_s^N]) + \lambda_2 \nabla L_1(\theta_1) \big) \cdot \nabla_{\theta_1} \varphi \\
&\qquad \qquad \qquad  + \big( \lambda_1 (\theta_2 - m_{L_2}^{\alpha}[\rho_s^N]) + \lambda_2 \nabla L_2(\theta_2) \big) \cdot \nabla_{\theta_2} \varphi, \mu_s^N(d\theta_1, d\theta_2) \big\rangle ds \\
&\qquad \quad -\int_0^t \big\langle \big(\frac{\sigma_1^2}{2} |\theta_1 - m_{L_1}^{\alpha}[\rho_s^N]|^2 + \frac{\sigma_2^2}{2} |\nabla L_1(\theta_1)|^2 \big) \Delta_{\theta_1} \varphi\\
&\qquad \qquad \qquad + \big(\frac{\sigma_1^2}{2} |\theta_2 - m_{L_2}^{\alpha}[\rho_s^N]|^2 + \frac{\sigma_2^2}{2} |\nabla L_2(\theta_2)|^2 \big) \Delta_{\theta_2} \varphi, \mu_s^N(d\theta_1, d\theta_2) \big\rangle ds \bigg)\\
& \quad - \bigg( \big\langle \varphi(\theta_1, \theta_2), \tilde \mu_t(d\theta_1, d\theta_2) \big\rangle - \big\langle \varphi(\theta_1, \theta_2), \tilde \mu_0(d\theta_1, d\theta_2) \big\rangle \\
&\qquad \quad + \int_0^t \big\langle \big( \lambda_1 (\theta_1 - m_{L_1}^{\alpha}[\tilde \rho_s]) + \lambda_2 \nabla L_1(\theta_1) \big) \cdot \nabla_{\theta_1} \varphi \\
&\qquad \qquad \qquad + \big( \lambda_1 (\theta_2 - m_{L_2}^{\alpha}[\tilde \rho_s]) + \lambda_2 \nabla L_2(\theta_2) \big) \cdot \nabla_{\theta_2} \varphi, \tilde \mu_s(d\theta_1, d\theta_2) \big\rangle ds \\
&\qquad \quad -\int_0^t \big\langle \big(\frac{\sigma_1^2}{2} |\theta_1 - m_{L_1}^{\alpha}[\tilde \rho_s]|^2 + \frac{\sigma_2^2}{2} |\nabla L_1(\theta_1)|^2 \big) \Delta_{\theta_1} \varphi\\
&\qquad \qquad \qquad + \big(\frac{\sigma_1^2}{2} |\theta_2 - m_{L_2}^{\alpha}[\tilde \rho_s]|^2 + \frac{\sigma_2^2}{2} |\nabla L_2(\theta_2)|^2 \big) \Delta_{\theta_2} \varphi, \tilde \mu_s(d\theta_1, d\theta_2) \big\rangle ds \bigg)\bigg|\\
\leq &\E \bigg[ \bigg| \big\langle \varphi(\theta_1, \theta_2), \mu_t^N(d\theta_1, d\theta_2) - \mu_0^N(d\theta_1, d\theta_2) \big\rangle - \big\langle \varphi(\theta_1, \theta_2),\tilde \mu_t(d\theta_1, d\theta_2) - \tilde \mu_0(d\theta_1, d\theta_2) \big\rangle \bigg| \bigg]\\
+ & \E \bigg[ \bigg|\int_0^t \lambda_1 \big\langle    (\theta_1 - m_{L_1}^{\alpha}[\rho_s^N])  \cdot \nabla_{\theta_1} \varphi + (\theta_2 - m_{L_2}^{\alpha}[\rho_s^N])  \cdot \nabla_{\theta_2} \varphi, \mu_s^N(d\theta_1, d\theta_2) \big\rangle ds\\
&\quad - \int_0^t \lambda_1 \big\langle  (\theta_1 - m_{L_1}^{\alpha}[\tilde \rho_s])  \cdot \nabla_{\theta_1} \varphi +  (\theta_2 - m_{L_2}^{\alpha}[\tilde \rho_s])  \cdot \nabla_{\theta_2} \varphi,\tilde  \mu_s(d\theta_1, d\theta_2) \big\rangle ds \bigg| \bigg] \\
+ &\E \bigg[ \bigg|\int_0^t \frac{\sigma_1^2}{2} \big\langle  |\theta_1 - m_{L_1}^{\alpha}[\rho_s^N]|^2  \Delta_{\theta_1} \varphi +  |\theta_2 - m_{L_2}^{\alpha}[\rho_s^N]|^2 \Delta_{\theta_2} \varphi, \mu_s^N(d\theta_1, d\theta_2) \big\rangle ds\\
&\quad - \int_0^t \frac{\sigma_1^2}{2} \big\langle  |\theta_1 - m_{L_1}^{\alpha}[\tilde \rho_s]|^2  \Delta_{\theta_1} \varphi +  |\theta_2 - m_{L_2}^{\alpha}[\tilde \rho_s]|^2 \Delta_{\theta_2} \varphi, \tilde \mu_s(d\theta_1, d\theta_2) \big\rangle ds \bigg| \bigg]\\
+ &\E \bigg[ \bigg| \int_0^t \big\langle \lambda_2 (\nabla L_1 (\theta_1) \cdot \nabla_{\theta_1} \varphi + \nabla_{L_2} (\theta_2) \cdot \nabla_{\theta_2} \varphi)\\
&\quad + \frac{\sigma_2^2}{2} (|\nabla L_1(\theta_1)|^2 \Delta_{\theta_1}\varphi + |\nabla L_2(\theta_2)|^2 \Delta_{\theta_2} \varphi), \mu_s^N(d\theta_1, d\theta_2) - \tilde \mu_s(d\theta_1, d\theta_2) \big \rangle ds \bigg| \bigg]\\
&=: I_1^N + I_2^N + I_3^N + I_4^N.
\end{aligned}
\end{equation*}

For $\varphi \in \mathcal{C}_c^2(\R^d \times \R^d)$, using the convergence  \eqref{eqn: convergence in second moment}, we obtain
\begin{equation*}
\begin{aligned}
\lim_{N \rightarrow + \infty} I_1^N &= \lim_{N \rightarrow +\infty} \mathbb{E}\bigg[ \bigg| \big\langle \varphi(\theta_1, \theta_2), \mu_t^N(d\theta_1, d\theta_2) - \mu_0^N(d\theta_1, d\theta_2) \big\rangle - \big\langle \varphi(\theta_1, \theta_2), \tilde \mu_t(d\theta_1, d\theta_2) - \tilde \mu_0(d\theta_1, d\theta_2) \big\rangle \bigg| \bigg]\\
&= 0.
\end{aligned}
\end{equation*}
For the term $I_2^N$, we have
\begin{equation*}
\begin{aligned}
&\big| \int_0^t \langle (\theta_1 - m_{L_1}^{\alpha}[\rho_s^N]) \cdot \nabla_{\theta_1} \varphi + (\theta_2 - m_{L_2}^{\alpha}[\rho_s^N]) \cdot \nabla_{\theta_2} \varphi, \mu_s^N(d\theta_1, d\theta_2) \rangle ds \\
&\quad - \int_0^t \langle (\theta_1 - m_{L_1}^{\alpha}[\tilde \rho_s]) \cdot \nabla_{\theta_1} \varphi + (\theta_2 - m_{L_2}^{\alpha}[\tilde \rho_s]) \cdot \nabla_{\theta_2} \varphi, \tilde \mu_s(d\theta_1, d\theta_2) \rangle ds \big|\\
\leq&\int_0^t \big| \langle (\theta_1 - m_{L_1}^{\alpha}[\rho_s^N]) \cdot \nabla_{\theta_1}\varphi + (\theta_2 - m_{L_2}^{\alpha}[\rho_s^N]) \cdot \nabla_{\theta_2}\varphi, \mu_s^N(d\theta_1, d\theta_2) - \tilde \mu_s(d\theta_1, d\theta_2) \rangle \big|ds\\
&\quad + \int_0^t \big|\langle (m_{L_1}^{\alpha}[\tilde \rho_s] - m_{L_1}^{\alpha}[\rho_s^N]) \cdot \nabla_{\theta_1} \varphi + (m_{L_2}^{\alpha}[\tilde \rho_s] - m_{L_2}^{\alpha}[\rho_s^N]) \cdot \nabla_{\theta_2} \varphi,\tilde  \mu_s(d\theta_1, d\theta_2) \rangle \big|ds\\
=: & \int_0^t |I_2^{N, 1}(s)|ds + \int_0^t |I_2^{N,2}(s)|ds.
\end{aligned}
\end{equation*}
One can compute
\begin{equation*}
\begin{aligned}
\mathbb{E}[|I_2^{N,1}(s)|] &\leq \mathbb{E} \big[\big| \langle \theta_1 \cdot \nabla_{\theta_1} \varphi + \theta_2 \cdot \nabla_{\theta_2}\varphi, \mu_s^N(d\theta_1, d\theta_2) - \tilde \mu_s(d\theta_1, d\theta_2) \rangle \big| \big]\\
&\quad + \mathbb{E}\big[ \big| \langle m_{L_1}^{\alpha}[\rho_s^N] \cdot \nabla_{\theta_1}\varphi, \mu_s^N(d\theta_1, d\theta_2) - \tilde \mu_s(d\theta_1, d\theta_2) \rangle \big| \big]\\
&\quad + \mathbb{E}\big[ \big| \langle m_{L_2}^{\alpha}[\rho_s^N] \cdot \nabla_{\theta_2}\varphi, \mu_s^N(d\theta_1, d\theta_2) -\tilde  \mu_s(d\theta_1, d\theta_2) \rangle \big| \big]\\
&\leq \mathbb{E} \big[\big| \langle \theta_1 \cdot \nabla_{\theta_1} \varphi + \theta_2 \cdot \nabla_{\theta_2}\varphi, \mu_s^N(d\theta_1, d\theta_2) - \tilde \mu_s(d\theta_1, d\theta_2) \rangle \big| \big]\\
&\quad + K^{\frac{1}{2}} \big(\mathbb{E} \big[\big|\langle \nabla_{\theta_1} \varphi, \mu_s^N(d\theta_1, d\theta_2) - \tilde \mu_s(d\theta_1, d\theta_2) \rangle \big|^2 \big] \big)^{\frac{1}{2}}\\
&\quad + K^{\frac{1}{2}} \big(\mathbb{E} \big[\big|\langle \nabla_{\theta_2} \varphi, \mu_s^N(d\theta_1, d\theta_2) - \tilde \mu_s(d\theta_1, d\theta_2) \rangle \big|^2 \big] \big)^{\frac{1}{2}},
\end{aligned}
\end{equation*}
where we have used Lemma \ref{lemma: moment esitmate corollary} in the second inequality. Since $\varphi$ has a compact support, applying \eqref{eqn: convergence in second moment} leads to
\begin{equation*}
    \lim_{N \rightarrow +\infty} \mathbb{E} \big[ |I_2^{N,1}(s)| \big] = 0.
\end{equation*}
Moreover, the uniform boundedness of $\mathbb{E}\big[ |I_1^N(s)|\big] $ as a function of $s$ follows directly from \eqref{eqn: fourth moment of mu_t} and the estimates in Lemma \ref{lemma: moment esitmate corollary}. Thus, by the dominated convergence theorem
\begin{equation}\label{eqn: I2N2}
\lim_{N \rightarrow +\infty} \int_0^t \mathbb{E}\big[ |I_2^{N,1}(s)| \big]ds = 0.
\end{equation}
As for $I_2^{N,2}$, we know that
\begin{equation*}
\begin{aligned}
&\big| \langle (m_{L_1}^{\alpha}[\tilde \rho_s] - m_{L_1}^{\alpha}[\rho_s^N]) \cdot \nabla_{\theta_1} \varphi + (m_{L_1}^{\alpha}[\tilde \rho_s] - m_{L_1}^{\alpha}[\rho_s^N]) \cdot \nabla_{\theta_1} \varphi, \tilde \mu_s(d\theta_1, d\theta_2) \rangle \big|\\
& \leq \|\nabla \varphi\|_{\infty} \big( \big|m_{L_1}^{\alpha}[\tilde \rho_s] - m_{L_1}^{\alpha}[\rho_s^N] \big| + \big|m_{L_2}^{\alpha}[\tilde \rho_s] - m_{L_2}^{\alpha}[\rho_s^N] \big|\big).
\end{aligned}
\end{equation*}
Hence, from \eqref{eqn: convergence in second moment} it follows
\begin{equation*}
\lim_{N \rightarrow +\infty} \mathbb{E}\big[ |I_2^{N,2}(s)| \big] = 0.
\end{equation*}
Again, by the dominated convergence theorem, we have
\begin{equation*}
\lim_{N \rightarrow +\infty} \int_0^t \mathbb{E} \big[ |I_2^{N,2}(s)|ds \big] = 0.
\end{equation*}
This combined with \eqref{eqn: I2N2} leads to
\begin{equation}\label{eqn: estimate part 2}
\begin{aligned}
\lim_{N \rightarrow +\infty} I_2^N &= \lim_{N \rightarrow +\infty} \mathbb{E} \bigg[ \bigg| \int_0^t \langle (\theta_1 - m_{L_1}^{\alpha}[\rho_s^N]) \cdot \nabla_{\theta_1} \varphi + (\theta_2 - m_{L_2}^{\alpha}[\rho_s^N]) \cdot \nabla_{\theta_2} \varphi, \mu_s^N(d\theta_1, d\theta_2) \rangle ds \\
&\qquad \qquad \qquad - \int_0^t \langle (\theta_1 - m_{L_1}^{\alpha}[\tilde \rho_s]) \cdot \nabla_{\theta_1} \varphi + (\theta_2 - m_{L_2}^{\alpha}[\tilde \rho_s]) \cdot \nabla_{\theta_2} \varphi, \tilde \mu_s(d\theta_1, d\theta_2) \rangle ds \bigg| \bigg]\\
&= 0.
\end{aligned}
\end{equation}
Similarly, for term $I_3^N$ we split the error
\begin{equation*}
\begin{aligned}
&\big| \int_0^t \langle |\theta_1 - m_{L_1}^{\alpha}[\rho_s^N]|^2 \Delta_{\theta_1}\varphi + |\theta_2 - m_{L_2}^{\alpha}[\rho_s^N]|^2 \Delta_{\theta_2} \varphi, \mu_s^N(d\theta_1, d\theta_2) \rangle ds\\
&\quad - \int_0^t \langle |\theta_1 - m_{L_1}^{\alpha}[\tilde \rho_s]|^2 \Delta_{\theta_1}\varphi + |\theta_2 - m_{L_2}^{\alpha}[\tilde \rho_s]|^2 \Delta_{\theta_2} \varphi, \tilde \mu_s(d\theta_1, d\theta_2) \rangle ds \big|\\
= & \int_0^t \big|\langle |\theta_1 - m_{L_1}^{\alpha}[\rho_s^N]|^2 \Delta_{\theta_1} \varphi, \mu_s^N(d\theta_1, d\theta_2) - \tilde \mu_s(d\theta_1, d\theta_2) \rangle  \big| ds \\
& \quad + \int_0^t \big|\langle |\theta_2 - m_{L_2}^{\alpha}[\rho_s^N]|^2 \Delta_{\theta_2} \varphi, \mu_s^N(d\theta_1, d\theta_2) - \tilde \mu_s(d\theta_1, d\theta_2) \rangle  \big| ds\\
&\quad + \int_0^t \big| \langle \big( |\theta_1 - m_{L_1}^{\alpha}[\tilde \rho_s]|^2 - |\theta_1 - m_{L_1}^{\alpha}[\rho_s^N]|^2  \big)\Delta_{\theta_1}\varphi, \tilde \mu_s(d\theta_1, d\theta_2) \rangle \big|ds\\
&\quad + \int_0^t \big| \langle \big( |\theta_2 - m_{L_2}^{\alpha}[\tilde \rho_s]|^2 - |\theta_2 - m_{L_2}^{\alpha}[\rho_s^N]|^2  \big)\Delta_{\theta_2}\varphi, \tilde \mu_s(d\theta_1, d\theta_2) \rangle \big|ds\\
&=: \int_0^t |I_3^{N,1}(s)|ds + \int_0^t |I_3^{N,2}(s)|ds + \int_0^t |I_3^{N,3}(s)|ds + \int_0^t |I_3^{N,4}(s)|ds .
\end{aligned}
\end{equation*}
One can compute
\begin{equation*}
\begin{aligned}
\mathbb{E} \big[ |I_3^{N,1}(s)| \big] &\leq 2\mathbb{E} \big[ \big| \langle (|\theta_1|^2 + |m_{L_1}^{\alpha}[\rho_s^N]|^2) \Delta_{\theta_1}\varphi, \mu_s^N(d\theta_1, d\theta_2) - \tilde \mu_s(d\theta_1, d\theta_2) \rangle \big| \big]\\
&\leq 2\mathbb{E}\big[ \big| \langle |\theta_1|^2 \Delta_{\theta_1} \varphi, \mu_s^N(d\theta_1, d\theta_2) - \tilde \mu_s(d\theta_1, d\theta_2) \rangle \big| \big]\\
&\quad + 2\mathbb{E} \big[ |m_{L_1}^{\alpha}[\rho_s^N]|^2 |\langle \Delta_{\theta_1} \varphi, \mu_s^N(d\theta_1, d\theta_2) - \tilde \mu_s(d\theta_1, d\theta_2) \rangle| \big]\\
&\leq 2\mathbb{E}\big[ \big| \langle |\theta_1|^2 \Delta_{\theta_1} \varphi, \mu_s^N(d\theta_1, d\theta_2) - \tilde \mu_s(d\theta_1, d\theta_2) \rangle \big| \big]\\
&\quad + 2K \big(\mathbb{E} \big[ \big| \langle \Delta_{\theta_1}\varphi, \mu_s^N(d\theta_1, d\theta_2) - \tilde \mu_s(d\theta_1, d\theta_2) \rangle \big|^2 \big] \big)^{\frac{1}{2}},
\end{aligned}
\end{equation*}
where we have used Lemma \ref{lemma: moment esitmate corollary} in the last inequality. Since $\varphi$ has a compact support, applying \eqref{eqn: convergence in second moment} leads to 
\begin{equation*}
\lim_{N \rightarrow +\infty} \mathbb{E}\big[ |I_3^{N,1}(s)| \big] = 0.
\end{equation*}
Moreover, the uniform boundedness of $\mathbb{E}\big[ |I_3^{N,1}(s)| \big]$ follows directly from \eqref{eqn: fourth moment of mu_t} and the estimates in Lemma \ref{lemma: moment esitmate corollary}, which by the dominated convergence theorem implies
\begin{equation*}
    \lim_{N \rightarrow + \infty} \int_0^t \mathbb{E} \big[ |I_3^{N,1}| \big]ds = 0.
\end{equation*}
With similar calculations, we would also have 
\begin{equation*}
    \lim_{N \rightarrow + \infty} \int_0^t \mathbb{E} \big[ |I_3^{N,2}| \big]ds = 0.
\end{equation*}
As for $I_3^{N,3}$ and $I_3^{N,4}$,
\begin{equation*}
\begin{aligned}
\mathbb{E}\big[ |I_3^{N,3}| \big] &= \mathbb{E}\big[ \big| \langle \big( |\theta_1 - m_{L_1}^{\alpha}[\tilde \rho_s]|^2 - |\theta_1 - m_{L_1}^{\alpha}[\rho_s^N]|^2  \big)\Delta_{\theta_1}\varphi, \tilde \mu_s(d\theta_1, d\theta_2) \rangle \big| \big]\\
&= \mathbb{E} \bigg[ \big| \big[ \langle  |m_{L_1}^{\alpha}[\tilde \rho_s]|^2 - m_{L_1}^{\alpha}[\rho_s^N] \big|^2 + 2\theta_1 \cdot (m_{L_1}^{\alpha}[\rho_s^N] - m_{L_1}^{\alpha}[\tilde \rho_s]) \big]\Delta_{\theta_1} \varphi, \tilde \mu_s(d\theta_1, d\theta_2) \rangle \big| \bigg]\\
&\leq \mathbb{E}\bigg[ \big( |m_{L_1}^{\alpha}[\tilde \rho_s]| + |m_{L_1}^{\alpha}[\rho_s^N]| \big) \big|m_{L_1}^{\alpha}[\rho_s^N] - m_{L_1}^{\alpha}[\tilde \rho_s] \big| + 2\|\theta_1 \Delta_{\theta_1}\varphi\|_{\infty} \big|m_{L_1}^{\alpha}[\rho_s^N] - m_{L_1}^{\alpha}[\tilde \rho_s]\big| \bigg]\\
&\leq 4K^{\frac{1}{2}} \big( \mathbb{E}\big[ \big| m_{L_1}^{\alpha}[\rho_s^N] - m_{L_1}^{\alpha}[\tilde \rho_s]\big|^2 \big] \big)^{\frac{1}{2}} + 2\|\theta_1 \Delta_{\theta_1}\varphi\|_{\infty} \mathbb{E}\big[\big| m_{L_1}^{\alpha}[\rho_s^N] - m_{L_1}^{\alpha}[\tilde \rho_s]\big| \big],
\end{aligned}
\end{equation*}
where we have used Lemma \ref{lemma: moment esitmate corollary} in the last inequality. Since $\varphi$ has a compact support, applying \eqref{eqn: convergence in second moment} leads to
\begin{equation*}
    \lim_{N \rightarrow +\infty} \mathbb{E} \big[ |I_3^{N,3}(s)|\big] = 0.
\end{equation*}
Additionally, the uniform boundedness of $\mathbb{E}[|I_3^{N,3}(s)|]$ follows directly from \eqref{eqn: fourth moment of consensus point} and the estimates in Lemma \ref{lemma: moment esitmate corollary}, which by dominated convergence theorem implies 
\begin{equation*}
    \lim_{N \rightarrow +\infty} \int_0^t \mathbb{E}\big[ |I_3^{N,3}(s)| \big]ds = 0.
\end{equation*}
Similar to the above computations, we also have
\begin{equation*}
    \lim_{N \rightarrow +\infty} \int_0^t \mathbb{E}\big[ |I_3^{N,4}(s)| \big]ds = 0.
\end{equation*}
Therefore, we conclude
\begin{equation}\label{eqn: estimate part 3}
\begin{aligned}
\lim_{N \rightarrow +\infty} I_3^N &= \lim_{N \rightarrow +\infty} \mathbb{E} \bigg[\bigg| \int_0^t \langle |\theta_1 - m_{L_1}^{\alpha}[\rho_s^N]|^2 \Delta_{\theta_1}\varphi + |\theta_2 - m_{L_2}^{\alpha}[\rho_s^N]|^2 \Delta_{\theta_2} \varphi, \mu_s^N(d\theta_1, d\theta_2) \rangle ds\\
&\qquad \qquad \qquad - \int_0^t \langle |\theta_1 - m_{L_1}^{\alpha}[\tilde \rho_s]|^2 \Delta_{\theta_1}\varphi + |\theta_2 - m_{L_2}^{\alpha}[\tilde \rho_s]|^2 \Delta_{\theta_2} \varphi, \tilde \mu_s(d\theta_1, d\theta_2) \rangle ds \bigg| \bigg]\\
&= 0.
\end{aligned}
\end{equation}
For term $I_4^N$, since $L_1, L_2$ have bounded gradients and $\varphi \in \mathcal{C}_c^2(\R^d \times \R^d)$, by \eqref{eqn: convergence in second moment} we know that
\begin{equation*}
\begin{aligned}
\lim_{N \rightarrow + \infty} \E &\bigg[ \bigg|  \big\langle \lambda_2 (\nabla L_1 (\theta_1) \cdot \nabla_{\theta_1} \varphi + \nabla_{L_2} (\theta_2) \cdot \nabla_{\theta_2} \varphi)\\
&\qquad \qquad  + \frac{\sigma_2^2}{2} (|\nabla L_1(\theta_1)|^2 \Delta_{\theta_1}\varphi + |\nabla L_2(\theta_2)|^2 \Delta_{\theta_2} \varphi), \mu_s^N(d\theta_1, d\theta_2) - \tilde \mu_s(d\theta_1, d\theta_2) \big \rangle  \bigg| \bigg] = 0.
\end{aligned}
\end{equation*}
Also note that above object is uniformly bounded. Then by dominated convergence theorem we obtain that
\begin{equation*}
\lim_{N \rightarrow + \infty} I_4^N = 0.
\end{equation*}
Combining the estimations of $I_1^N, I_2^N, I_3^N$ and $I_4^N$ and the fact that $N_1/N$ was assumed to converge toward $w_1$ we infer 
\begin{equation*}
\lim_{N \rightarrow + \infty} \E \big[ |F_{\varphi}(\mu^N, N_1/N) - F_{\varphi}(\tilde \mu, w_1)| \big] = 0.
\end{equation*}
Then we have
\begin{align*}
\mathbb{E}\big[ |F_{\varphi}(\tilde \mu, w_1)| \big] & \leq \mathbb{E} \big[ |F_{\varphi}(\mu^N,N_1/N) - F_{\varphi}(\tilde \mu, w_1)| \big] + \mathbb{E}\big[ |F_{\varphi}(\mu^N,N_1/N)| \big] 
\\& \leq \mathbb{E} \big[ |F_{\varphi}(\mu^N,N_1/N) - F_{\varphi}(\tilde \mu,w_1)| \big] + \sqrt{\frac{C}{N_1N_2}} + 2 | \frac{N_1}{N} - w_1  |^2 \rightarrow 0
\end{align*}
as $N_1, N_2 \rightarrow +\infty$ and $N \rightarrow +\infty$,
where we have used Lemma \ref{lemma: bound of second moment for empirical measure} in the last inequality. This implies that $F_{\varphi}(\tilde \mu, w_1) = 0 $ almost surely. In other words, it holds that
\begin{equation*}
\begin{aligned}
&\langle \varphi(\theta_1, \theta_2), \tilde \mu_t(d\theta_1, d\theta_2) \rangle - \langle \varphi(\theta_1, \theta_2), \tilde \mu_0(d\theta_1, d\theta_2) \rangle \\ 
&\quad+  \int_0^t \langle \big( \lambda_1(\theta_1 - m_{L_1}^{\alpha}[\tilde \rho_s]) + \lambda_2 \nabla L_1 (\theta_1) \big) \cdot \nabla_{\theta_1} \varphi\\
&\qquad \qquad \qquad  + \big( \lambda_1(\theta_2 - m_{L_2}^{\alpha}[\tilde \rho_s]) + \lambda_2 \nabla L_2(\theta_2) \big) \cdot \nabla_{\theta_2} \varphi, \tilde \mu_s(d\theta_1, d\theta_2)  \rangle ds\\
& \quad - \int_0^t \langle \big( \frac{\sigma_1^2}{2} |\theta_1 - m_{L_1}^{\alpha}[\tilde \rho_s]|^2 + \frac{\sigma_2^2}{2} |\nabla L_1(\theta_1)|^2 \big) \Delta_{\theta_1} \varphi \\
&\qquad \qquad \qquad + \big(\frac{\sigma_1^2}{2}|\theta_2 - m_{L_2}^{\alpha}[\tilde \rho_s]|^2 + \frac{\sigma_2^2}{2} |\nabla L_2 (\theta_2)|^2 \big) \Delta_{\theta_2}\varphi, \tilde \mu_s(d\theta_1, d\theta_2) \rangle ds = 0 \qquad \text{a.s.},
\end{aligned}
\end{equation*}
for all $\varphi \in C_c^2(\mathbb{R}^d \times \mathbb{R}^d)$ and all $t \in [0,T]$.

From the above we can infer, using a density argument, that $\tilde{\mu}$ is a solution to \eqref{eqn: FP eqn for joint distribution} almost surely. Using the uniqueness of weak solutions to \eqref{eqn: FP eqn for joint distribution} (see Lemma \ref{lemma: uniqueness of weak solutions} in Appendix) and Remark \ref{rem:JointFP} we conclude that $\tilde{\mu } = \mu$ almost surely. In particular, $\text{Law}(\mu^N)$ converges weakly toward $\delta_{\mu}$, the Dirac delta at $\mu$. The latter statement thus holds in any probability space (and not just the one provided by Skorohod's theorem) supporting random variables with the laws of the $\mu^N$, and, since the convergence in law is toward a Dirac, we may infer from Slutsky's theorem that $\mu^N $ converges in probability toward the constant random variable $\mu$ in the original probability space as well (and not just in the space provided by Skrorohod's theorem). By convergence in probability here we mean: for every $\zeta>0$ we have
\[ \lim_{N \rightarrow \infty} \mathbb{P} \left( \sup_{t \in [0,T]} d_{LP} ( \mu^N_t, \rho_t^1 \otimes \rho_t^2) \geq \zeta   \right) =0.   \]
Notice that, indeed, the metric in the space $\mathcal{C}([0,T], \mathcal{P}(\R^d \times \R^d))$ is given by 
\[ d(\mu', \tilde \mu ') = \sup_{t \in [0,T]} d_{LP}(\mu_t', \tilde \mu_t'),  \]
where $d_{LP}(\cdot, \cdot)$ is the Levy-Prokhorov metric between probability measures over $\R^d \times \R^d$. This completes the proof.

\red 


\nc

\end{proof}

\section{Conclusions}
\label{sec:Conclusions}

This paper is a first step in bridging the consensus-based optimization literature and other PDE-based optimization methods with the federated learning problem. In particular, we have proposed a new CBO-type system of interacting particles that can be used to solve non-convex optimization problems arising in practical clustered federated learning settings. We prove that our particle system converges to a suitable mean-field limit when the number of interacting particles goes to infinity. In turn, we analyze the time evolution of the mean-field model and discuss how it forces particles within each cluster to reach consensus around a global minimizer of the cluster's objective function. This mean-field point of view may actually not be too far from reality, specially when dealing with cross-devices federated learning problems, where the number of users is indeed quite large.
Motivated by our new CBO-type particle dynamics, we propose the FedCBO algorithm and empirically asses its performance. In our experiments, we show that our algorithm outperforms current state-of-the-art methods for federated learning.

Some important questions motivated by our work that deserve further investigation are the following. On the theoretical side, the long-term stability behavior of the mean-field system is still an open problem. In particular, it is unclear how the model behaves after the variance (defined in Theorem \ref{thm: large time behavior}) reaches the prescribed tolerance level $\varepsilon$. Additionally, it would be of interest to study the finite particle system directly without passing to the mean-field limit. In addition, it would also be of interest to analyze the more realistic setting where agents within the same cluster may have different, although related, loss functions. On the experimental side, we would like to investigate the robustness to adversarial attacks of the FedCBO algorithm. Indeed, given the weighted averaging mechanism in the model aggregation step \eqref{eqn: calculate consensus point} it is not unreasonable to expect that the FedCBO algorithm can offer some protection against adversarial attacks. Finally, exploring further strategies to reduce the communication cost of our algorithm is another research topic of interest.

\bibliographystyle{abbrv}
\bibliography{ref}

\appendix
\section{Moment estimates for the Stochastic Empirical Measures}
For the solution $\bm{\theta}^{1, N} \in \mathcal{C}([0, T], \mathbb{R}^d)^{N_1}, \bm{\theta}^{2, N} \in \mathcal{C}([0, T], \mathbb{R}^d)^{N_2}$ of the particle system \eqref{eqn: particle system}, we denote by
\begin{equation*}
\rho_t^{1, N_1} = \frac{1}{N_1} \sum_{i=1}^{N_1}  \delta_{\theta_t^{1, i}} \qquad 
\rho_t^{2, N_2} = \frac{1}{N_2} \sum_{i=1}^{N_2} \delta_{\theta_t^{2, i}} \qquad
\rho_t^N = \frac{N_1}{N}\rho_t^{1, N_1} + \frac{N_2}{N} \rho_t^{2, N_2}
\end{equation*}
the empirical measures corresponding to $\bm{\theta}^{1, N}, \bm{\theta}^{2, N}$ for each $t \in [0, T]$.

\begin{lemma}\label{lemma: moment estimate}
Let $L_1, L_2$ satisfy Assumption \ref{assump: Lips of L_k} and either $(1)$ boundedness, or $(2)$ quadratic growth at infinity, and $\rho_0 \in \mathcal{P}_{2p}(\mathbb{R}^d), p \geq 1$. Further, let $\bm{\theta}^{1, N}, \bm{\theta}^{2, N}$ be the solution of the particle system \eqref{eqn: particle system} with $\rho_0^{\otimes N}$-distributed initial data $\bm{\theta}_0^{1, N}, \bm{\theta}_0^{2, N}$ and $\rho^{1, N_1}, \rho^{2, N_2}$ and $\rho^N$ the corresponding empirical measures. Then, there exists a constant $K > 0$, independent of $N$, such that
\begin{equation}\label{eqn: estimate 1 in lemma 3.1}
\sup_{t \in [0, T]} \mathbb{E} \int |\theta|^{2p} d\rho_t^N, \quad \sup_{t \in [0, T]} \mathbb{E} |m_{L_1}^{\alpha}[\rho_t^N]|^{2p}, \quad \sup_{t \in [0, T]} \mathbb{E} |m_{L_2}^{\alpha}[\rho_t^N]|^{2p} \leq K,
\end{equation}
and consequently also the estimates
\begin{subequations}\label{eqn: estimate 2 in lemma 3.1}
\begin{equation*}
\sup_{t \in [0, T]} \mathbb{E} \int |\theta|^2 d\eta_{1, t}^{\alpha, N}, \qquad \sup_{t \in [0, T]} \mathbb{E} |\theta_t^{1, i}|^{2p} \leq K,
\end{equation*}
\begin{equation*}
\sup_{t \in [0, T]} \mathbb{E} \int |\theta|^2 d\eta_{2, t}^{\alpha, N}, \qquad \sup_{t \in [0, T]} \mathbb{E} |\theta_t^{2, j}|^{2p} \leq K,
\end{equation*}
\end{subequations}
for $i=1, 2, \dots, N_1$ and $j = 1, 2, \dots, N_2$.
\end{lemma}
\begin{proof}
Let $\bm{\theta}^{1,N}, \bm{\theta}^{2, N}$ be the solution of the particle system \eqref{eqn: particle system}. Using the inequality $(a+b)^q \leq 2^{q-1}(a^q + b^q), q \geq 1$ and the Itô isometry and Jensen's inequality yields
\begin{equation}\label{eqn: ineq 1 in lemma 3.1}
\begin{aligned}
\E|\theta_t^{1, i}|^{2p} &= \E \bigg[ \theta_0^{1,i} + \int_0^t \big(-\lambda_1(\theta_s^{1,i} - m_{L_1}^{\alpha}[\rho_s^N]) - \lambda_2 \nabla L_1(\theta_s^{1, i})  \big) ds \\
&\quad + \int_0^t \sigma_1 |\theta_s^{1,i} - m_{L_1}^{\alpha}[\rho_s^N]|dB_s^{1,i} + \int_0^t \sigma_2 |\nabla L_1 (\theta_s^{1,i})| d\widetilde{B}_s^{1,i} \bigg]^{2p}\\
&\leq 2^{2p-1} \E|\theta_0^{1,i}|^{2p} + 2^{3(2p-1)} \E \bigg[\int_0^t \big(-\lambda_1(\theta_s^{1,i} - m_{L_1}^{\alpha}[\rho_s^N]) - \lambda_2 \nabla L_1(\theta_s^{1, i}) \big)ds \bigg]^{2p} \\
&\quad + 2^{3(2p-1)} \E \bigg[\int_0^t \sigma_1 |\theta_s^{1,i} - m_{L_1}^{\alpha}[\rho_s^N]|dB_s^{1,i}\bigg]^{2p} + 2^{3(2p-1)} \E \bigg[\int_0^t \sigma_2 |\nabla L_1 (\theta_s^{1,i})| d\widetilde{B}_s^{1,i}\bigg]^{2p}\\
&\leq 2^{2p-1} \E|\theta_0^{1,i}|^{2p} + 2^{4(2p-1)} \lambda_1^{2p} T^{2p-1} \int_0^t \big( \E |\theta_s^{1,i}|^{2p} ds + \E |m_{L_1}^{\alpha}[\rho_s^N]|^{2p} \big) ds\\
&\quad + 2^{3(2p-1)} \lambda_2^{2p}T^{2p-1} \int_0^t \E|\nabla L_1 (\theta_s^{1,i})|^{2p} ds\\
&\quad + 2^{4(2p-1)} T^{p-1} \sigma_1^{2p} p(2p-1)^p \int_0^t \big( \E |\theta_s^{1,i}|^{2p} ds + \E |m_{L_1}^{\alpha}[\rho_s^N]|^{2p} \big)ds\\
&\quad + 2^{3(2p-1)}T^{p-1} \sigma_2^{2p} p(2p-1)^p \int \E|\nabla L_1 (\theta_s^{1,i})|^{2p} ds\\
&= 2^{2p-1} \E|\theta_0^{1,i}|^{2p} + 2^{4(2p-1)} \big(\lambda_1^{2p} T^p + p(2p-1)^p \sigma_1^{2p} \big)T^{p-1} \int_0^t \big( \E |\theta_s^{1,i}|^{2p} ds + \E |m_{L_1}^{\alpha}[\rho_s^N]|^{2p} \big)ds\\
&\quad + 2^{3(2p-1)} \big(\lambda_2^{2p} T^p + \sigma_2^{2p} p(2p-1)^p \big) T^p C_{\nabla L_1}^{2p},
\end{aligned}
\end{equation}
for $i \in [N_1]$. Similar computations gives
for $j \in [N_2]$
\begin{equation*}
\begin{aligned}
\E|\theta_t^{2, j}|^{2p} &\leq 2^{2p-1} \E|\theta_0^{2,j}|^{2p} + 2^{4(2p-1)} \big(\lambda_1^{2p} T^p + p(2p-1)^p \sigma_1^{2p} \big)T^{p-1} \int_0^t \big( \E |\theta_s^{2,j}|^{2p} ds + \E |m_{L_2}^{\alpha}[\rho_s^N]|^{2p} \big)ds\\
&\quad + 2^{3(2p-1)} \big(\lambda_2^{2p} T^p + \sigma_2^{2p} p(2p-1)^p \big) T^p C_{\nabla L_2}^{2p}.
\end{aligned}
\end{equation*}
Summing above two inequalities over $i \in [N_1]$ and $j \in [N_2]$, dividing by $N$, we have
\begin{equation}\label{eqn: ineq 2 in lemma 3.1}
\begin{aligned}
\E \int |\theta|^{2p} d\rho_t^N &\leq 2^{2p-1} \E \int |\theta|^{2p} d\rho_0^N \\
&\quad +  2^{4(2p-1)} \big(\lambda_1^{2p} T^p + p(2p-1)^p \sigma_1^{2p} \big)T^{p-1} \int_0^t \bigg[\E \int |\theta|^{2p} d\rho_s^N \\
&\qquad \qquad \qquad \qquad \qquad \qquad \qquad \qquad \qquad \qquad \;\;+ \frac{N_1}{N} \E |m_{L_1}^{\alpha}[\rho_s^N]|^{2p} + \frac{N_2}{N} \E |m_{L_2}^{\alpha}[\rho_s^N]|^{2p}\bigg]ds\\
&\quad + 2^{3(2p-1)} \big(\lambda_2^{2p} T^p + \sigma_2^{2p} p(2p-1)^p \big)T^p \big( \frac{N_1}{N} C_{\nabla L_1}^{2p} + \frac{N_2}{N} C_{\nabla L_2}^{2p} \big).
\end{aligned}
\end{equation}
As shown in Section \ref{sec: M-F limit} and Appendix \ref{sec: quadratic growth at infinity}, for loss functions $L_1, L_2$ satisfying Assumption \ref{assump: Lips of L_k} and either bounded above or growing quadratic at infinity, we have
\begin{equation}\label{eqn: ineq 3 in lemma 3.1}
|m_{L_k}^{\alpha}[\rho_s^N]|^2 \leq \int |\theta|^2 d\eta_{k,s}^{\alpha, N} \leq c_1 + c_2 \int |\theta|^2 d\rho_s^N,
\end{equation}
for $k=1,2$ and appropriate constants $c_1, c_2$ independent of $N$, where by construction $m_{L_k}^{\alpha}[\rho_s^N] = \int \theta d\eta_{k, t}^{\alpha, N}$ with $\eta_{k, t}^{\alpha, N} = \frac{w_{L_k}^{\alpha}\rho_t^N}{\|w_{L_k}^{\alpha}\|_{\mathbb{L}^1(\rho_t^N)}}$. Therefore, we further obtain
\begin{equation*}
|m_{L_i}^{\alpha}[\rho_s^N]|^{2p} \leq \big(c_1 + c_2 \int |\theta|^2 d\rho_s^N \big)^p \leq 2^{p-1} \big( c_1^p + c_2^p \int |\theta|^{2p} d\rho_s^N \big).
\end{equation*}
Inserting above inequalities into \eqref{eqn: ineq 2 in lemma 3.1} and applying the Grönwall's inequality provides a constant $K_p > 0$, independent of $N$, such that $\sup_{t \in [0, T]} \mathbb{E} \int |\theta|^{2p} d\rho_t^N \leq K_p$ holds, and consequently also
\begin{equation*}
\sup_{t \in [0, T]} \mathbb{E} |m_{L_i}^{\alpha}[\rho_t^N]|^{2p} \leq 2^{p-1} (c_1^p + c_2^p K_p),
\end{equation*}
which concludes the proof of the estimates in \eqref{eqn: estimate 1 in lemma 3.1} by choosing $K$ sufficiently large. The other two estimates easily follow by \eqref{eqn: ineq 3 in lemma 3.1} and by applying the Grönwall's inequality on \eqref{eqn: ineq 1 in lemma 3.1}, respectively.
\end{proof}

\section{Auxiliary Propositions and Lemma for Well-posedness of System}
\subsection{loss Functions bounded above}

\begin{lemma}\label{lemma: lemma 2.2}
Let $L_1, L_2$ satisfy Assumption \ref{assump: Lips of L_k} and $\mu \in \mathcal{P}_2(\mathbb{R}^d)$ with $\int |\theta|^2 d\mu \leq K$. Then
\begin{equation*}
    \frac{e^{-\alpha \underline{L_1}}}{\|w_{L_1}^{\alpha}\|_{\mathbb{L}^1(\mu)}} \leq \exp(\alpha C_{L_1} (1 + K)) =: C_{K_1}, \qquad
    \frac{e^{-\alpha \underline{L_2}}}{\|w_{L_2}^{\alpha}\|_{\mathbb{L}^1(\mu)}} \leq \exp(\alpha C_{L_2} (1 + K)) =: C_{K_2}
\end{equation*}
\end{lemma}
\begin{proof}
The proof is the same as in \cite[Lemma 3.1]{carrillo2018analytical}.
\end{proof}

\begin{lemma}
Let $L_1, L_2$ satisfy Assumption \ref{assump: Lips of L_k} and $\mu, \hat{\mu} \in \mathcal{P}_2(\mathbb{R}^d)$ with $\int |\theta|^4 d\mu, \int |\hat{\theta}|^4 d\hat{\mu} \leq K$. Then the following stability estimates hold
\begin{equation*}
    |m_{L_1}^{\alpha}[\mu] - m_{L_1}^{\alpha}[\hat{\mu}]| \leq c_{0,1} W_2(\mu, \hat{\mu}), \qquad
    |m_{L_2}^{\alpha}[\mu] - m_{L_2}^{\alpha}[\hat{\mu}]| \leq c_{0,2} W_2(\mu, \hat{\mu}),
\end{equation*}
for constants $c_{0,k} > 0$ depending only on $\alpha, M_{L_k}$ and $K$.
\end{lemma}
\begin{proof}
The proof is the same as in \cite[Lemma 3.2]{carrillo2018analytical}.
\end{proof}

\subsection{loss Functions with Quadratic Growth at Infinity}\label{sec: quadratic growth at infinity}
In this section, we will prove the Theorem \ref{thm: well-posedness of mean-field system} for the case that the objective functions have quadratic growth at infinity, i.e. there exist constants $M > 0$ and $c_{q_k} > 0$ such that $L_k(\theta) - \underline{L_k} \geq c_{q_k}|\theta|^2$ for all $|\theta| \geq M$.

\begin{lemma}\label{lemma: bound for the 2nd-monent of weighted measure}
Let $L_1, L_2$ satisfy Assumption \ref{assump: Lips of L_k} and have quadratic growth at infinity, and $\mu \in \mathcal{P}_2(\R^d)$. Then for $k=1,2$
\begin{equation*}
\int |\theta|^2 d\eta_k^{\alpha} \leq b_{k,1} + b_{k,2} \int |\theta|^2 d\mu, \qquad \eta_k^{\alpha} = \frac{w_{L_k}^{\alpha} \mu}{\|w_{L_k}^{\alpha}\|_{\mathbb{L}^1(\mu)}},
\end{equation*}
with constants
\begin{equation*}
b_{k,1} := M^2 + b_{k,2}, \qquad b_{k,2} = 2\frac{C_{L_k}}{C_{q_k}}(1 + \frac{1}{\alpha C_{q_k}M^2}).
\end{equation*}
\end{lemma}
\begin{proof}
The proof is the same as in \cite[Lemma~3.3]{carrillo2018analytical}.
\end{proof}

\begin{proof}[Proof of Theorem \ref{thm: well-posedness of mean-field system}] 
Here we provide the proof for the case of quadratic growth at infinity. Since steps 1, 2 and 4 remain the same, we only show step 3.\\
\noindent \textbf{Step 3:} Let $(u^1, u^2) \in \mathcal{C}([0, T], \mathbb{R}^d) \times \mathcal{C}([0,T], \mathbb{R}^d)$ satisfy $(u^1, u^2) = \tau \mathcal{T}(u^1, u^2)$ for $\tau \in [0,1]$. In particular, there exists $\nu_1, \nu_2 \in \mathcal{C}([0,T], \mathcal{P}_2(\mathbb{R}^d))$ satisfying \eqref{eqn: F-P eqn for u^1} and \eqref{eqn: F-P eqn for u^2} respectively such that $(u^1, u^2) = \tau (m_{L_1}^{\alpha}[\nu], m_{L_2}^{\alpha}[\nu])$, where $\nu = w_1 \nu_1 + w_2\nu_2$. From Lemma \ref{lemma: bound for the 2nd-monent of weighted measure} and Jensen's inequality, we have
\begin{equation}\label{eqn: bound for u_t^1}
|u_t^1|^2 = \tau^2 |m_{L_1}^{\alpha}[\nu_t]|^2 \leq \tau^2 \big( b_{1,1} + b_{1,2} \int |\theta|^2 d\nu_t \big) \leq \tau^2 \big[b_{1,1} + b_{1,2}(w_1 \int |\theta|^2 d\nu_t^1 + w_2 \int |\theta|^2) d\nu_t^2 \big] .
\end{equation}
Therefore, a similar computation of the second moment estimate as in bounded case gives
\begin{equation*}
\begin{aligned}
\frac{d}{dt}\int |\theta|^2 d\nu_t^1 &\leq (d\sigma_1^2 - 2\lambda_1 + |\gamma| + \lambda_2) \int |\theta|^2 d\nu_t^1 + (d\sigma_1^2 + |\gamma|) |u_t^1|^2 + (\lambda_2 + d\sigma_2^2) C_{\nabla L_1}^2\\
&\leq (d\sigma_1^2 + |\gamma| + \lambda_2) (1 + b_{1,2}) \int |\theta|^2 d\nu_t^1 + (d\sigma_1^2 + |\gamma| + \lambda_2) b_{1,2} \int |\theta|^2 d\nu_t^2\\
&\quad + (d\sigma_1^2 + |\gamma| + \lambda_2) b_{1,1} + (\lambda_2 + d\sigma_2^2) C_{\nabla L_1}^2.
\end{aligned}
\end{equation*}
Similarly, we have
\begin{equation*}
\begin{aligned}
\frac{d}{dt}\int |\theta|^2 d\nu_t^2 &\leq (d\sigma_1^2 + |\gamma| + \lambda_2) b_{2,2} \int |\theta|^2 d\nu_t^1 + (d\sigma_1^2 + |\gamma| + \lambda_2) (1 + b_{2,2}) \int |\theta|^2 d\nu_t^2\\
&\quad + (d\sigma_1^2 + |\gamma| + \lambda_2) b_{2,1} + (\lambda_2 + d\sigma_2^2) C_{\nabla L_2}^2.
\end{aligned}
\end{equation*}
Adding above two inequalities gives
\begin{equation*}
\frac{d}{dt}\bigg( \int |\theta|^2 d\nu_t^1 + \int |\theta|^2 d\nu_t^2 \bigg) \leq C_1 \bigg( \int |\theta|^2 d\nu_t^1 + \int |\theta|^2 d\nu_t^2 \bigg) + C_2.
\end{equation*}
Then the Grönwall's inequality yields
\begin{equation*}
\int |\theta|^2 d\nu_t^1 + \int |\theta|^2 d\nu_t^2 \leq \exp(C_1 t) \bigg( \int |\theta|^2 d\nu_t^1 + \int |\theta|^2 d\nu_t^2 \bigg) + \frac{C_2}{C_1} \big(\exp(C_1 t) - 1\big).
\end{equation*}
Consequently, we know that $\|u^1\|_{\infty}$ is bounded via \eqref{eqn: bound for u_t^1}. Similar bound also hold for $\|u^2\|_{\infty}$. Then we conclude the proof by the argument as in \textbf{Step 3} for the bounded case. 
\end{proof}

\section{Auxiliary Lemma for Large-time behavior of the Mean-field Particle System}
\begin{lemma}[Evolution of variance]\label{lemma: evolution of the variances}
For $k=1,2$ let $L_k: \R^d \rightarrow \R$ and fix $\alpha, \lambda_1, \lambda_2, \sigma_1, \sigma_2 >0$. Moreover, let $T >0$ and $\rho^k \in \mathcal{C}([0,T], \mathcal{P}_4(\R))$ be the weak solution to the Fokker-Planck equations \eqref{eqn: FP eqn 1} and \eqref{eqn: FP eqn 2} respectively. Then the functional $\mathcal{V}(\rho_t^k)$ satisfies
\begin{equation*}
\begin{aligned}
\frac{d}{dt} \mathcal{V}(\rho_t^k) &\leq -\big(2\lambda_1 - 2\lambda_2 M - d\sigma_1^2 - d\sigma_2^2 M^2 \big) \mathcal{V}(\rho_t^k)\\
&\quad+ \sqrt{2} (\lambda_1 + d\sigma_1^2) |m_{L_k}^{\alpha}[\rho_t] - \theta_k^*| \sqrt{\mathcal{V}(\rho_t^k)} + \frac{d\sigma_1^2}{2} |m_{L_k}^{\alpha}[\rho_t^k] - \theta_k^*|^2,
\end{aligned}
\end{equation*}
with $M := \max\{M_{\nabla L_1}, M_{\nabla L_2}\}$.
\end{lemma}

\begin{proof}
Since $\rho^1$ is the weak solution of the Fokker-Planck equation \eqref{eqn: FP eqn 1}, then the evolution of $\mathcal{V}(\rho_t^1)$ reads
\begin{equation*}
\begin{aligned}
\frac{d}{dt} \mathcal{V}(\rho_t^1) &= \frac{1}{2} \frac{d}{dt} \int |\theta - \theta_1^*|^2 d\rho_t^1(\theta) \\
&= -\lambda_1 \int (\theta - m_{L_1}^{\alpha}[\rho_t]) \cdot (\theta - \theta_1^*) d\rho_t^1 - \lambda_2 \int \nabla L_1(\theta) \cdot (\theta - \theta_1^*) d\rho_t^1 \\
&\quad + \frac{d\sigma_1^2}{2} \int |\theta - m_{L_1}^{\alpha}[\rho_t]|^2 d\rho_t^1 + \frac{d\sigma_2^2}{2} \int |\nabla L_1(\theta)|^2 d\rho_t^1\\
&=: T_1 + T_2 + T_3 + T_4.
\end{aligned}
\end{equation*}
Expanding the right-hand side of the inner product in the integral of $T_1$ by subtracting and adding $\theta_1^*$ yields
\begin{equation}\label{eqn: evolution of mean-field limit T1}
\begin{aligned}
T_1 &= -\lambda \int \|\theta - \theta_1^*\|_2^2 d\rho_t^1(\theta) - \lambda \int (\theta_1^* - m_{L_1}^{\alpha}[\rho_t]) \cdot (\theta - \theta_1^*) d\rho_t^1(\theta)\\
&\leq -2\lambda \mathcal{V}(\rho_t^1) + \lambda \|\mathbb{E}[\rho_t^1] - \theta_1^*\|_2 \|m_{L_1}^{\alpha}[\rho_t] - \theta_1^*\|_2,
\end{aligned}
\end{equation}
where the last step is due to Cauchy-Schwarz inequality. Also note that
\begin{equation}\label{eqn: evolution of mean-field limit auxilary eqn}
\|\mathbb{E}[\rho_t^1] - \theta_1^*\|_2 \leq \int \|\theta - \theta_1^*\|_2 d\rho_t^1(\theta) \leq \sqrt{\int \|\theta - \theta_1^*\|_2^2 d\rho_t^1(\theta)} = \sqrt{2\mathcal{V}(\rho_t^1)}.
\end{equation}
Hence, we cget 
\begin{equation*}
T_1 \leq -2\lambda_1 \mathcal{V}(\rho_t^1) + \lambda_1 \sqrt{2\mathcal{V}(\rho_t)} |m_{L_1}^{\alpha} [\rho_t] - \theta_1^*|
\end{equation*}
For term $T_2$, one can compute
\begin{equation*}
T_2 = -\lambda_2 \int \nabla L_1(\theta) \cdot (\theta - \theta_1^*) d\rho_t^1 \leq \lambda_2 M_{\nabla L_1} \int |\theta - \theta_1^*|^2 d\rho_t^1 = 2\lambda_2 M_{\nabla L_1} \mathcal{V}(\rho_t^1).
\end{equation*}
For term $T_3$, again by subtracting and adding $\theta_1^*$, we have
\begin{equation}\label{eqn: evolution of mean-field limit T2}
\begin{aligned}
T_2 &= \frac{d\sigma^2}{2} \int \|\theta - m_{L_1}^{\alpha}[\rho_t]\|_2^2 d\rho_t^1(\theta)\\
&= \frac{d\sigma^2}{2} \bigg( \int \|\theta - \theta_1^*\|_2^2 d\rho_t^1(\theta) - 2 \int (\theta - \theta_1^*) \cdot (m_{L_1}^{\alpha}[\rho_t] - \theta_1^*) d\rho_t^1(\theta) + \|m_{L_1}^{\alpha}[\rho_t] - \theta_1^*\|_2^2\bigg)\\
&\leq d\sigma^2 \bigg(\mathcal{V}(\rho_t^1) + \|m_{L_1}^{\alpha}[\rho_t] - \theta_1^*\|_2 \int \|\theta - \theta_1^*\|_2 d\rho_t^1(v) + \frac{1}{2} \|m_{L_1}^{\alpha}[\rho_t] - \theta_1^*\|_2^2 \bigg),
\end{aligned}
\end{equation}
with Cauchy-Schwarz inequality being used in the last step. For term $T_4$, one can compute
\begin{equation*}
T_4 = \frac{d\sigma_2^2}{2} \int |\nabla L_1(\theta)|^2 d\rho_t^1(\theta) \leq \frac{d\sigma_2^2}{2} \int M_{\nabla L_1}^2 |\theta - \theta_1^*|^2 d\rho_t^1(\theta) = d\sigma_2^2 M_{\nabla L_1}^2 \mathcal{V}(\rho_t^1).
\end{equation*}
Therefore, combining the estimations of $T_1, T_2, T_3$ and $T_4$, we get
\begin{equation*}
\begin{aligned}
\frac{d}{dt} \mathcal{V}(\rho_t^1) &\leq -\big(2\lambda_1 - 2\lambda_2 M_{\nabla L_1} - d\sigma_1^2 - d\sigma_2^2 M_{\nabla L_1}^2 \big) \mathcal{V}(\rho_t^1)\\
&\quad+ \sqrt{2} (\lambda_1 + d\sigma_1^2) |m_{L_1}^{\alpha}[\rho_t] - \theta_1^*| \sqrt{\mathcal{V}(\rho_t^1)} + \frac{d\sigma_1^2}{2} |m_{L_1}^{\alpha}[\rho_t^1] - \theta_1^*|^2.
\end{aligned}
\end{equation*}
The computations for $\mathcal{V}(\rho_t^2)$ are similar and hence we get the upper bound as in the statement.
\end{proof}

\begin{lemma}[Quantitative Laplace principle]\label{lemma: quantitative laplace principle}
For $k=1, 2$, denote  $\underline{L_k} := \inf_{\theta \in \R^d} L_k(\theta)$. 
Let $\rho \in \mathcal{P}(\mathbb{R}^d)$ and fix $\alpha > 0$. 
For any $r > 0$, we define $L_r^k := \sup_{\theta \in B_r(\theta_k^*)} L_k(\theta)$. Then, under Assumption \ref{assump: global conv assump}, for any $r \in (0, \min\{R_0^1, R_0^2\}]$ and $q_k > 0$ such that $q_k + L_r^k - \underline{L_k} \leq L_{\infty}^k$, we have
\begin{equation*}
\|m_{L_k}^{\alpha}[\rho] - \theta_k^*\| \leq \frac{(q_k + L_r^k - \underline{L_k})^{\nu_k}}{\eta_k} + \frac{\exp(-\alpha q_k)}{\rho(B_r(\theta_k^*))} \int \|\theta - \theta_k^*\|_2 d\rho(\theta).
\end{equation*}
\end{lemma}
\begin{proof}
The same as the proof of Proposition $21$ in \cite{fornasier2021consensus}.
\end{proof}

\begin{definition}[Mollifier]\label{def: mollifier}
For $k=1,2$, $r > 0$, we define the mollifiers $\phi_r^k : \mathbb{R}^d \rightarrow \mathbb{R}$ by 
\begin{equation}\label{eqn: mollifier}
\phi_r^k(\theta) := 
\begin{cases}
\exp\big(-\frac{r^2}{r^2 - \|\theta - \theta_k^*\|_2^2}\big), & \text{if} \;\, \|\theta - \theta_k^*\|_2 < r,\\
0, & \text{else}
\end{cases}
\end{equation}
 We have $\phi_t^k(\theta_k^*) = 1$, $\text{Im}(\phi_t^k) = [0, 1]$, $\text{supp}(\phi_r^k) = B_r(\theta_k^*), \phi_r^k \in \mathcal{C}_c^{\infty} (\R^d)$ and
 \begin{align*}
\nabla \phi_r^k(\theta) &= -2r^2 \frac{\theta - \theta_k^*}{(r^2 - |\theta - \theta_k^*|^2)^2} \phi_r^k (\theta),\\
\Delta \phi_r^k(\theta) &= 2r^2 \bigg(\frac{2(2|\theta - \theta_k^*|^2 - r^2)|\theta - \theta_k^*|^2 - d(r^2 - |\theta - \theta_k^*|^2)^2}{(r^2 - |\theta - \theta_k^*|^2)^4} \bigg) \phi_r^k(\theta).
\end{align*}
\end{definition}

\begin{lemma}\label{lemma: lower bound for the mass around global minimizer}
For $k=1,2$, let $T>0, r>0$, and fix parameters $\alpha, \lambda_1, \lambda_2, \sigma_1, \sigma_2 > 0$. Assume $\rho^1, \rho^2 \in \mathcal{C}([0,T], \mathcal{P}(\R^d))$ weakly solve the Fokker-Planck equations \eqref{eqn: FP eqn 1} and \eqref{eqn: FP eqn 2} respectively with initial conditions $\rho_0^1, \rho_0^2 \in \mathcal{P}(\R^d)$. Furthermore, denote $B_k := \sup_{t \in [0, T]} |m_{L_k}^{\alpha}[\rho_t] - \theta_k^*|$. Then for all $t \in [0, T]$, we have
\begin{equation*}
\rho_t^k (B_r(\theta_k^*)) \geq \big(\int \phi_r^k(\theta) d\rho_0^k(\theta) \big) \exp\big(-(q_k^l + q_k^g)t \big),
\end{equation*}
with 
\begin{align}
q_k^l &:= \max \bigg\{ \frac{2\lambda_1 (\sqrt{c}r + B_k) \sqrt{c}}{(1-c)^2 r} + \frac{2\sigma_1^2 (cr^2 + B_k^2) (2c+d)}{(1-c)^4 r^2}, \frac{4\lambda_1^2}{(2c-1)\sigma_1^2} \bigg\},\\
q_k^g &:= \max \bigg\{\frac{2\lambda_2 c M_{\nabla L_k}}{(1-c)^2} + \frac{\sigma_2^2 M_{\nabla L_k}^2 c(2c+d)}{(1-c)^4}, \frac{4\lambda_2^2}{(2c-1)\sigma_2^2} \bigg\},
\end{align}
where $c \in (\frac{1}{2}, 1)$ can be any constant that satisfies the inequality 
\begin{equation}\label{eqn: ineq for constant c}
(2c-1)c \geq d(1-c)^2.
\end{equation}
\end{lemma}
\begin{proof}
Here we will prove the case for $\rho_t^k(B_r(\theta_1^*))$, the computation for the other one is similar. By the properties of the mollifier in Definition \ref{def: mollifier} we have $0 \leq \phi_r^1(\theta) \leq 1$ and $\text{supp}(\phi_r^1) = B_r(\theta_1^*)$. This implies $\rho_t^1(B_r(\theta_1^*)) = \rho_t^1 \big(\{ \theta \in \R^d: |\theta - \theta_1^*| \leq r \} \big) \geq \int \phi_r^1(\theta) d\rho_t^1(\theta)$. Similar to the proof in \cite{fornasier2021consensus,riedl2022leveraging}, we will derive a lower bound for the right-hand side of this inequality. Since $\rho^1$ is the weak solution of \eqref{eqn: FP eqn 1} and $\phi_r^1 \in \mathcal{C}_c^{\infty} (\R^d)$, we have 
\begin{equation*}
\frac{d}{dt} \int \phi_r^1(\theta) d\rho_t^1(\theta) = \int \big( T_1(\theta) + T_2(\theta) + T_3(\theta) + T_4(\theta) \big) d\rho_t^1(\theta),
\end{equation*}
with 
\begin{align*}
T_1(\theta) &:= -\lambda_1 (\theta - m_{L_1}^{\alpha}[\rho_t]) \cdot \nabla \phi_r^1(\theta), \qquad
T_2(\theta) := -\lambda_2 \nabla L_1 (\theta) \cdot \nabla \phi_r^1(\theta),\\
T_3(\theta) &:= \frac{\sigma_1^2}{2} |\theta - m_{L_1}^{\alpha}[\rho_t]|^2 \Delta \phi_r^1 (\theta), \qquad
T_4(\theta) := \frac{\sigma_2^2}{2} |\nabla L_1(\theta)|^2 \Delta \phi_r^1(\theta).
\end{align*}
From the proof in \cite{fornasier2021consensus,riedl2022leveraging}, we know that
\begin{equation}\label{eqn: est for T1 + T3}
    T_1(\theta) + T_3(\theta) \geq -q_1^l \phi_r^1(\theta) \qquad \text{for all $\theta \in \R^d$},
\end{equation}
where
\begin{equation*}
q_1^l := \max \bigg\{ \frac{2\lambda_1 (\sqrt{c} r + B_1) \sqrt{c}}{(1-c)^2 r} + \frac{2\sigma_1^2 (cr^2 + B_1^2) (2c + d)}{(1-c)^4 r^2}, \frac{4\lambda_1^2}{(2c-1)\sigma_1^2} \bigg\}.
\end{equation*}
Now we aim to show that $T_2(\theta) + T_4(\theta) \geq -q_1^g \phi_r^1 (\theta)$ holds for all $\theta \in \R^d$ and some constants $q_1^g > 0$. Since the mollifier $\phi_r^1$ and its first and second derivatives vanish outside of $\Omega_r := \{ \theta \in \R^d: |\theta - \theta_1^*| < r \}$ we can restrict our attention to the open ball $\Omega_r$. To achieve the lower bound over $\Omega_r$, we introduce the subsets
\begin{align*}
K_1 &:= \big\{\theta \in \R^d : |\theta - \theta_1^*| > \sqrt{c} r \big\},\\
K_2 &:= \big\{\theta \in \R^d: -\lambda_2 \nabla L_1(\theta) \cdot (\theta - \theta_1^*) (r^2 - |\theta - \theta_1^*|^2)^2 > (2c-1) r^2 \frac{\sigma_2^2}{2} |\nabla L_1 (\theta)|^2 |\theta - \theta_1^*|^2 \big\},
\end{align*}
where $c$ is the constant adhering to \eqref{eqn: ineq for constant c}. We now decompose $\Omega_r$ according to 
\begin{equation*}
\Omega_r = (K_1^c \cap \Omega_r ) \cup (K_1 \cap K_2^c \cap \Omega_r) \cup (K_1 \cap K_2 \cap \Omega_r).
\end{equation*}
In the following we treat each of these three subsets respectively.\\
\noindent \textbf{Subset $K_1^c \cap \Omega_r$:} On this subset we have $|\theta - \theta_1^*| \leq \sqrt{c} r$, then one can compute
\begin{equation*}
\begin{aligned}
T_2(\theta) &= 2\lambda_2 r^2 \frac{\nabla L_1 (\theta) \cdot (\theta - \theta_1^*)}{(r^2 - |\theta - \theta_1^*|^2)^2} \phi_r^1 (\theta)\\
&\geq -2\lambda_2 r^2 \frac{M_{\nabla L_1} |\theta - \theta_1^*|^2}{(r^2 - |\theta - \theta_1^*|^2)^2} \phi_r^1 (\theta)\\
&\geq -2\lambda_2 r^2 \frac{M_{\nabla L_1} c r^2}{(1-c)^2 r^4} \phi_r^1 (\theta)\\
&= \frac{-2c \lambda_2 M_{\nabla L_1}}{(1-c)^2} \phi_r^1 (\theta) =: -q^{g,1} \phi_r^1 (\theta).
\end{aligned}
\end{equation*}
For term $T_4$, we deduce
\begin{equation*}
\begin{aligned}
T_4(\theta) &= \frac{\sigma_2^2}{2} |\nabla L_1|^2 2r^2 \bigg(\frac{2(2|\theta - \theta_1^*|^2 - r^2) |\theta - \theta_1^*|^2 - d(r^2 - |\theta - \theta_1^*|^2)^2}{(r^2 - |\theta - \theta_1^*|^2)^4} \bigg) \phi_r^1 (\theta)\\
&\leq -\sigma_2^2 M_{\nabla L_1}^2 |\theta - \theta_1^*|^2 r^2 \bigg(\frac{2c(2c-1)r^4 - d(1-c)^2 r^4}{(1-c)^4 r^8} \bigg) \phi_r^1 (\theta)\\
&\leq -\frac{\sigma_2^2 M_{\nabla L_1}^2 c(2c + d)}{(1-c)^4} \phi_r^1 (\theta) =: -q^{g,2} \phi_r^1(\theta).
\end{aligned}
\end{equation*}
\textbf{Subset $K_1 \cap K_2^c \cap \Omega_r$:} By the definition of $K_1$ and $K_2$ we have $|\theta - \theta_1^*| > \sqrt{c}r$ and
\begin{equation*}
-\lambda_2 \nabla L_1 (\theta) \cdot (\theta - \theta_1^*) (r^2 - |\theta - \theta_1^*|^2)^2 \leq (2c-1)r^2 \frac{\sigma_2^2}{2} |\nabla L_1|^2 |\theta - \theta_1^*|^2,
\end{equation*}
respectively. Our goal now is to show that $T_2(\theta) + T_4(\theta) \geq 0$ for all $\theta$ in this subset. We first compute
\begin{equation*}
\begin{aligned}
\frac{T_2(\theta) + T_4(\theta)}{2r^2 \phi_r^1(\theta)} &= \frac{\lambda_2 \nabla L_1(\theta) \cdot (\theta - \theta_1^*) (r^2 - |\theta - \theta_1^*|^2)^2}{(r^2 - |\theta - \theta_1^*|^2)^4}\\
&\quad + \frac{\sigma_2^2}{2} |\nabla L_1|^2 \frac{2(2|\theta - \theta_1^*|^2 - r^2)|\theta - \theta_1^*|^2 - d(r^2 - |\theta - \theta_1^*|^2)^2}{(r^2 - |\theta - \theta_1^*|^2)^4}.
\end{aligned}
\end{equation*}
Therefore, we have $T_2(\theta) + T_4(\theta) \geq 0$ whenever we can show 
\begin{equation*}
\big( -\lambda_2 \nabla L_1(\theta) \cdot (\theta - \theta_1^*) + \frac{d\sigma_2^2}{2} |\nabla L_1|^2 \big) (r^2 - |\theta - \theta_1^*|^2)^2 \leq \sigma_2^2 |\nabla L_1|^2 \big(2|\theta - \theta_1^*|^2 - r^2 \big)|\theta - \theta_1^*|^2.
\end{equation*}
The first term on the left-hand side can be bounded above by
\begin{equation*}
\begin{aligned}
-\lambda_2 \nabla L_1(\theta) \cdot (\theta - \theta_1^*) (r^2 - |\theta - \theta_1^*|^2)^2 &\leq (2c-1)r^2 \frac{\sigma_2^2}{2} |\nabla L_1 (\theta)|^2 |\theta - \theta_1^*|^2\\
&\leq \frac{\sigma_2^2}{2} |\nabla L_1|^2 (2|\theta - \theta_1^*|^2 -r^2) |\theta - \theta_1^*|^2.
\end{aligned}
\end{equation*}
For the second term on the left-hand side, we can use $d(1-c)^2 \leq (2c-1)c$ to get
\begin{equation*}
\begin{aligned}
\frac{d\sigma_2^2}{2} |\nabla L_1|^2 (r^2 - |\theta - \theta_1^*|^2)^2 &\leq \frac{d\sigma_2^2}{2} |\nabla L_1|^2 (1-c)^2 r^4\\
&\leq \frac{\sigma_2^2}{2} |\nabla L_1|^2 (2c-1)r^2 cr^2\\
&\leq \frac{\sigma_2^2}{2} |\nabla L_1|^2 (2|\theta - \theta_1^*|^2 - r^2) |\theta - \theta_1^*|^2.
\end{aligned}
\end{equation*}
Hence we have $T_2(\theta) + T_4(\theta) \geq 0$ uniformly on this subset.\\
\textbf{Subset $K_1 \cap K_2 \cap \Omega_r$:} On this subset we have $|\theta - \theta_1^*| > \sqrt{c}r$ and
\begin{equation*}
-\lambda_2 \nabla L_1 (\theta) \cdot (\theta - \theta_1^*) (r^2 - |\theta - \theta_1^*|^2)^2 > (2c-1)r^2 \frac{\sigma_2^2}{2} |\nabla L_1 (\theta)|^2 |\theta - \theta_1^*|^2.
\end{equation*}
We first note that $T_2(\theta) = 0$ whenever $\lambda_2^2 |\nabla L_1(\theta)|^2 = 0$ provided that $\lambda_2 > 0$. On the other hand, if $\lambda_2^2 |\nabla L_1 (\theta)|^2 > 0$, one can compute
\begin{equation*}
\begin{aligned}
T_2(\theta) &= 2\lambda_2 r^2 \frac{\nabla L_1 (\theta) \cdot (\theta - \theta_1^*)}{(r^2 - |\theta - \theta_1^*|^2)^2} \phi_r^1(\theta)\\
&\geq 2\lambda_2 r^2 \frac{-\lambda_2 (\nabla L_1 (\theta) \cdot (\theta - \theta_1^*))^2}{(2c-1) r^2 \frac{\sigma_2^2}{2}|\nabla L_1(\theta)|^2 |\theta - \theta_1^*|^2} \phi_r^1(\theta)\\
&\geq -2\lambda_2^2 r^2 \frac{|\nabla L_1(\theta)|^2 |\theta - \theta_1^*|^2}{(2c-1)r^2 \frac{\sigma_2^2}{2}|\nabla L_1(\theta)|^2 |\theta - \theta_1^*|^2} \phi_r^1(\theta)\\
&= -\frac{4\lambda_2^2}{(2c-1)\sigma_2^2} \phi_r^1 (\theta) =: -q^{g,3} \phi_r^1(\theta).
\end{aligned}
\end{equation*}
For term $T_4$, we deduce
\begin{equation*}
\begin{aligned}
T_4(\theta) &= \frac{\sigma_2^2}{2} |\nabla L_1|^2 2r^2 \bigg( \frac{2(2|\theta - \theta_1^*|^2 - r^2)|\theta - \theta_1^*|^2 - d(r^2 - |\theta - \theta_1^*|^2)^2}{(r^2 - |\theta - \theta_1^*|^2)^4} \bigg) \phi_r^1(\theta)\\
&\geq \frac{\sigma_2^2}{2} |\nabla L_1|^2 2 r^2\bigg(\frac{2(2c-1)r^2 cr^2 - d(1-c)^2 r^4}{(r^2 - |\theta - \theta_1^*|^2)^4} \bigg) \phi_r^1 (\theta)\\
&= \frac{\sigma_2^2}{2} |\nabla L_1|^2 2r^2 \bigg(\frac{[2(2c-1)c - d(1-c)^2]r^4}{(r^2 - |\theta - \theta_1^*|^2)^4} \bigg) \phi_r^1(\theta) \geq 0,
\end{aligned}
\end{equation*}
provided $c$ satisfies $2(2c-1)c \geq d(1-c)^2$.\\
\textbf{Concluding the proof:} From the above computations, one can obtain
\begin{equation*}
\begin{aligned}
\int \big(T_2(\theta) + T_4(\theta) \big) d\rho_t^1(\theta) &= \int_{K_1^c \cap \Omega_r} \underbrace{\big( T_2(\theta) + T_4(\theta) \big)}_{\geq - (q^{g,1} + q^{g,2}) \phi_r^1(\theta)} d\rho_t^1 (\theta) + \int_{K_1 \cap K_2^c \cap \Omega_r} \underbrace{\big(T_2(\theta) + T_4(\theta) \big)}_{\geq 0} d\rho_t^1(\theta)\\
&\quad + \int_{K_1 \cap K_2 \cap \Omega_r} \underbrace{\big(T_2(\theta) + T_4(\theta) \big)}_{-q^{g,3} \phi_r^1(\theta)} d\rho_t^1(\theta)\\
&\geq \int -q_1^g \phi_r^1(\theta) d\rho_t^1(\theta),
\end{aligned}
\end{equation*}
where
\begin{equation*}
q_1^g := \max \bigg\{ q^{g, 1} + q^{g,2}, q^{g,3} \bigg\} = \max \bigg\{ \frac{2\lambda_2c M_{\nabla L_1}}{(1-c)^2} + \frac{\sigma_2^2 M_{\nabla L_1}^2 c(2c+d)}{(1-c)^4}, \frac{4\lambda_2^2}{(2c-1)\sigma_2^2} \bigg\}.
\end{equation*}
Combining above estimation with \eqref{eqn: est for T1 + T3}, we get
\begin{equation*}
\frac{d}{dt} \int \phi_r^1 (\theta) d\rho_t^1(\theta) \geq - (q_1^l + q_1^g) \int \phi_r^1(\theta) d\rho_t^1 (\theta).
\end{equation*}
By applying Grönwall's inequality and multiplying both sides $(-1)$ gives
\begin{equation*}
\int \phi_r^1(\theta) d\rho_t^1(\theta) \geq \bigg(\int \phi_r^1(\theta) d\rho_0^1(\theta) \bigg) \exp\big(-(q_1^l + q_1^g)t \big)
\end{equation*}.
Hence, we conclude
\begin{equation*}
\rho_t^1 (B_r(\theta_1^*)) \geq \bigg(\int \phi_r^1(\theta) d\rho_0^1(\theta) \bigg) \exp\big(-(q_1^l + q_1^g)t \big).
\end{equation*}
\end{proof}

\section{Auxiliary Lemmas for Mean-field Limit}

\begin{lemma}\label{lemma: Aldous criteria}
Let $\{X^n\}_{n \in \mathbb{N}}$ be a sequence of random variables defined on a probability space $(\Omega, \mathcal{F}, \mathbb{P})$ and valued in $\mathcal{C}([0, T], \mathbb{R}^d)$. The sequence of probability distributions $\{\mu_{X^n}\}_{n \in \mathbb{N}}$ of $\{X^n\}_{n \in \mathbb{N}}$ is tight on $\mathcal{C}([0, T], \mathbb{R}^d)$ if the following two conditions hold.
\begin{itemize}
    \item\label{Condition 1} (Con1) For all $t \in [0, T]$, the set of distributions of $X_t^n$, denoted by $\{\mu_{X_t^n}\}_{n \in \mathbb{N}}$, is tight as a sequence of probability measures on $\mathbb{R}^d$.
    \item\label{Condition 2} (Con2) For all $\varepsilon > 0, \eta > 0$, there exists $\delta_0 > 0$ and $n_0 \in \mathbb{N}$ such that for all $n \geq n_0$ and for all discrete-valued $\sigma(X_s^n; s\in [0,T])$-stopping times $\beta$ with $0 \leq \beta + \delta_0 \leq T$, it holds that
    \begin{equation*}
        \sup_{\delta \in [0, \delta_0]} \mathbb{P}\big(|X_{\beta + \delta_0}^n - X_{\beta}^n| \geq \eta \big) \leq \varepsilon.
    \end{equation*}
\end{itemize}
\end{lemma}

\noindent In the following we provide the more details on how to apply the Skorokhod's lemma in section \ref{sec: identification of limit measure}. 

\begin{remark}\label{remark: details on applying Skorokhod's lemma}
\noindent {Let $B := \{B^{1,i}, B^{2,j}\}_{i,j=1}^{\infty} $ denote the collection of all Brownian motions appearing in the SDEs describing the evolution of class $1$ and class $2$ particles, respectively. Also, let $(\theta^{1, i}, B^{1,i})$ and $(\theta^{2, j}, B^{2,j})$ be the respective solutions of SDEs \eqref{eqn: sde for particle 1}} and \eqref{eqn: sde for particle 2} for $i=1, \dots, N_1$ and $j=1, \dots, N_2$. Then, by the existence of universal functional solutions to classical SDEs \cite{Kallenberg1996OnTE}, we know there exist (deterministic) functionals $F_N^k: B \rightarrow  F_N^k(B) \in \mathcal{P}(\mathcal{C}([0,T], \mathbb{R}^d))$ such that $\rho^{k, N} := F_N^k(B)$ for $k=1,2$. This simply says that the empirical measures $\rho^{k,N} := \frac{1}{N_k}\sum_{i=1}^{N_k} \delta_{X^{k, i}}$ are deterministic functionals of $B$. When invoking Skorohod's lemma in section \ref{sec: identification of limit measure}, we should invoke it for the collection of random variables $\{ ( (\rho^{1,N} , B) , (\rho^{2,N}, B) ) \}_{N}$, which can be written as $\{ ( ( F_N^1(B) , B) , ( F_N^2(B), B) ) \}_{N}$. 
Since $F_N^1$ and $F_N^2$ are deterministic, we infer that the versions of $\rho^{k,N}$ in the new probability space provided by Skorohod's lemma can be written as in \eqref{eq:EmpiricalMeasuresParticleSystem} for a collection of Brownian motions in the new probability space.

\end{remark}

\begin{lemma}\label{lemma: uniqueness of weak solutions}
Assume that $\mu^1, \mu^2 \in \mathcal{C}([0,T], \mathbb{R}^d \times \mathbb{R}^d)$ are two weak solutions to PDE \eqref{eqn: FP eqn for joint distribution} with the same initial data $\mu_0$. Then it holds that
\begin{equation*}
    \sup_{t \in [0, T]} W_2(\mu_t^1, \mu_t^2) = 0,
\end{equation*}
where $W_2$ is the $2$-Wasserstein distance.
\end{lemma}
\begin{proof}
We construct two linear processes $(\widehat{\theta}_t^{1,k}, \widehat{\theta}_t^{2,k})_{t \in [0, T]}, k=1,2$ satisfying
\begin{equation*}
\begin{aligned}
d\left( \begin{array}{c}
     \widehat{\theta}_t^{1,k}  \\
     \widehat{\theta}_t^{2,k} 
\end{array}
\right) &= - \left( \begin{array}{c}
     \lambda_1 (\widehat{\theta}_t^{1,k} - m_{L_1}^{\alpha}[\rho_t^k]) + \lambda_2 \nabla L_1 (\widehat{\theta}_t^{1,k})  \\
      \lambda_1 (\widehat{\theta}_t^{1,k} - m_{L_2}^{\alpha}[\rho_t^k]) + \lambda_2 \nabla L_2 (\widehat{\theta}_t^{2,k})
\end{array}
\right)dt \\
&\quad + \sigma_1 \left(\begin{array}{cc}
   \text{diag}(|\widehat{\theta}_t^{1,k} - m_{L_1}^{\alpha} [\rho_t^k]|) & 0 \\
   0 & \text{diag}(|\widehat{\theta}_t^{2,k} - m_{L_2}^{\alpha} [\rho_t^k]|)
\end{array}
\right) \left( \begin{array}{c}
    dB_t^{1,k}     \\
    dB_t^{2,k}
   \end{array} \right) \\
&\quad + \sigma_2 \left(\begin{array}{cc}
   \text{diag}(|\nabla L_1 (\widehat{\theta}_t^{1,k})|) & 0 \\
   0 & \text{diag}(|\nabla L_2(\widehat{\theta}_t^{2,k}) |)
\end{array}
\right) \left( \begin{array}{c}
    d\widetilde{B}_t^{1,k}     \\
    d\widetilde{B}_t^{2,k}
   \end{array} \right),
\end{aligned}
\end{equation*}
with the common initial data $(\widehat{\theta}_0^1, \widehat{\theta}_0^2)$ distributed according to $\mu_0$. Above processes are linear because $\mu_t^k = (\rho_t^{1,k}, \rho_t^{2,k})$ and $\rho_t^k = w_1 \rho_t^{1,k} + w_2 \rho_t^{2,k}$ are prescribed. Let us denote $\widehat{\mu}_t^k := \text{Law}(\widehat{\theta}_t^{1,k}, \widehat{\theta}_t^{2,k}), k=1,2$, which are the weak solutions to the following linear PDE
\begin{equation*}
\begin{aligned}
\partial_t \widehat{\mu}_t^k(\theta) &= \nabla \cdot \bigg( \big(\lambda_1(\theta_1 - m_{L_1}^{\alpha}[\rho_t^k]) + \lambda_2 \nabla L_1 (\theta_1), \lambda_1(\theta_2 - m_{L_2}^{\alpha}[\rho_t^k]) + \lambda_2 \nabla L_2 (\theta_2)\big) \widehat{\mu}_t^k \bigg)\\
&\quad + \Delta_{\theta_1} \bigg( \big(\frac{\sigma_1^2}{2} |\theta_1 - m_{L_1}^{\alpha}[\rho_t^k]|^2 + \frac{\sigma_2^2}{2} |\nabla L_1(\theta_1)|^2 \big) \widehat{\mu}_t^k \bigg) + \Delta_{\theta_2} \bigg( \big(\frac{\sigma_1^2}{2} |\theta_2 - m_{L_2}^{\alpha}[\rho_t^k]|^2 + \frac{\sigma_2^2}{2} |\nabla L_2(\theta_2)|^2 \big) \widehat{\mu}_t^k \bigg),
\end{aligned}
\end{equation*}
where $\theta = (\theta_1, \theta_2) \in \R^d \times \R^d$. By the uniqueness of weak solution to the above linear PDE (see Theorem \ref{thm: existence and uniqueness of linear pde}) and the fact that $\mu^k$ are also weak solutions to the above PDE, it follows that $\widehat{\mu}_t^k = \mu_t^k$ for $ k=1,2$. Consequently, the process $\big( \widehat{\theta}_t^{1, k}, \widehat{\theta}_t^{2,k} \big)_{t \in [0, T]} = \big(\overline{\theta}_t^{1, k}, \overline{\theta}_t^{2, k} \big)_{t \in [0, T]}$ are solutions to the nonlinear SDE \eqref{eqn: mean-field sde}, for which the uniqueness has been obtained in Theorem \ref{thm: well-posedness of mean-field system}. In particular,
it holds that
\begin{equation*}
\sup_{t \in [0, T]} \mathbb{E} \big[ |\overline{\theta}_t^{1,1} - \overline{\theta}_t^{1,2}|^2 + |\overline{\theta}_t^{2,1} - \overline{\theta}_t^{2,2}|^2 \big] = 0,
\end{equation*}
which by the definition of Wasserstein distance implies
\begin{equation*}
\begin{aligned}
\sup_{t \in [0, T]} W_2(\mu_t^1, \mu_t^2) &= \sup_{t \in [0, T]} W_2(\widehat{\mu}_t^1, \widehat{\mu}_t^2)\\
&\leq \sup_{t \in [0, T]} \mathbb{E} \big[ |\widehat{\theta}_t^{1,1} - \widehat{\theta}_t^{1,2}|^2 + |\widehat{\theta}_t^{2,1} - \widehat{\theta}_t^{2,2}|^2 \big]\\
&= \sup_{t \in [0, T]} \mathbb{E} \big[ |\overline{\theta}_t^{1,1} - \overline{\theta}_t^{1,2}|^2 + |\overline{\theta}_t^{2,1} - \overline{\theta}_t^{2,2}|^2 \big] = 0.
\end{aligned}
\end{equation*}
Thus the uniqueness is obtained.
\end{proof}

\begin{theorem}[Existence and Uniqueness of Linear PDE]\label{thm: existence and uniqueness of linear pde}
For any $T > 0$, let $b^1, b^2 \in \mathcal{C}([0, T], \mathbb{R}^d)$ and $\mu_0 \in \mathcal{P}_2(\mathbb{R}^d) \times \mathcal{P}_2(\mathbb{R}^d)$. Then the following linear PDE
\begin{equation}\label{eqn: linear PDE}
\begin{aligned}
\partial_t \mu_t &= -\nabla \cdot \bigg( \big( \lambda_1 (x_1 - b_t^1) + \lambda_2 \nabla L_1(x_1), \lambda_1 (x_2 - b_t^2) + \lambda_2 \nabla L_2(x_2) \big) \mu_t \bigg)\\
&\quad + \Delta_{x_1} \bigg( \big( \frac{\sigma_1^2}{2} |x_1 - b_t^1|^2 + \frac{\sigma_2^2}{2} |\nabla L_1(x_1)|^2 \big) \mu_t \bigg) + \Delta_{x_2} \bigg( \big( \frac{\sigma_1^2}{2} |x_2 - b_t^2|^2 + \frac{\sigma_2^2}{2} |\nabla L_2(x_2)|^2 \big) \mu_t \bigg),
\end{aligned}
\end{equation}
has a unique weak solution $\mu \in \mathcal{C}([0, T], \mathcal{P}_2(\mathbb{R}^d) \times \mathcal{P}_2(\mathbb{R}^d))$.
\end{theorem}

\begin{proof}[Sketch of the proof]
The existence is obivous, which can be obtained as the law of the solution to the associated linear SDE. To show the uniqueness we can follow a duality argument. For each $t_0 \in (0, T]$ and function $\varphi \in \mathcal{C}_c^{\infty}(\R^d \times \R^d)$, we consider the following backward PDE
\begin{equation}\label{eqn: backward pde}
\begin{aligned}
\partial_t h_t &= - \big(\lambda_1(x_1 - b_t^1) + \lambda_2\nabla L_1(x_1), \lambda_1(x_2 - b_t^2) + \lambda_2 \nabla L_2(x_2) \big) \cdot \nabla h_t\\
&\quad\, - \big[(\frac{\sigma_1^2}{2} |x_1 - b_t^1|^2 + \frac{\sigma_2^2}{2} |\nabla L_1(x_1)|^2) \Delta_{x_1}h_t + (\frac{\sigma_1^2}{2}|x_2 - b_t^2|^2 + \frac{\sigma_2^2}{2} |\nabla L_2 (x_2)|^2) \Delta_{x_2} h_t \big],
\end{aligned}
\end{equation}
with $(t,x_1,x_2) \in [0, t_0] \times \R^d \times \R^d; h_{t_0} = \varphi$. It admits a classical solution $h \in \mathcal{C}^1([0, t_0], \mathcal{C}^2(\R^d \times \R^d))$. Indeed, by Kolmogorov backward equation we can explicitly construct a solution 
\begin{equation*}
h_t(x_1, x_2) = \E\big[ \varphi(X_{t_0}^{1, t, x_1}, X_{t_0}^{2, t, x_2}) \big] \qquad t \in [0, t_0],
\end{equation*}
where $(X_s^{1, t, x_1}, X_s^{2, t, x_2})_{0 \leq t \leq s \leq t_0}$ is the strong solution to the following SDE
\begin{equation*}
\begin{aligned}
d\left( \begin{array}{c}
     X_s^{1, t, x_1}  \\
     X_s^{2, t, x_2}
\end{array}
\right) &=  \left( \begin{array}{c}
     \lambda_1 (X_s^{1, t, x_1} - b_s^1) + \lambda_2 \nabla L_1 (X_s^{1, t, x_1})  \\
      \lambda_1 (X_s^{2, t, x_2} - b_s^2) + \lambda_2 \nabla L_2 (X_s^{2, t, x_2})
\end{array}
\right)ds \\
&\quad + \sigma_1 \left(\begin{array}{cc}
   \text{diag}(|X_s^{1, t, x_1} - b_s^1|) & 0 \\
   0 & \text{diag}(|X_s^{2, t, x_2} - b_s^2|)
\end{array}
\right) \left( \begin{array}{c}
    dB_s^1     \\
    dB_s^2
   \end{array} \right) \\
&\quad + \sigma_2 \left(\begin{array}{cc}
   \text{diag}(|\nabla L_1 (X_s^{1, t, x_1})|) & 0 \\
   0 & \text{diag}(|\nabla L_2(X_s^{2, t, x_2}) |)
\end{array}
\right) \left( \begin{array}{c}
    d\widetilde{B}_s^1     \\
    d\widetilde{B}_s^2
   \end{array} \right),\\
\left( \begin{array}{c}
     X_t^{1, t, x_1}  \\
     X_t^{2, t, x_2}
\end{array}
\right) &=  \left( \begin{array}{c}
     x_1  \\
     x_2
\end{array}
\right),
\end{aligned}
\end{equation*}
with $B^1, B^2, \widetilde{B}^1$ and $\widetilde{B}^2$ being independent $d$-dimensional Brownian motion.\\
\noindent Suppose that $\mu^1$ and $\mu^2$ are two weak solutions of \eqref{eqn: linear PDE} with the same initial condition $\mu_0^1 = \mu_0^2$. Denote $\delta_{\mu} := \mu^1 - \mu^2$. Using the above defined solution $h$ to the backward PDE \eqref{eqn: backward pde} as a test function, we have
\begin{equation*}
\begin{aligned}
&\big\langle h_{t_0}(x_1, x_2), \delta_{\mu_{t_0}}(dx_1, dx_2) \big\rangle\\
&= \int_0^{t_0} \big\langle \partial_s h_s(x_1, x_2), \delta_{\mu_s}(dx_1, dx_2) \big\rangle ds\\
& \quad + \int_0^{t_0} \big\langle \big(\lambda_1(x_1 - b_s^1) + \lambda_2 \nabla L_1(x_1), \lambda_1(x_2 - b_s^2) + \lambda_2 \nabla L_2(x_2) \big) \cdot \nabla h_s(x_1, x_2), \delta_{\mu_s}(dx_1, dx_2) \big\rangle ds \\
& \quad + \int_{0}^{t_0} \big\langle (\frac{\sigma_1^2}{2}|x_1 - b_s^1|^2 + \frac{\sigma_2^2}{2}|\nabla L_1(x_1)|^2) \Delta_{x_1} h_s(x_1, x_2)\\
& \qquad \qquad  + (\frac{\sigma_1^2}{2} |x_2 - b_s^2|^2 + \frac{\sigma_2^2}{2} |\nabla L_2(x_2)|^2) \Delta_{x_2} h_s(x_1, x_2), \delta_{\mu_s} (dx_1, dx_2) \big\rangle ds\\
&= \int_0^{t_0} \big\langle \partial_s h_s(x_1, x_2), \delta_{\mu_s}(dx_1, dx_2) \big\rangle ds + \int_0^{t_0} \big\langle -\partial_s h_s(x_1, x_2), \delta_{\mu_s}(dx_1, dx_2) \big\rangle ds\\
&=0,
\end{aligned}
\end{equation*}
which gives $\int_{\R^d \times \R^d} \varphi(x_1, x_2) \delta_{\mu_{t_0}}(dx_1, dx_2) = 0$ for arbitrary $\varphi \in \mathcal{C}_c^{\infty}(\R^d \times \R^d)$. This implies $\delta_{\mu_{t_0}} = 0$, which yields the uniqueness by the arbitrariness of $t_0$.
\end{proof}

\end{document}